\pdfoutput=1
\documentclass[11pt]{article}

\usepackage[letterpaper,margin=1in]{geometry}
\usepackage{lipsum}

\newcommand\blfootnotea[1]{%
  \begingroup
  \renewcommand\thefootnote{}\footnote{#1}%
  \endgroup
}

\usepackage[utf8]{inputenc} %
\usepackage[T1]{fontenc}    %

\usepackage{url}            %
\usepackage{booktabs}       %
\usepackage{nicefrac}       %
\usepackage{microtype}      %
\usepackage{enumitem}
\usepackage{dsfont}
\usepackage{multicol}
\usepackage{xcolor}

\usepackage{amsthm,amsfonts,amsmath,amssymb,epsfig,color,float,graphicx,verbatim, enumitem}

\usepackage{algpseudocode,algorithm,algorithmicx}  

\usepackage{mathtools}
\usepackage{bbm}

\usepackage{thm-restate}

\definecolor{green}{rgb}{0.0, 0.5, 0.0}
\usepackage[colorlinks,citecolor=green,linkcolor=blue,bookmarks=true,hypertexnames=false]{hyperref}
\usepackage[nameinlink]{cleveref}

\crefname{equation}{equation}{equations}
\crefname{lemma}{lemma}{lemmata}
\crefname{claim}{claim}{claims}
\crefname{theorem}{theorem}{theorems}
\crefname{proposition}{proposition}{propositions}
\crefname{corollary}{corollary}{corollaries}
\crefname{claim}{claim}{claims}
\crefname{remark}{remark}{remarks}
\crefname{definition}{definition}{definitions}
\crefname{fact}{fact}{facts}
\crefname{question}{question}{questions}
\crefname{condition}{condition}{conditions}
\crefname{algorithm}{algorithm}{algorithms}
\crefname{assumption}{assumption}{assumptions}
\crefname{setting}{setting}{settings}
\crefname{example}{Example}{Examples}

\newtheorem{theorem}{Theorem}[section]
\newtheorem{lemma}[theorem]{Lemma}

\newtheorem{corollary}[theorem]{Corollary}
\newtheorem{claim}[theorem]{Claim}

\newtheorem{definition}[theorem]{Definition}

\newtheorem{fact}[theorem]{Fact}

\theoremstyle{definition}
\newtheorem{assumption}[theorem]{Assumption}
\newtheorem{condition}[theorem]{Condition}
\newtheorem{setting}[theorem]{Setting}

\newtheorem{example}[theorem]{Example}

\newcommand{\eps}{\epsilon}

\newcommand{\poly}{\mathrm{poly}}
\newcommand{\polylog}{\mathrm{polylog}}
\newcommand{\dtv}{d_\mathrm{TV}}
\newcommand{\Ind}{\mathds{1}}
\newcommand{\1}{\Ind}

\newcommand{\trace}{\operatorname{tr}}

\newcommand{\Var}{\operatorname{Var}}

\def\P{\mathbb P}
\def\R{\mathbb R}

\def\N{\mathbb N}

\newcommand\numberthis{\addtocounter{equation}{1}\tag{\theequation}}

\newcommand{\cE}{\mathcal{E}}
\newcommand{\cF}{\mathcal{F}}

\newcommand{\cN}{\mathcal{N}}

\newcommand{\cU}{\mathcal{U}}
\newcommand{\cV}{\mathcal{V}}

\newcommand{\cX}{\mathcal{X}}

\newcommand{\bA}{\vec{A}}
\newcommand{\bB}{\vec{B}}
\newcommand{\bC}{\vec{C}}

\newcommand{\bI}{\vec{I}}

\newcommand{\bM}{\vec{M}}

\newcommand{\bU}{\vec{U}}

\DeclareMathOperator*{\pr}{\mathbf{Pr}}
\DeclareMathOperator*{\E}{\mathbf{E}}

\newcommand{\tr}{\mathrm{tr}}
\DeclarePairedDelimiter\abs{\lvert}{\rvert}
\let\vec\mathbf
\newcommand{\failp}{\tau}

\newcommand{\s}{\vec \Sigma}
\newcommand{\os}{\vec \Sigma} %
\newcommand{\solGenerator}{\textsc{SampleTopEigenvector} }
\newcommand{\op}{\mathrm{op}}
\newcommand{\fr}{\mathrm{F}}
\newcommand{\suc}{\mathrm{s}}
\newcommand{\hatm}{\widehat{\vec M}}
\newcommand{\hatg}{\widehat{g}}
\newcommand{\hatf}{\widehat{g}}
\newcommand{\iid}{i.i.d.\ }

\def\colorful{0}

\ifnum\colorful=1
\newcommand{\new}[1]{{\color{red} #1}}

\newcommand{\inote}[1]{\footnote{{\bf [Ilias: {#1}\bf ] }}}
\newcommand{\dnote}[1]{\footnote{{\bf [Daniel: {#1}\bf ] }}}
\newcommand{\anote}[1]{\footnote{{\bf [Ankit: {#1}\bf ]}}}
\newcommand{\tnote}[1]{\footnote{{\bf [Thanasis: {#1}\bf ] }}}

\else
\newcommand{\new}[1]{{#1}}

\newcommand{\inote}[1]{}
\newcommand{\dnote}[1]{}
\newcommand{\anote}[1]{}
\newcommand{\tnote}[1]{}

\fi

\allowdisplaybreaks

\title{Nearly-Linear Time and Streaming Algorithms for\\ Outlier-Robust PCA\blfootnotea{Authors are in alphabetical order.}}

\author{
Ilias Diakonikolas\thanks{Supported by NSF Medium Award CCF-2107079, NSF Award CCF-1652862 (CAREER), a Sloan Research
Fellowship, and a DARPA Learning with Less Labels (LwLL) grant.}\\
University of Wisconsin-Madison\\
{\tt ilias@cs.wisc.edu}\\
\and
Daniel M. Kane\thanks{Supported by NSF Medium Award CCF-2107547, NSF Award CCF-1553288 (CAREER), and a Sloan Research
Fellowship.}\\
University of California, San Diego\\
{\tt dakane@cs.ucsd.edu}
\and
Ankit Pensia\thanks{Supported by NSF Award CCF-1652862 (CAREER), and NSF grants CCF-1841190 and CCF-2011255.}\\
University of Wisconsin-Madison\\
{\tt ankitp@cs.wisc.edu}\\
\and
Thanasis Pittas\thanks{Supported by NSF Medium Award CCF-2107079.}\\
University of Wisconsin-Madison\\
{\tt pittas@wisc.edu}\\
}

\begin{document}

\maketitle

\begin{abstract}
We study principal component analysis (PCA), where given a dataset in $\R^d$ from a distribution, the task is to find a unit vector $v$ that approximately maximizes the variance of the distribution after being projected along $v$.
Despite being a classical task, standard estimators fail drastically if the data contains even a small fraction  of outliers, 
motivating the problem of robust PCA.
Recent work has developed computationally-efficient algorithms for robust PCA that either take super-linear time or have sub-optimal error guarantees.
Our main contribution is to develop a nearly-linear time algorithm for robust PCA with near-optimal error guarantees. We also develop a single-pass streaming algorithm for robust PCA with memory usage  nearly-linear in the dimension.
\end{abstract}

\thispagestyle{empty}
\newpage
{
  \hypersetup{linktoc=all,linkcolor=black}
  \tableofcontents
}
\thispagestyle{empty}
\newpage
\setcounter{page}{1}

\section{Introduction}

Principal component analysis (PCA) is a central subroutine in dimension
reduction and data visualization.
PCA is typically used to identify directions of large variance in the data, which are considered to be
the most informative aspects of the data.
In the classical setting, we observe a set $S$ of $n$ \iid points in $\R^d$
from a (subgaussian) distribution with covariance $\s$.
Then, classical estimators, for e.g., the leading eigenvector of the empirical
covariance of $S$, output a direction $v$ with the property that
$v^\top \s v$ approaches $\|\s\|_\op$ as $n$ increases.
Importantly, these estimators can be computed fast --- namely, they have nearly linear runtime.\footnote{By the term ``nearly linear time'' algorithms, we mean runtime scaling linearly in the size of input $nd$, (up to logarithmic factors.}

However, access to \iid samples is often an unrealistic assumption, and data may contain a small
number of outliers as formalized below:
\begin{definition}[Strong Contamination Model]
	\label{def:strongadv}
	Given a parameter $0<\eps<1/2$ and a class of distributions $\mathcal{D}$,
	the strong adversary operates as follows: The algorithm specifies
	a number of samples $n$, then the adversary draws a set of $n$ \iid samples
	from some $D \in \mathcal{D}$ and after inspecting them,
	removes up to $\eps n$ of them and replaces them with arbitrary points.
	The resulting set is given as input to the learning algorithm.
	We call a set $\eps$-corrupted if it has been generated by the above process.
\end{definition}

Unfortunately, outliers can skew the results of classical estimators, by
artificially increasing the variance along low-variance directions, which
renders the output of these estimators unreliable.\footnote{Naive techniques such as naive pruning and random projection are also susceptible to outliers.}
To address this, a robust algorithm must be able to effectively remove outliers
from the data, while still being computationally efficient in high dimensions.

Recent works have developed computationally-efficient algorithms 
for robust PCA~\cite{JamLT20,KonSKO20}.
These works fall under the umbrella of algorithmic robust 
statistics, where the goal is to develop computationally-efficient 
algorithms for high-dimensional statistical estimation 
that are robust to a constant fraction of outliers, see, 
e.g.,~\cite{DKKLMS16,LaiRV16, DK19-survey, DiaKan22-book}.
However, these polynomial-time algorithms~\cite{JamLT20,KonSKO20} either
(i) achieve near-optimal error but take super-linear time (scaling with $nd^2$) or
(ii) run in nearly linear time but achieve sub-optimal error and require additional assumptions. 
As modern datasets continue to expand in both sample size and dimension,
there is a growing need for nearly linear time algorithms and streaming
algorithms with near-optimal error guarantees.
Our work addresses this need by developing robust PCA algorithms that meet
these demands.

	\subsection{Our Results}\label{sec:our-results}

	The main result of this paper is a nearly-linear time algorithm to identify a
	direction of large variance from a corrupted high-dimensional dataset.
	We state our results below for subgaussian distributions (\Cref{def:subgaussianity}), and note that the same results
	hold for a more general class of ``stable'' distributions
	(\Cref{def:stability}).
	\begin{theorem}[Nearly Linear Time Robust PCA]
		\label{thm:near-linear-main}
		Let $\epsilon \in (0, c)$ for a small constant $c > 0$.
		Let $D$ be a subgaussian distribution with mean zero and (unknown) covariance
		matrix $\s$.
		For a sufficiently large $C>0$, let $S$ be an $\eps$-corrupted set of
		$n \geq Cd/\epsilon^2$ samples from $D$ in the strong contamination model
		(\Cref{def:strongadv}).
		There exists an algorithm that takes as inputs $S$ and $\eps$, runs in
		$\widetilde{O}(n d/\eps^2)$ time and with probability at least $0.99$ outputs
		a unit vector $u$ such that $u^\top \s u \geq (1-O(\eps \log(1/\eps)))\| \vec
			\Sigma \|_\op$.
		\looseness=-1
	\end{theorem}

We note that the error guarantee of our algorithm is near-optimal
	up to logarithmic factors (cf.\ \Cref{sec:info-theoretic-error}).
	The work most closely related to ours is \cite[Theorem
		2]{JamLT20}, and our
	algorithm improves upon it in three
	significant ways:
	\begin{itemize}%
		\item (Near-optimal error) The output of \Cref{thm:near-linear-main} achieves the near-optimal $(1-O(\eps \log(1/\eps)))$ approximation of $\|\s\|_\op$, whereas the result of \cite{JamLT20} achieves the sub-optimal approximation of $ (1 - O(\sqrt{\eps
					\log(1/\eps)\log d}))$.\footnote{In particular, the latter result requires that $\eps \to 0$ 
     as $d \to \infty$ to obtain non-vacuous guarantees.}
		\item (Dependence on $\eps$ in runtime) The runtime of our algorithm is $\widetilde{O}(nd/\eps^{2})$, in contrast to the runtime of \cite{JamLT20}, which is
			$\widetilde{O}\left(nd/\eps^{4.5} {+} ndm/\eps^{1.5}\right)$, where $m$ is a parameter discussed below that belongs in $[2,d]$.
			Thus, the runtime of our algorithm improves upon theirs in all cases.
			\looseness=-1
		\item (Eigenvalue separation) The algorithm of \cite{JamLT20} requires the assumption that the $m$-th largest eigenvalue of $\s$ is  smaller than $\|\s\|_\op$ by $(1-\widetilde{\Theta}(\eps))$ for some $m\in [2,d]$, which appears in the runtime. Thus, \cite{JamLT20} obtains nearly linear time only if $m
					{=} \polylog(d/\eps)$.
			\Cref{thm:near-linear-main} has no such restriction.
	\end{itemize}

	We finally note that, even without corruptions, 
$\widetilde{\Omega}(d/\eps^2)$ samples are necessary for $(1{-}
O(\eps\log(1/\eps)))$-approximation of $\|\s\|_\op$.
	Moreover, even if there are no corruptions, the standard Lanczos algorithm
	achieving such a guarantee runs in time $\widetilde{O}(\frac{nd}{\sqrt{\eps}}{+}
\frac{1}{\eps})$ \cite[Theorem 10.1]{SacVis14}.
	Thus, our result considerably reduces the price of robustness in PCA, and
	brings it much closer to the clean data setting.\looseness=-1

	\subsubsection*{Streaming Algorithm}
	
	The distributed nature of modern data science applications and large-scale
	datasets often impose the restriction that the complete dataset cannot be
	stored in the memory.
	In such scenarios, the streaming model, defined below, is a much more
	realistic setting:
	\begin{definition}[Single-Pass Streaming Model] \label{def:streaming} Let $S$ be
		a fixed set.
		In the one-pass streaming model, the elements
		of $S$ are revealed one at a time to the algorithm, and the
		algorithm is allowed a single pass over these points.
	\end{definition}
	In the streaming model, the algorithm also needs to optimize the amount of
	memory that it uses.
	We still consider corruptions in the data. This means that the input $S$ above is not comprised of \iid samples from a distribution, but comes from a ``corrupted'' distribution, formalized below:
	\begin{definition}[TV-contamination]
		\label{def:oblivious}
		Given a parameter $\eps \in (0,1/2)$ and a distribution class $\mathcal{D}$,
		the adversary specifies a distribution $D'$
		such that there exists $D \in \mathcal{D}$ with $\dtv(D,D') \leq \eps$.
		Then the algorithm draws \iid samples
		from $D'$.
		We say that the distribution $D'$ is an $\eps$-corrupted version
		of the distribution $D$ in total variation distance.
	\end{definition}
	All prior algorithms for robust PCA needed to store the entire dataset in
	memory, leading to space usage of at least
$\widetilde{\Omega}(d^2/\eps^2)$.\footnote{In experiments, memory usage has been highlighted as a key bottleneck in robust estimation; see \cite{DiaKKLMS17}.}
	Building on \Cref{thm:near-linear-main}, we present the first algorithm for
	robust PCA that uses $\widetilde{O}_\epsilon(d)$ memory.

	\begin{theorem}[Streaming Algorithm for Robust PCA; informal version]
		\label{thm:streaming-main}
		Let $\epsilon \in (0, c)$ for a small constant $c > 0$.
		Let $D$ be a subgaussian distribution with mean zero and (unknown) covariance
		matrix $\s$.
		Let $P$ be an $\eps$-corrupted
		version of $D$ in total variation distance (\Cref{def:oblivious}).
		There is a single-pass streaming algorithm that given $\eps$, it reads
		$\poly(d/\eps)$ many
        samples from $P$, and with
		probability $0.99$, returns a unit vector $u$ such that $u^\top\s u \geq (1 -
			O(\eps \log(1/\eps)))\|\s\|_\op$.
		Moreover, the algorithm uses memory $(d/\eps)\cdot \polylog(d/\eps)$ and runs
		in time $(nd/\eps^2) \polylog(d/\eps)$.
	\end{theorem}
	Observe that the asymptotic error of the algorithm is near-optimal and the memory
	usage is nearly linear in $d$ (the size of the output).
	We refer the reader to \Cref{sec:bit-complexity} for further discussion on bit
	complexity.
	\subsection{Our Techniques}
	\label{sec:our-tech}
	Our approach is to combine the ``easy'' version of the robust PCA algorithm
	from \cite{JamLT20, KonSKO20} with the fast mean
	estimation algorithm of \cite{DKPP22} (see also \cite{DKKLT21}).
	This ``easy'' algorithm works
	essentially by computing the top eigenvector, $v$, of the empirical
	covariance matrix.
	It then projects the samples onto the $v$-direction and estimates the true
	covariance in the $v$-direction.
	If this is close to the empirical covariance, then this and the fact that
	errors cannot substantially decrease the empirical covariance in any direction
	implies that $v$ is close to a principal eigenvector of the true covariance
	matrix (cf.\ \Cref{lem:basic-cert}).
	Otherwise, some small number of corruptions must
	account for the difference, and these can be algorithmically filtered out.
	One then repeats this process of filtering out outliers until an
	approximate principal eigenvector is found.

	The issue with this simple algorithm is its runtime.
	Although each filtering step can be performed in nearly linear time, there is no guarantee
	 that the number of these steps will be small.
	It is entirely possible that each filtering step removes only a tiny fraction
	of corruptions with very large projections in the leading eigenvector direction $v$.
	To fix this, we leverage ideas from \cite{DKPP22} and make $v$ essentially a random
	linear combination of the few largest eigenvectors of $\s_t$ (the empirical
	covariance at step $t$), by defining it to be $\s_t^p w$ (instead of the
	principal eigenvector of the empirical covariance matrix), where $w$ is a random Gaussian vector and $p$ is a
	suitably large integer.
	This prevents an adversary from ``hiding'' a corruption by making it orthogonal to
	some particular $v$.
	Instead, any corruption substantially contributing to one of the largest eigenvalues
	of $\s_t$ is reasonably likely to be filtered out.
	Our approach is to keep track of the potential function $\tr(\s_t^{2p+1})$ and
	show: (i) small value of the potential certifies that a solution vector can be
	recovered from the empirical covariance, (ii) whenever this certification is
	not possible, we can reduce the potential by a multiplicative factor of ($1 - \Omega(\eps)$).
	This approach presents two main challenges.

	The first key challenge, pertinent to the ``certification'' that was mentioned
	before, is to obtain optimal error without any restriction on the spectrum of
$\s$
	(recall that the near-linear time result of \cite{JamLT20} imposes restrictions on the spectrum of $\s$).
	A natural stopping condition is when the current $\s_t$ is comparable (in an
	appropriate sense) to $\s$ up to $(1 \pm \epsilon)$ factor.
	\cite{JamLT20} used Schatten-$p$ norm for large enough $p$ to compare $\s$ and $\s_t$.
	However, a simple example (\Cref{ex:hard-schatten}) shows that, even then, any
	deterministic algorithm relying on $\s_t$ must take $nd^2$ time.
	Our two-pronged solution is to (i) perform a white-box analysis of robust PCA
	to enforce a stronger stopping condition, and (ii) output a \emph{random}
	leading eigenvector of $\s_t$ (cf.\ \Cref{sec:advanced-cert-lemma}). \looseness=-1

	Second, unlike in \cite{DKPP22}, we cannot quite achieve a runtime of
$\widetilde{O}(nd)$.
	This is essentially because of the choice of $p$.
	Assuming that $\|\s\|_{\op} =1$, na\"ive
	filtering can guarantee that our initial $\s_t$ at $t=0$ has eigenvalues at most
$\poly(d)$.\footnote{As all quantities scale with $\|\s\|_\op$, we use this
		normalization for the sake of simplifying presentation.}
	Thus, the starting value of our potential is $d^{O(p)}$, and requires
$\widetilde{O}(p/\eps)$ rounds of filtering (each of which takes $O(pnd)$
	time since we need to do $p$ many matrix-vector products with $\s$), for a total runtime of
$\widetilde{O}(ndp^2/\eps)$.
	Now, our algorithm requires $p$ to be at least $\log(d)/\eps$ to distinguish
	between eigenvalues that differ by a ($1+\eps$)-factor, as opposed to $p =
O(\log(d))$ in \cite{DKPP22}, resulting in a runtime of
$\widetilde{O}(nd/\eps^3)$.

The runtime can be improved to $\widetilde{O}(nd/\eps^2)$ by noting that such fine differences in the eigenvalues become relevant only at the last  stage (``certification stage'') of our algorithm. Until then, we can use a smaller value of $p$.
	More formally, for any $p'\geq \log d$, we can still
	decrease the potential multiplicatively 
	until our stopping condition is satisfied. 
	Although this no longer certifies that we have a good solution, we can use it to upper bound the potential by $\poly(d)$. 
    Thus, we run our method in multiple stages: 
    in stage $k$, we take $p_k = 2^k \log d$ and
	reduce the potential until it becomes $d^C$, which takes
$\Tilde{O}(ndp_k/\eps)$ time.
	By the time we reach $p = \log(d)/\eps$, since that stage also starts with
	potential at most $d^C$, the algorithm terminates in $\Tilde{O}(nd/\eps^2)$
	time, overall.
	\looseness=-1

	Finally, as in \cite{DKPP22}, the above described algorithm 
        can be turned into a single-pass streaming algorithm.
	Since we use a small number of filters, where
	each filter removes all samples $x$ with $|v^\top x| > T$ (for some carefully
	chosen $v$ and $T$), we can implement this with low memory by storing just the
$v$'s and $T$'s of the previously created filters, along with a minimal set of
temporary variables for necessary calculations.

\subsection{Related Work}

Our work is situated within the field of algorithmic robust statistics, where
the goal is to develop computationally-efficient algorithms for
high-dimensional estimation problems that are robust to outliers.
Since the publication of \cite{DKKLMS16,LaiRV16},
computationally-efficient robust algorithms have been developed for many tasks,
including mean estimation~\cite{KotSS18,CheDG19,DonHL19,DepLec19,DiaKP20,HopLZ20, HuRein21, CheTBJ20, ZhuJS20}, sparse estimation~\cite{BDLS17,DiaKKPS19,CheDKGGS21,DiaKKPP22-colt,DiaKLP22}, covariance estimation~\cite{CheDGW19}, linear regression~\cite{KliKM18, DiaKS19, PenJL20,CheATJFB20}, clustering and list-decodable learning~\cite{DKS18-list,HopLi18,KotSS18,KarKK19,LiuMoi21,BakDJKKV22},  and
stochastic optimization~\cite{DiaKKLSS19,PraSBR20}.
We refer the reader to the recent book \cite{DiaKan22-book} for an overview of this field. 

The study of near-linear time algorithms for high-dimensional robust 
estimation tasks was initiated by~\cite{CheDG19} in the context of mean 
estimation. Since this work, a number of improvements and generalizations 
have been obtained in various contexts. 
The most related line of works to the current paper is the work 
on high-dimensional robust PCA, starting from~\cite{XuCM13}. While this 
early work did not obtain dimension-independent error guarantees in 
polynomial time, polynomial-time algorithms with near-optimal dependence on
$\epsilon$ in the error guarantee were subsequently developed in
\cite{JamLT20,KonSKO20}.
We note that these algorithms have runtimes at least $\Omega(nd^2)$.
Additionally, \cite{JamLT20}  developed a nearly linear time algorithm for
robust PCA; however, their algorithm runs in nearly  linear-time 
only in certain cases, and the dependence on $\epsilon$ in the error guarantee is sub-optimal
and dimension-dependent.
Finally, our streaming algorithm is inspired by \cite{DKPP22}, who presented
the first
streaming algorithms for high-dimensional robust mean estimation and related tasks.

We remark that robust PCA is closely related to the  problem of robust covariance estimation in the operator norm, where the algorithm is required to output an estimate $\widehat{\s}$ given a corrupted set of samples such that $\|\s -\widehat{\s}\|_\op$ is small.
It is easy to see that the top eigenvector of $\widehat{\s}$ satisfies the guarantees of robust PCA.
In this paragraph, we restrict our attention to Gaussian data and when $\epsilon$ is a small enough absolute constant.
Although the sample complexity of robust covariance estimation in operator norm is still $\tilde{O}(d)$, \cite[Theorem 1.5]{DiaKS17} has shown that the problem is computationally hard (in the Statistical Query model) unless one takes $\Omega(d^2)$ samples.  
With $\Omega(d^2)$ samples, one can estimate the covariance in the relative Frobenius norm (and thus in operator norm); see the discussion in \cite{DiaKS17}.
Thus, robust PCA is an easier task, both computationally and statistically, than robust covariance estimation in operator norm, while still being useful. 

Finally, we emphasize that the focus of our work (and that of
\cite{JamLT20,KonSKO20}) is somewhat orthogonal to other works
with similar terminology~\cite{CanLMW11,MarMA18,BieJSY22}, as the 
definition of robustness is significantly different.
\section{Preliminaries}
\label{sec:prelim}
\paragraph{Notation}
We denote $[n]:= \{1,\ldots,n\}$.
For a vector $v$, we let $\|v\|_2$ and $\|v\|_\infty$ denote its $\ell_2$ and
infinity-norm respectively.
We use boldface capital letters for matrices.
We use $\mathbf{I}$ for the identity matrix.
For two matrices $\vec A$ and $\vec B$, we use $\tr(\vec A)$ for the trace of $\vec A$, and
$\langle \vec A,\vec B \rangle:=\tr(\vec A^\top \vec B)$ for the trace-inner product.
We use $\| \vec A \|_\fr, \| \vec A \|_\op, \| \vec A \|_p$ for the
Frobenius, operator, and Schatten $p$-norm of a matrix $\vec A$ for $p\ge 1$.
In particular, if $\vec A$ is a $d\times d$ symmetric matrix with eigenvalues $\lambda_1, \dots, \lambda_d $, then $\|\vec A\|_p := (\sum_i |\lambda_i|^p)^{1/p}$. 
We say that a square symmetric matrix $\bA$ is PSD (positive
semidefinite), and write $\bA \succeq 0$, if $x^\top \bA x \geq 0$ for all $x \in \R^d$.
We write $\bA\preceq \bB$ when $\bB-\bA$ is PSD.
We write $a \lesssim b$, when there exists an absolute universal
constant $C>0$ such that $a \leq C b$.

For a distribution $D$ over a domain $\cX$ and a weight function $w:\cX \to
  \R_+$,
we use $D_w$ to denote the distribution over $\R^d$ with pdf $D_w(x):=
  D(x)w(x)/\int D(x)w(x)$. 
  We denote the second moment of $D$ by $\vec \Sigma_D := \E_{X \sim D}[XX^\top]$.
  We will repeatedly use that $\E_{X \sim D}[\|\vec U X\|_2^2] = \langle \vec U^\top \vec U, \s_D \rangle $.

We use the following notion of subgaussianity that was also used in 
\cite{JamLT20}.
\begin{definition}
\label{def:subgaussianity}
For a parameter $r \geq 1$, we say a distribution $D$ on $\R^d$ with mean zero and covariance $\s$ is $r$-subgaussian if for all unit vectors $v$ and $t>0$, 
\begin{align*}
    \E_{X \sim D}[\exp(t v^\top X)] \leq \exp\left(t^2 r (v^\top \s v) /2\right)
\end{align*}
\end{definition}
Observe that this notion of subgaussianity is stronger than  \cite[Definition 3.4.1]{Vershynin18} because the sub-gaussian proxy in the direction $v$ scales with $v^\top \s v$.

\subsection{Miscellaneous Facts}
We now collect some facts that we shall use in the paper.

\paragraph{Linear Algebraic Facts}
We first begin by noting the following useful properties of trace inner product for PSD matrices:

\begin{restatable}{fact}{TRACEINEQ}
  \label{fact:trace_PSD_ineq}
  If $\bA,\bB,\bC$ are symmetric $d \times d$ matrices with $\bA \succeq 0$ and
  $\bB \preceq \bC$ then $\tr(\bA \bB) \leq \tr(\bA \bC)$.
\end{restatable}

A proof of the following fact can be found, e.g., in \cite{JamLT20}.

\begin{fact}
  \label{fact:trace_PSD_ineq2}
  Let $\vec A,\vec B$ be PSD matrices with $\vec B \preceq \vec A$ and $p \in
    \N$.
  Then,
  $ \tr\left(\vec B^p \right) \leq \tr\left( \vec A^{p-1} \vec B \right)$.
\end{fact}
The following fact shows that a PSD matrix continues to be PSD if both pre-multiplied and post-multiplied:
\begin{fact}\label{fact:PSDness}
  If $\vec A \in \R^{d \times d}$ is a PSD matrix, for any other  $\vec B  \in \R^{d \times d}$ it holds $\vec B \vec A \vec B^\top \succeq 0$.
\end{fact}
\begin{proof}
  For any vector $x$, denoting $y := \vec B^\top x$, we have that $x^\top (\vec B \vec A \vec B^\top) x = y^\top \vec A y \geq 0$.
\end{proof}
We shall use the following fact to compare $\vec A^m$ for some $m\geq 1$ and $\vec A$ along certain directions:
\begin{fact}
  \label{fact:psd-power}
  Let $\vec A$ be a PSD matrix.
  Then for any unit vector $x$ and $m \in \N$, $x^\top \vec A^m x \geq (x^\top
    \vec A x)^m$.
\end{fact}
\begin{proof}
  By spectral decomposition, $\vec A$ is equal to $\sum_i \lambda_i u_iu_i^\top$
  for $\lambda_i \geq 0$, and orthonormal vectors $u_i$.
  Furthermore, $\vec A^m = \sum_i \lambda_i^m u_iu_i^\top$.
  For any unit vector $x$, the orthonormality of $u_i$'s imply that $\sum_i
    (u_i^\top x)^2 = 1$.

  We now define the non-negative random variable $Z$ that is equal to $\lambda_i$
  with probability $(u_i^\top x)^2$, which is a valid probability assignment
  since it is non-negative and sums up to $1$.
  In this definition, we have that $\E[Z] = \sum_i \lambda_i (u_i^\top x)^2 =
    x^\top \vec A x$.
  Moreover, $\E[Z^m] = \sum_i \lambda_i^m (u_i^\top x)^2= x^\top (\sum_i
    \lambda_i^m u_iu_i^\top)x = x^\top \vec A^m x$.
  Thus, the desired result is equivalent to saying $\E[Z^m] \geq (\E [Z])^m$ for
  the non-negative random variable $Z$, which is satisfied by Jensen's
  inequality.
\end{proof}

We now turn our attention to Schatten norms: for a symmetric matrix $\vec A$, we have that $\|\vec A\|_2
  = \|\vec A\|_\fr$, $\|\vec A\|_\infty = \|\vec A\|_\op$, and for a PSD matrix $\vec A$,
$\|\vec A\|_1 = \tr(\vec A)$.
Schatten norms satisfy H{\" o}lder inequality for trace inner product:
\begin{fact}[H{\"o}lder Inequality for Schatten Norms]
  \label{fact:holder-schatten}
  Let $\vec A,\vec B$ be two matrices with dimensions $m \times n$
  Let $p \geq 1$.
  Define $q = \frac{p}{p-1}$ if $p > 1$ and $\infty$ otherwise, then 
  $ \langle \vec A, \vec B \rangle \leq \|\vec A \|_p \|\vec  B\|_q$. In particular, $\langle \vec A, \vec B \rangle \leq \|\vec A\|_\fr\|\vec B\|_\fr $ and $\langle \vec A, \vec B \rangle \leq \|\vec A\|_\op \|\vec B\|_1 $.
\end{fact}
Moreover, Schatten-$p$ norms decrease with $p$ and are close to each other if
$p$ is large.
\begin{fact}
  \label{fact:normReln}
  If $\vec A \in \R^{d\times d}$ is symmetric and $p \geq 1$, the Schatten norms
  of $\vec A$ satisfy the following: 
  $\|\vec A\|_{p+1} \leq \|\vec A\|_p \leq \|\vec A\|_{p+1}d^{\frac{1}{p(p+1)}}.
  $
\end{fact}
In particular, we will frequently use the following special case:
\begin{fact}\label{fact:fr-trace}
  If $\vec A$ is a PSD matrix,  $\|\vec A \|_\fr^2 \leq \tr(\vec A)^2$.
\end{fact}
\begin{proof}
Note that $\|\vec A\|_\fr = \|\vec A\|_2$ and $\tr(\vec A) = \|\vec A\|_1$ since $\vec A$ is PSD. The result then follows by \Cref{fact:normReln}. 
\end{proof}

\paragraph{Probability Facts}
The quadratic polynomial of Guassians will play an important role in our analysis, i.e., $z^\top \vec A z$ for $z  \sim \cN(0, \bI)$, and we will repeatedly use the following fact.
\begin{fact}
  \label{fact:var-quadratic}
  For any symmetric $d\times d$ matrix $\vec A$, we have $\Var_{z \sim \cN(0,\bI)}[z^\top
      \vec A z]=2 \|\vec A \|_\fr^2$.
  \label{fact:quadratics-approximation}
  If $\vec A$ is a PSD matrix, then for any $\beta > 0$, it holds 
  $\pr_{z \sim \cN(0,\bI)}[ z^\top \vec A z \geq \beta \tr(\vec A) ] \geq
    1-\sqrt{e \beta}$.%
\end{fact}
For the streaming section of the paper, we will use the following result on the concentration of the empirical second moment matrix:
\begin{fact}[see, e.g, \cite{Ver10}] \label{fact:ver_cov}
Consider a distribution $D$ on $\R^d$ that is supported in an $\ell_2$-ball of radius $R$ from the origin. 
Let $\vec \Sigma$ be its second moment matrix and  
$\vec{\Sigma}_N=(1/N) \sum_{i=1}^N X_iX_i^\top$  be the empirical second moment matrix 
using $N$ i.i.d.\ samples $X_i \sim D$. There is a constant $C$ such that for any $0<\eps<1$ 
and $0<\tau <1$, if $N > C \eps^{-2} \| \vec \Sigma \|_2^{-1} R^2  \log(d/\tau)$, we have that 
$\| \vec \Sigma - \vec \Sigma_N \|_\op \leq \eps \| \vec \Sigma \|_\op$, 
with probability at least $1-\tau$.
\end{fact}
Finally, we record the guarantee of power iteration:
\begin{fact}[Power iteration]
  \label{fact:power-iter}
  For any PSD matrix $\vec A \in \R^{d \times d}$ and $\eps,\delta \in (0,1)$, if $p
    > \frac{C }{\eps}\log(d/(\eps \delta))$ for a sufficiently large constant $C$,
  and $u:= \vec A^p z$ for $z \sim \cN(0,\bI)$, then
  \begin{equation*}
    \pr \left[ u^\top \vec A u/\|u\|_2^2 \geq (1-\eps)\|\vec A\|_\op
      \right] \geq 1- \delta \;.
  \end{equation*}
\end{fact}

\subsection{Stability Condition on Inliers}\label{sec:stability-main-body}

Recall that in our notation we use $\s_D=\E_{X\sim D}[XX^\top]$ for the
second moment of a distribution $D$, and $D_w(x):=
  D(x)w(x)/\int D(x)w(x)$ for the distribution $D$ re-weighted by $w(x)$.
Our algorithm will rely on the following condition. \looseness=-1

\begin{definition}[Stability Condition]
  \label{def:stability}
  Let $0<\eps<1/2$ and $\eps \leq \gamma < 1$.
  A distribution $G$ on $\R^d$ is called $(\eps,\gamma)$-stable with respect to a
  PSD matrix $\s \in \R^{d \times d}$, if for every weight function $w : \R^d \to
    [0,1]$ with $\E_{X \sim G}[w(X)] \geq 1-\eps$, the weighted second moment
  matrix,
  $\os_{G_w}:= \E_{X \sim G}[w(X) XX^\top]/\E_{X \sim G}[w(X)]$, satisfies that
    $(1 - \gamma) \vec \s \preceq \os_{G_w} \preceq (1 + \gamma) \vec \s \;$.
\end{definition}
In particular, the second moment matrix is \emph{stable} under deletion of $\epsilon$-fraction. 
Some remarks are in order:
(i) The definition above is 
intended for distributions with zero mean; This can be assumed without loss of generality
since we can always work with pairwise differences, 
(ii) The uniform distribution over a set of
  $n=Cd/(\eps^2\log(1/\eps))$ i.i.d.\ samples from a $O(1)$-subgaussian distribution (cf.\ \Cref{def:subgaussianity}) with covariance $\s$ is w.h.p.\ $(\eps,\gamma)$-stable with $\gamma=O(\eps\log(1/\eps))$~  \cite[Lemma 11]{JamLT20}, and
  (iii) As stated in \Cref{thm:near-linear-main}, the algorithm takes as input an $\eps$-corrupted set of samples from a subgaussian distribution.
  However, for notational convenience, we will consider the input to be a distribution $P{=}(1{-}\eps)G{+}\eps B$ where $G$ is an (unknown) stable distribution and $B$ is arbitrary.  $P$ is meant to be the uniform distribution over the $\eps$-corrupted set of samples, $G$ will be the part from the remaining inliers, and $B$ that of outliers.\footnote{To be precise, the claim mentioned in (iii) needs a proof; see,
  e.g., Lemma 2.12 in \cite{DKPP22}.}
  We emphasize that the identity of outliers is unknown to the algorithm.

Our algorithm will remove points iteratively, so we will use binary weights $w(x) \in \{0,1\}$ to distinguish between the points that we keep $(w(x) = 1)$ and the ones that have been removed $(w(x) = 0)$. 
Throughout the run of our algorithm, we will ensure that the weights satisfy the following setting:
\begin{setting}
  \label{setting:main}
  Let $0<20\eps\leq \gamma < \gamma_0$ for a small constant
  $\gamma_0$.
  Let  $P$ be a mixture distribution $P=(1-\eps)G + \eps B$ on $\R^d$, where $G$ is
  $(20\eps,\gamma)$-stable distribution with respect to a PSD $d\times d$ matrix
  $\s$, and $\vec B$ is arbitrary.
  Let $w:\R^d \to \{0,1\}$ be a weight function with
  $\E_{X \sim G}[w(X)] \geq 1-3\eps$.
\end{setting}
As we will use projections of points to filter outliers, we need the following implication of stability, concerning the average value of these projections (and thresholded projections), as well as their quantiles and their robust estimates.\

\begin{restatable}{lemma}{StabImplications}
  \label{lem:stability-implications}
  In \Cref{setting:main}, define the following: (i) $g(x) := \|\vec U x\|_2^2$ for some $\vec U \in \R^{m \times d}$, (ii) $L$ be the top $3\eps$-quantile of $g(x)$ under $P_w$ ($P$ re-weighted by $w$) and $S:=\{ x \; : \; x > L\}$, (iii) $\widehat\sigma := \E_{X \sim P}[w(X) g(X) \Ind(g(X) \leq L )]$.
  Then,
  \begin{enumerate}[label = (\roman*)]%
     \item \label{it:stability_multidim} $\left|\E\limits_{X
              \sim G}[w(X)g(X)]-\left\langle\vec U^\top \vec U, \s \right\rangle\right| \leq 2\gamma
            \left\langle\vec U^\top \vec U, \s \right\rangle$,

    \item \label{it:goodscores} $\E_{X \sim G}[w(X)g(X)\1\{X \in S \}] \leq 2.35 \gamma
            \left\langle \bU^\top \bU, \s \right\rangle$,

    \item \label{it:bound-quantile} $L\leq 1.65(\gamma/\eps) \langle \bU^\top \bU,
            \s \rangle $,

    \item \label{it:proj-var} $\left|\widehat{\sigma}-  \left\langle \bU^\top \bU, \s  \right\rangle  \right| \leq 4\gamma \left\langle \bU^\top \bU, \s  \right\rangle $,

    \item  \label{it:initialization} $\pr_{X \sim G}\left[ \|X\|_2^2 > 2 (d/\eps)\|\s\|_\op \right] \leq \eps$.
  \end{enumerate}
\end{restatable}
As discussed earlier, the goal is to filter out points until a
top eigenvector $u$ of the data set's empirical second moment $\os_{P_w}$
approximately maximizes the true variance.
As in previous work, an easy way to detect when this happens is by comparing
the empirical variance along $u$ to the true variance along the same direction
(which although unknown, it can be
easily robustly estimated using Item 4 above).
The following is implicit in \cite{XuCM13,JamLT20,KonSKO20}:
\begin{lemma}[Basic Certificate Lemma]
  \label{lem:basic-cert}
  Consider  \Cref{setting:main}.
  Let $u$ be a unit vector with $u^\top\os_{P_w}u \geq (1 - O(\gamma))\|\os_{P_w}\|_\op$. If $u^\top \s u \geq
    (1-O(\gamma)) u^\top \os_{P_w} u$, then we have that $u^\top \s u/\|u\|_2^2
    \geq (1-O(\gamma)) \| \s \|_\op$.
\end{lemma}

For completeness, its proof can be found in \Cref{app:stability}, where the lemma is restated as \Cref{lem:certificate}.

\subsection{Filtering}\label{sec:filtering-main-body}
We now recall the standard filtering routine from algorithmic robust statistics
literature.
The filter uses a score $\tau(x)$ for each point $x$, indicating how atypical
the point is for the distribution.
Given a (known) upper bound $T$ for the average scores over only the inliers $\E_{X \sim G}[w(X)\tau(X)]$, if the average score of all (corrupted)
points, $\E_{X \sim P}[w(X)\tau(X)]$, is much bigger than $T$, then points with large scores are likely to be outliers.
Thus, \Cref{alg:filter-main} removes all points with scores greater than a (random) threshold.\footnote{Consequently, for each point $x$, $\P( x \text { is removed}) \propto \tau(x)$.
} We also note the following:
\begin{enumerate}[label = (\roman*)]%
  \item Instead of storing a bit $w(x)$ for every $x$, we can store a more succinct description of the filters, i.e., just the $r_\ell$'s of  \Cref{line:random-thr}, and calculate $w(x)$ whenever needed.
        This modification is amenable to the streaming setting. \looseness=-1
  \item Since every time we remove all points with $\tau(x)>r_\ell$ and $r_\ell \sim \cU([0,r_{\ell-1}])$, this (in expectation) halves the range of $\tau(x)$, and thus \Cref{alg:filter-main} terminates after $O(\log(\max_{x}\tau(x)/\min_x \tau(x)))$ steps in expectation.\looseness=-1
\end{enumerate}

The following result states that \Cref{alg:filter-main} removes  more outliers than inliers in expectation.
\begin{algorithm}
  \caption{\textsc{HardThresholdingFilter}}
  \label{alg:filter-main}
  \begin{algorithmic}[1]
    \State \textbf{Input}:
    Distribution $P=(1-\eps)G+\eps B$, weights $w$, scores $\tau$, parameters $\widehat{T},R$.
    \State $w_{0}(x) \gets w(x)$, and $r_0 \gets R$, $\ell \gets 1$.
    \State \textbf{while }{$\E_{X \sim P}[w(X)\tau(X)] > \frac{5}{2}\widehat{T}$}
    \label{line:stopping-condd}
    \State $\,\,\,\,\,$ Draw $r_\ell \sim \cU([0,r_{\ell-1}])$.
    \label{line:random-thr}
    \State $\,\,\,\,\,$ $w_{\ell+1}(x) \gets w_{\ell}(x) \cdot \Ind( \tau(x) > r_\ell)$.
    \label{line:new-filter}
    \State $\,\,\,\,\,$ $\ell \gets \ell + 1$.
    \State \textbf{return} $w_{\ell}(x)$.
  \end{algorithmic}
\end{algorithm}

\begin{lemma}[Guarantees of Filtering]
  \label{lem:filter_guar-main-body}
  Let $T$ be such that $(1{-}\eps)\E_{X \sim
      G}\left[w(X)\tau(X)\right] < T$ and $\widehat{T}$ such that $|\widehat{T}-T| <
    T/5$.
  Denote by $F$ the randomness of \textsc{HardThresholdingFilter}  (i.e., the collection of the random
  thresholds $r_1,r_2\ldots$ used).
  Let $w'$ be the weight function returned.
  Then, 
  \begin{enumerate}[label = (\roman*)]
      \item $\E_{X \sim P}[w'(X)\tau(X)] \leq 3 T$ almost surely, and 
      \item $\E_{F}[ \eps \E_{X \sim B}[w(X) - w'(X)] ]$
  ${>} (1{-}\eps) \E_{X \sim G}[w(X) - w'(X) ] ]$.
  \end{enumerate}
\end{lemma}

The proof of \Cref{lem:filter_guar-main-body} can be found in \Cref{sec:filtering}.
Our main algorithm will repeatedly call \textsc{HardThresholdingFilter} with
appropriate $\tau(x)$ and other parameters.
From a technical standpoint, if the second part of
\Cref{lem:filter_guar-main-body} (which states that more mass is removed from
outliers than inliers) were true deterministically, we would have that $\E_{X
    \sim G}[w(X)]\geq 1-\eps$ no matter how many times the filter is called.
This would mean that \Cref{setting:main} would be maintained throughout the main
algorithm.
However, part (ii) of \Cref{lem:filter_guar-main-body} holds only in
expectation (with respect to $F$), but one can still show via a martingale argument that a relaxed condition of 
$\E_{X \sim G}[w(X)]\geq 1-3\eps$ will still be true throughout the main algorithm with high probability; further details can be found in \Cref{sec:filtering} (\Cref{lem:martingale}).
For the first reading, one can simply think that
\Cref{setting:main} will always hold. %

\section{Robust PCA in Nearly-Linear Time}
\label{sec:multi-pass}

\paragraph{Organization}
In this section we prove \Cref{thm:multi-pass}, which is a more general version of the \Cref{thm:near-linear-main} that was stated in \Cref{sec:our-results}.
To simplify the presentation, \Cref{sec:advanced-cert-lemma,sec:pot-reduction} will focus on establishing a weaker version of \Cref{thm:multi-pass}, with a runtime of $\widetilde{O}(nd/\gamma^3)$ (where $\gamma$ is the stability parameter; recall that $\gamma=O(\eps\log(1/\eps))$ for subgaussians). 
This still improves upon prior work in terms of both runtime and
error guarantee, while effectively showing key aspects of our approach.
Next, in \Cref{sec:improvement}, we will outline an improvement that reduces the runtime to $\widetilde{O}(nd/\gamma^2)$. Finally, in \Cref{sec:formal_proof}, we will provide a formal proof of \Cref{thm:multi-pass} with all the details.

We now present the main result of our paper, which establishes \Cref{thm:near-linear-main} since $\gamma = \tilde{O}(\epsilon)$ for subgaussian distributions:
\begin{restatable}{theorem}{FEWITER}
  \label{thm:multi-pass}
  Let $\gamma_0$ be a sufficiently small positive constant.
  Let $d>2$ be an integer, and $\eps, \gamma \in (0,1)$ such that $0<20 \eps<\gamma < \gamma_0$.
  Let $P$ be the uniform distribution over a set of $n$ points in $\R^d$, that
  can be decomposed as $P=(1-\eps)G+ \eps B$, where $G$ is a
  $(20\eps,\gamma)$-stable distribution (\Cref{def:stability}) with respect to a
  PSD matrix $\s \in \R^{d \times d}$.
  There exists an algorithm that takes as input $P,\eps,\gamma$, runs for $O(\frac{n d}{\gamma^2}
    \log^4(d/\eps))$ time, and with probability at least $0.99$, outputs a unit
  vector $u$ such that $u^\top \s u \geq (1-O(\gamma))\| \vec \Sigma \|_\op$.
\end{restatable}

As mentioned earlier in \Cref{sec:our-tech,sec:prelim}, we run an iterative
algorithm, where in each iteration
$t$, we assign a score to each point $x$, and use these scores to removes
points.
These scores are usually projections of points along a direction, where
outliers have much larger projections than the inliers.
Under the stability condition, we can reliably filter outliers as long as the
empirical (corrupted) variance is $(1 + C \gamma)$ times larger than the true
variance in a direction.

As each round of filtering decreases the variance, we see that $\s_t$ should
decrease with $t$ (after appropriate normalization).
As noted in prior work, using scores that involve projections of the samples on
the top eigenvector $v_t$ of the empirical second moment $\os_t$ does not
necessarily make good progress because the filtering removes points only in a
single (fixed) direction of the largest variance (cf.\
\Cref{sec:our-tech} for details).
Instead, one needs to filter along many of these large variance directions
simultaneously (on average), which we do by filtering along the
direction $v_t = \vec M_t z$ for $z \sim \cN(0,\bI)$ and $\vec M_t = \os_t^p$
for a large integer $p = \frac{C \log(d/\gamma)}{\gamma}$.

To track progress of the algorithm, we use the potential function $\phi_t$ :=
$\tr(\os_t^{2p+1})$.
This potential  tracks the contribution from all large eigenvalues (and
not just the largest one), and has been used in prior work for robust mean
estimation.
In our setting, we will show that (I) small enough $\phi_t$
(e.g., $\phi_t \leq \|\s\|_\op^{2p+1}/\poly(d^p)$) implies that $\s$ and
$\s_t$ share the leading eigenspace \emph{without any restriction on the spectrum
	of $\s$}, which in turn, certifies that a good solution vector can be produced, and (II) if the spectrum of $\s$ and $\s_t$ disagree, then we can
decrease
$\phi_t$ multiplicatively (on average).
In particular, we will show that $\phi_{t+1} \leq (1{-}\gamma)\phi_t$ on
average.
Combining this with the fact that a simple pruning can always ensure $\phi_0
	\leq
	\poly(d^p)\|\s\|_\op^{2p+1}$, we can obtain $\phi_t <
	\|\s\|_\op^{2p+1}/\poly(d)$ after just $t=O(p\log(d)/\gamma)$
rounds.

We now explain the items (I) and (II) above: \Cref{sec:advanced-cert-lemma}
formalizes the notion of ``shared leading eigenspaces'', and
\Cref{sec:pot-reduction} outlines the decrease of the potential function.
The formal proof is then given in \Cref{sec:formal_proof}.

\subsection{Advanced Certificate Lemma}
\label{sec:advanced-cert-lemma}
We begin by summarizing the algorithmic framework of
\cite{JamLT20}.
In particular, we highlight the roadblocks in their framework to remove the
eigenvalue separation condition, and our proposed modifications that rely on
formalizing the notion of ``shared leading eigenspaces''.

Roughly speaking, their iterative algorithm stops at an iteration $t$ such that
$\|\vec \s_t\|_p^p \leq (1 + C \gamma)^p\|\vec \Sigma\|_p^p$,
where $\gamma$ is
the stability parameter.%
\footnote{While comparing norms under stability, $(1+C\gamma)$ is necessary.}
Then, their algorithm tries the top-$m$ eigenvectors of $\s_t$ for some $m \in \N$ and
returns the (normalized) eigenvector $v$ that attains the best possible $v^\top
	\s v$ (which can be estimated up to small error); see \Cref{lem:basic-cert}.
In particular, the last step takes $\widetilde{O}_\gamma(ndm)$ time.
However, as the following example shows (by taking $r = d/2$ as $m
	=\Omega(d)$), this approach cannot give a nearly-linear time algorithm.
\begin{example}[Hard Example for Schatten $p$-Norm Stopping Condition]
	\label{ex:hard-schatten}
	Let $\s$ be a PSD projection matrix of rank $r$, then $\s_t = \bI$ satisfies
	the condition $\|\s_t\|_p^p \leq (1 + C\gamma)^p\|\s\|_p^p$
	as long as $p \gtrsim
		\log(d/r)/\gamma$. %
	 However, any deterministic algorithm to identify a good direction of $\s$ from
	the eigenvectors of $\s_t$ must take $m = \Omega(d-r)$.
\end{example}
Thus, \cite{JamLT20} imposes an additional assumption that
the largest eigenvalue and the $m$-th
largest eigenvalue of $\s$ are separated for some $m$.
Consequently, their result and their version of certificate are inherently
restricted to have dependence on $m$; see Proposition 4 therein.

We make two crucial changes in our algorithm (for technical reasons, we also
make a change of variable to use $2p+1$ instead of $p$).
The first change is to strengthen the stopping condition: instead of reducing
the Schatten norm, which is equivalent to $\langle \s_t, \s_t^{2p} \rangle \leq
	(1 + C\gamma)^{2p+1} \langle \s, \s^{2p} \rangle$, we
stop at a stronger condition\footnote{Please see \Cref{sec:comparison-two-stopping-conditions} for why this is stronger.
} of $\langle \s_t, \s_t^{2p}\rangle \leq (1 + C\gamma) \langle \s, \s_t^{2p}\rangle$.
That is, we stop whenever the empirical variance in directions of $\s_t^p$,
$\langle \s_t, \s_t^{2p} \rangle $, is comparable to the true variance,
$\langle \s, \s_t^{2p}\rangle$, up to $(1 + C\gamma)$ factor (cf.\ \Cref{lem:stability-implications} with $\vec U = \s_t^p$).
Observe that this $(1 + C \gamma)$ factor is the best possible factor
that we
can achieve while comparing the robust and empirical variances.
We are able to achieve this stronger stopping condition because until this
condition is satisfied, our algorithm will remove outliers along $\s_t^{p}$,
which will decrease the potential by a multiplicative factor (this will be the topic of the next subsection). 

Still, the hard example from \Cref{ex:hard-schatten} continues to satisfies
this stronger condition: one must take $m=\Theta(d \gamma)$ if $r =
	d/(1+\gamma)$.
Thus, any deterministic algorithm would require $m = \Omega(nd^2\gamma)$ time,
which is quadratic in $d$.
Our second main insight is to randomize the output of the algorithm.
That is, even though, a deterministic algorithm would need to try
$\Omega(\gamma d)$ many top eigenvectors of $\s_t$ before finding a high-variance
direction of $\s$,
a top eigenvector sampled randomly from the spectrum of $\s_t$ will
suffice with a large constant probability.
In particular, we take a Gaussian
sample and iteratively multiply it by $\s_t$ in the style of power
iteration.
Formally, we prove the following:
\begin{restatable}{lemma}{ADVANCEDCERT}
  \label{lem:certificate1}
  \label{lem:certificate1-main-body}
    Let a sufficiently large constant $C$. Let $0<20\eps\leq \gamma < \gamma_0$ for a sufficiently small absolute constant $\gamma_0$.
  Let $P=(1-\eps)G+ \eps B$, where $G$ is a $(20\eps,\gamma)$-stable distribution
  with respect to $\s$.
  Let $w:\R^d \to [0,1]$ with
  $\E_{X \sim G}[w(X)] \geq 1-3\eps$, $p > C\log(d/\gamma)/\gamma$, $\vec M := (\E_{X\sim P} [w(X)XX^\top])^p$,  and assume
  $\langle \s, \vec M^2 \rangle \geq (1-250\gamma)\langle \os_{P_w}, \vec
    M^2 \rangle$.
    If $u:=\vec
    M z$ for a random $z \sim \cN(0,\bI)$,
  then with probability at least $0.9$ (over the random selection of $z$) we have that:
  \begin{enumerate}[label = (\roman*)]
      \item $  u^\top \os_{P_w}
      u/\|u\|_2^2 \geq (1-\gamma)\| \os_{P_w} \|_\op$, 
      \item $u^\top \s u > (1-O(\gamma)) u^\top \os_{P_w} u$.
  \end{enumerate}
\end{restatable}
The above result states that whenever $\langle \s, \vec M^2 \rangle \geq
	(1-250\gamma)\langle \os_{P_w}, \vec
	M^2 \rangle$ for large $p$,
 we can use \Cref{alg:sol-gen} to obtain
a vector $u$ that will satisfy the ``basic certificate'' from \Cref{lem:basic-cert}
and thus be a good solution.
Our main algorithm will thus call \Cref{alg:sol-gen} to output a vector.

\begin{algorithm}[h]
	\caption{\solGenerator}
	\label{alg:sol-gen}
	\begin{algorithmic}[1]
		\State \textbf{Input}: Distribution $P$, weights $w$, parameters $\eps,\gamma,\delta$.

		\label{line:stopping_cond-basic}
		\State $y \gets  \os_{P_w}^{p'} g$ for $p' =   \frac{C}{\gamma}\log\left(\frac{d}{\gamma \delta}\right)$ and $g \sim \cN(0,\bI)$.
		\State $\widehat{r} \gets y^\top \vec \os_{P_w} y/\|y\|_2^2$.
		\label{line:strong-estimator-basic}
		\hfill\Comment{cf. \Cref{fact:power-iter}}
		\State Let $\vec M{:=} (\E_{X \sim P}[w(X)XX^\top])^p$ for $p{=}
			C\frac{\log(d/\gamma)}{\gamma}$.
		\State $u \gets \vec M z$ for $z \sim \cN(0,\bI)$.
		\State Find $\widehat{\sigma}_u$ such that $|\widehat{\sigma}_u - u^\top \s u| \leq 4 \gamma u^\top \s u$.
		\If{$\widehat{\sigma}_{u} \geq (1-C\gamma)u^\top \os_{P_w} u $ and $\frac{u^\top \os_{P_w} u}{\|u\|_2^2} \geq (1 - \gamma)  \widehat{r} $}\label{line:stoppingcond}
		\State $\,\,\,\,\,\,\,\,\,\,\,$
		\textbf{return} $u/\|u\|_2$.  \hfill\Comment{c.f. \Cref{lem:basic-cert}}
		\EndIf
	\end{algorithmic}
\end{algorithm}

\begin{proof}[Proof of \Cref{lem:certificate1-main-body}]
  The part of the conclusion that $u^\top \os_{P_w} u/{\|u\|_2^2} \geq
    (1-\gamma)\| \os_{P_w}\|_\op$, follows directly by the guarantee of power
  iteration (\Cref{fact:power-iter} with probability of failure
  being less than $0.001$).

  We now focus on showing that $u^\top \s u > (1-O(\gamma)) u^\top \os_{P_w} u$.
  Let $\vec A:= \s - (1-250\gamma)\os_{P_w}$ with high probability.
  We analyze the random variable $Y:=z^\top \vec M \vec A \vec M z$ for
  $z \sim \cN(0,\bI)$.
  The assumption $\langle \s, \vec M^2 \rangle \geq (1-250\gamma)\langle
    \os_{P_w}, \vec M^2 \rangle$ means that 
    $ \langle \vec M^2, \Sigma - (1 - 250 \gamma) \s_{P_w} \rangle = \langle \vec M^2, \vec A\rangle \geq 0$. Noting that $\E[Y] = \tr(\vec M \vec A \vec M) =  \langle \vec M^2, \vec A\rangle$, we have     $\E[Y] \geq 0$.
  Thus, if variance of $Y$ is not too large,
  we will have that $Y$ is not too negative with large probability.
  
 We prove the following technical result on the variance of $Y$ in \Cref{sec:variance-y-small}.
   \begin{restatable}[Variance of  $Y$]{claim}{ClVarYSmall}
    \label{cl:varinace-y-appendix}
    $\Var(Y) \lesssim \gamma^2  \|\s_{P_w}\|_\op^2 \|\vec M\|_\fr^4  $.
  \end{restatable}  
\begin{proof}[Proof sketch of \Cref{cl:varinace-y-appendix}]
By \Cref{fact:var-quadratic}: $\Var[Y] {=}2 \|\vec M \vec A \vec M
	\|_\fr^2$, and thus we need to show $\|\vec M \vec A \vec M
	\|_\fr \lesssim \gamma \|\s_{P_w}\|_\op \|\vec M\|_\fr^2 $.
	Recall the definition of $\vec A$.
	We can decompose $\os_{P_w}$ into the contribution from inliers 
	(which  is very close to $\s$ by stability) 
	and the contribution due to outliers
	$\s_B:=\E_{X \sim B}[w(X)XX^\top]$.
	Roughly speaking, we have $\os_{P_w} \approx (1 - \eps )\s + \eps \s_B$,
	which implies that $\vec A \approx \gamma \s + \eps \s_B$ as $\eps \leq
		\gamma$.

	By triangle inequality, $\|\vec M \vec A \vec M \|_\fr \leq
		\gamma\|\vec M \vec \s \vec M \|_\fr + \eps \|\vec M \s_B\vec M
		\|_\fr$.
	The first term is easy to upper bound using PSD property of $\s$ and $\vec M$ as follows:
	$
		\|\vec M \vec \s \vec M \|_\fr \lesssim \|\s\|_\op  \|\vec M^2\|_\fr
		\lesssim \| \os_{P_w} \|_{\op } \|\vec M\|_\fr^2$ by
	\Cref{fact:normReln} and stability.
	The second term is more involved.
	Using the decomposition of $\os_{P_w}$ in the condition $\langle \s, \vec
		M^2\rangle \geq (1 - O(\gamma))\langle \os_{P_w}, \vec M^2 \rangle$, we obtain
	that
	$\langle \s_B, \vec M^2 \rangle \lesssim (\gamma/\eps) \langle \s , \vec
		M^2 \rangle $.
	Finally, as $\s_B$ is a PSD matrix, we obtain
	$\eps\|\vec M \s_B\vec M \|_\fr \lesssim \eps\, \tr(\vec M
		\s_B\vec M ) = \eps\langle \s_B, \vec M^2 \rangle \lesssim
		\gamma \langle \os, \vec M^2 \rangle \leq \gamma \|\s\|_\op \|\vec M\|_\fr^2$, leading to the result.
\end{proof}

Using \Cref{cl:varinace-y-appendix} along with Chebyshev's inequality, we get that with probability at least $0.999$,  $Y$ satisfies the following lower bound:
\begin{align}
    Y
    &\geq \E[Y] -\sqrt{ 1000\Var_{z \sim \cN(0,\bI)}[Y]}
    = -O\left(\gamma \| \vec M \|_\fr^2\| \os_{P_w}\|_{\op}  \right) \;. \label{eq:chebysev}
\end{align}

  We need one more intermediate result. For the vector $u:=\vec M z$ for
  $z \sim \cN(0,\bI)$, an application of \Cref{fact:quadratics-approximation}  with $\vec A= \vec M^2$ and $\beta=10^{-4}/e$ yields 
  \begin{align}
      \pr\left[ \|\vec M\|_\fr^2/ \|u\|_2^2 \leq 1/\beta  \right] = \pr_{z \sim \cN(0,\vec I)} \left[  z^\top \vec M^2 z \geq \beta \tr(\vec M^2) \right] \geq 1 - \sqrt{e \beta } = 0.99 \;. \label{eq:application}
  \end{align}
  
  We can now complete the proof of \Cref{lem:certificate1}.
  In what follows,
  we condition on the following events: (i) $\|\vec M\|_\fr^2/ \|u\|_2^2 \leq 1/\beta \leq 10^5$, (ii)  the inequality of \Cref{eq:chebysev} holds, and (iii) the conclusion from part one of the lemma. 
  Since all of these events happen individually with probability $0.99$ at least,  we have that the intersection of the three events has probability at least $0.97$ by a union bound. 
  Recalling that $Y = z^\top \vec M \vec A \vec M z = u^\top \s u - (1-250\gamma) u^\top \os_{P_w} u$ and dividing by $\|u\|_2^2$ both sides of \Cref{eq:chebysev}, we have the following:
  \begin{align}
    \frac{u^\top \s u}{\|u\|_2^2} &\geq (1-250\gamma)  \frac{u^\top \os_{P_w} u}{\|u\|_2^2} - O\left(\gamma \frac{\| \vec M \|_\fr^2}{\|u\|_2^2}\| \os_{P_w} \|_{\op}  \right) \notag \\
    &\geq (1-250\gamma)  \frac{u^\top \os_{P_w} u}{\|u\|_2^2} - O\left( \gamma \| \os_{P_w} \|_{\op}  \right) \tag{using that the event from \Cref{eq:application} holds} \notag \\
    &\geq(1-250\gamma)  \frac{u^\top \os_{P_w} u}{\|u\|_2^2}- O\left( \frac{\gamma}{1-\gamma}\frac{u^\top \os_{P_w} u}{\|u\|_2^2} \right) \tag{ $u^\top \os_{P_w} u/\|u\|_2^2 \geq (1-\gamma)\|\os_{P_w}\|_\op$} \\
    &=(1-O(\gamma))  \frac{u^\top \os_{P_w} u}{\|u\|_2^2} \;. \tag{$\gamma<1/20$}
\end{align}
Thus, we have shown that the second bullet of \Cref{lem:certificate1} holds with probability $0.97$. The first bullet holds with probability $0.99$ when $C$ is appropriately large constant.
Thus, by a union bound, both bullets hold with probability at least $0.90$.
This completes the proof of \Cref{lem:certificate1}.
\end{proof}

\subsubsection{Proof of \Cref{cl:varinace-y-appendix}: Bound on the Variance of $Y$}
\label{sec:variance-y-small}
We now prove the following result on the bound of the variance that was omitted earlier.
 \ClVarYSmall*
\begin{proof}
As $Y = z^\top \vec M \vec A \vec M z$, \Cref{fact:var-quadratic} implies that $\Var[Y] =   2\| \vec M \vec A \vec M \|_\fr^2$. We will, thus, upper bound  $\| \vec M \vec A \vec M \|_\fr$ in the remainder of the proof.
    
Since $P = (1-\eps) G + \eps B$,
we have that 
 $\os_{P_w} = (1-\eps)\os_{P_w}^G +
    \eps \os_{P_w}^B$, with the components being $\os_{P_w}^G := \E_{X \sim G}[w(X) XX^\top]/\E_{X
      \sim P}[w(X)]$, and $\os_{P_w}^B := \E_{X \sim B}[w(X)
      XX^\top]/\E_{X \sim P}[w(X)]$. %
  Using this decomposition, we obtain the following expression for $\vec A$:
  \begin{align*}
    \vec A &= \s - (1-250\gamma)\left( (1-\eps)\os_{P_w}^G + \eps \os_{P_w}^B\right) \\
    &=\left( \s - \os_{P_w}^G(1-250\gamma-\eps + 250\gamma\eps) \right)   - \eps(1-250\gamma) \vec \os_{P_w}^B \;.
\end{align*}
  In order to simplify notation, let us define $ \vec \Delta := \s -
    \os_{P_w}^G(1-250\gamma-\eps + 250\gamma\eps)$, and
  $\widetilde{\eps}:=\eps(1-250\gamma)$.
  We thus have $\vec A = \vec \Delta - \tilde{\eps} \os_{P_w}^B $.
  Our approach will be to upper bound $\| \vec M \vec  \Delta \vec M\|_\fr$ and $\| \vec M \vec  \os_{P_w}^B \vec M\|_\fr$ separately. The following intermediate result, that easily follows from the stability the of distribution $G$, will be useful:
  \begin{claim}
    \label{cl:annoying-details} We have that (i) $0 \preceq \vec \Delta \preceq 270 \gamma \os_{P_w}$, (ii) $\s -
        (1-\eps)\os_{P_w}^G \preceq 1.2\gamma \s$, and (iii) $\s \preceq
      (1/(1-1.2\gamma))\os_{P_w}$.
  \end{claim}
  \begin{proof}
    We start with establishing the PSD property of $\vec \Delta$.
    Using the stability, we first obtain the following upper bound on $\os_{P_w}^G$.
    \begin{align}
    (1-4\eps) \os_{P_w}^G \preceq (1-\eps)(1-3\eps) \os_{P_w}^G  
    \preceq     \frac{\E_{X \sim P}[w(X)]}{\E_{X \sim G}[w(X)]} \os_{P_w}^G  = \E_{X \sim G_{w}}[XX^\top] \preceq (1+\gamma)\s \;, \label{eq:st}
\end{align}
    where the second inequality uses $w(x)\leq 1$ for the denominator and the assumption that 
    $\E_{X \sim P}[w(X)]\geq (1-\eps)\E_{X \sim G}[w(X)] \geq
      (1-\eps)(1-3\eps)$ for the numerator, and the last inequality uses the
    stability of $G$.
    Rearranging the inequality above, we obtain the following:,
    \begin{align*}
    \os_{P_w}^G \preceq \frac{1+\gamma}{1-4\eps} \s
    \preceq \frac{1}{1-\eps-250\gamma+250\eps\gamma} \s \;,
\end{align*}
    where the first step uses \Cref{eq:st} and the second is true for $\eps<\gamma/20$.
    Plugging this inequality in the definition of $\vec \Delta$ implies that $\vec \Delta \succeq 0$.

    We now show that $\vec \Delta \preceq 270 \gamma \os_{P_w}$.First, we note that
    \begin{align}
    \os_{P_w} &\succeq (1-\eps)  \os_{P_w}^G =(1-\eps) \frac{\E_{X \sim G}[w(X)]}{\E_{X \sim P}[w(X)]}\E_{X \sim G_{w}}[XX^\top] \notag \\
    &\succeq (1-\eps)(1-3\eps)(1-\gamma)\s \succeq (1-1.2\gamma)\s \;, \label{eq:step1}
\end{align}
    where the first inequality uses the decomposition $\os_{P_w} = (1-\eps)\os_{P_w}^G +
    \eps \os_{P_w}^B$,
    the second is a
    rewriting, the third uses stability of $G$ along with $\E_{X \sim
        G}[w(X)]\geq 1-3\eps$, and the last one uses that $\eps<\gamma/20$.
    We can now complete our proof:
    \begin{align}
    \vec \Delta &= \s - (1-\eps-250\gamma+250\gamma \eps) \os_{P_w}^G \notag\\
    &\preceq \s - (1-\eps-250\gamma+250\gamma \eps) (1-3\eps)(1-\gamma)\s \tag{using middle part of \Cref{eq:step1}}\\
    &\preceq \s - (1-\eps-250\gamma+250\gamma \eps)(1-1.15\gamma)\s \tag{$\eps \leq \gamma/20$}\\
    &\preceq 252\gamma \s \tag{using $\eps\leq \gamma/20$ and $\gamma \leq 1/20$} \\
    &\preceq \frac{252\gamma}{1-1.2\gamma}\os_{P_w}\preceq 270\gamma \os_{P_w} \tag{using \Cref{eq:step1} and $\gamma \leq 1/20$} \;.
\end{align}
    Combing this with $\vec \Delta \succeq 0$ completes the proof of the part (i) of the claim.
    Part (ii) of the claim, i.e., $\s - (1-\eps)\os_{P_w}^G \preceq 1.2 \gamma \s$ can be
      extracted from the middle part of \Cref{eq:step1} after some rearranging. Part (iii) also follows imediately from \Cref{eq:step1}.
  \end{proof}

  We are now ready to upper bound the variance of $Y$: Using \Cref{fact:var-quadratic}, we have that
  \begin{align}
\Var[Y] 
&= 2\| \vec M \vec A \vec M \|_\fr^2
\leq 4\| \vec M \vec \Delta \vec M \|_\fr^2 + 4\widetilde{\eps}^2 \left\| \vec M \os_{P_w}^B \vec M \right\|_\fr^2 \;,   \label{eq:variance}
\end{align}
  where the last inequality is triangle inequality and $(a+b)^2\leq 2a +2 b$.
  We upper bound each term separately.
  Using \Cref{cl:annoying-details}, we bound the first variance term as follows:
  \begin{align}
\| \vec M \vec \Delta  \vec M \|_\fr^2 
&= \tr( \vec M^2\vec \Delta \vec M^2 \vec \Delta ) \tag{using $\|\vec A\|_\fr^2 = \langle \vec A , \vec A \rangle $ for symmetric $\vec A$ and cyclic property of trace} \\
&\lesssim \gamma \tr( \vec M^2\vec \Delta \vec M^2 \vec \os_{P_w} ) \tag{using \Cref{cl:annoying-details} and \Cref{fact:trace_PSD_ineq}}\\
&= \gamma \tr(  \vec M^2  \vec \os_{P_w} \vec M^2\vec \Delta   ) \tag{cyclic property of trace}\\
&\lesssim  \gamma^2 \tr( \vec\os_{P_w}\vec M^2 \vec \os_{P_w} \vec M^2  ) \tag{using \Cref{cl:annoying-details} and \Cref{fact:trace_PSD_ineq}}\\
&\leq  \gamma^2 \|\vec\os_{P_w}\|_\op^2 \tr( \vec M^4 ) \tag{using  \Cref{fact:trace_PSD_ineq}}\\
&=  \gamma^2 \|\vec\os_{P_w}\|_\op^2 \|\vec M\|_4^4 \notag\\
&\leq  \gamma^2 \|\vec\os_{P_w}\|_\op^2 \|\vec M\|_2^4 \tag{using $\|\vec A \|_{q+1}\leq \|\vec A \|_q$}\\
&=  \gamma^2 \|\vec\os_{P_w}\|_\op^2 \|\vec M\|_\fr^4 \;, \label{eq:first-term}
\end{align}
  where  the first application of \Cref{fact:trace_PSD_ineq} above required that $\vec M^2 \vec \Delta \vec M^2 \succeq 0$ which is indeed
  satisfied because we have shown that $\vec \Delta \succeq 0$ in \Cref{cl:annoying-details} and thus \Cref{fact:PSDness} implies  $\vec M^2 \vec \Delta \vec M^2 \succeq 0$. A similar argument was used in the second application of \Cref{fact:trace_PSD_ineq}.%

  It remains to bound the second term of \Cref{eq:variance}.
  To this end, we have that
  \begin{align}
    \eps \left\langle \vec \os_{P_w}^B, \vec M^2  \right\rangle
    &= \left\langle \vec  \os_{P_w} - (1-\eps)\os_{P_w}^G , \vec M^2\right\rangle \notag \\
    &\leq  \left( (1+O(\gamma)) \left\langle \s,  \vec M^2 \right\rangle - (1-\eps)  \left\langle \os_{P_w}^G ,  \vec M^2 \right\rangle\right) \notag \\
    &=\left(   \left\langle \s-(1-\eps)\os_{P_w}^G,  \vec M^2 \right\rangle +  O(\gamma)  \left\langle \s,  \vec M^2 \right\rangle    \right) \notag  \\
    &\lesssim {\gamma}  \left\langle \s,  \vec M^2 \right\rangle \;, \label{eq:bad_points}
\end{align}
  where the first line uses the decomposition $\os_{P_w}=(1-\eps)\os_{P_w}^G +
    \eps \os_{P_w}^B$, the second line uses our assumption $\langle \os_{P_w}, \vec
    M^2 \rangle \leq \langle \s, \vec M^2 \rangle/(1-250\gamma) \leq
    (1+O(\gamma))\langle \s, \vec M^2 \rangle$, and the
  fourth line uses $\vec \s-(1-\eps)\os_{P_w}^G \preceq 1.2\gamma \s \preceq
    2\gamma \s$ by \Cref{cl:annoying-details}.
  Using \Cref{eq:bad_points}, we can finally upper bound the second term of
  \Cref{eq:variance}:
  \begin{align}
    \widetilde{\eps}^2\left\| \vec M \os_{P_w}^B \vec M \right\|_\fr^2  \tag{using \Cref{fact:fr-trace}}
    &\leq \tilde{\eps}^2 \tr\left(  \vec M \os_{P_w}^B \vec M \right)^2 \notag \\
    &= \widetilde{\eps}^2 \left\langle \vec \os_{P_w}^B, \vec M^2  \right\rangle^2 \notag \\
    &\leq \eps^2 \left\langle \vec \os_{P_w}^B, \vec M^2  \right\rangle ^2 \tag{$\tilde{\eps} \leq \eps$}\\
    &\lesssim \gamma^2 \left\langle \s,  \vec M^2 \right\rangle^2 \tag{by \Cref{eq:bad_points}}\\
    &\leq \frac{\gamma^2}{1-1.2\gamma}   \cdot \left\langle \os_{P_w},  \vec M^2 \right\rangle^2 \tag{\Cref{cl:annoying-details} and \Cref{fact:trace_PSD_ineq}}  \\
    &\lesssim  \gamma^2  \| \os_{P_w} \|_{\op}^{2}  \tr(\vec M^2)^2 \\  
    &=  \gamma^2  \| \os_{P_w} \|_{\op}^{2}  \|\vec M\|_\fr^4 \;,  
\end{align}
where the first step uses that $\vec \os_{P_w}^B $ is a PSD matrix, and thus $ \vec M \os_{P_w}^B \vec M$ is also PSD.
  Putting everything together, we have shown that $\Var[Y] \lesssim \gamma^2  \| \os_{P_w} \|_{\op}^{2}  \|\vec M\|_\fr^4$.
\end{proof}

\subsection{Reducing the Potential Function Multiplicatively}
\label{sec:pot-reduction}
We now describe \Cref{alg:few_iter-basic}, the main component in achieving
\Cref{thm:near-linear-main}.
We follow the notation defined in \Cref{alg:few_iter-basic}.
Recall that the distribution $P$ given as input is just the
uniform distribution over the (corrupted) data set and is assumed to be
of the form $P=(1-\eps)G+\eps B$ where $G$ is ($C\eps$,$\gamma$)-stable (see
the comment below \Cref{def:stability}).

In this section, we will quantify the progress of our algorithm by keeping
track of the potential function $\phi_t:=\tr(\vec B_t^{2p+1})$ in each round
$t$ (where $\vec B_t$ is scaled version of $\s_t$, defined in Line \ref{line:definitions-main}).
\footnote{In previous sections, we used $\tr(\s_t^{2p+1})$ for simplicity.
	Since our formal proof will need $\phi_t$ to be naturally decreasing with $t$,
	we use the un-normalized second moment $\vec B_t$ for which $\vec B_{t+1}
		\preceq \vec B_t$.
	}
As mentioned earlier, the correctness of \Cref{alg:few_iter-basic} is
summarized by the following arguments: (i) whenever the condition $\langle \s,
	\vec M_t^2 \rangle \geq (1-C\gamma)\langle \os_{t}, \vec
	M_t^2 \rangle$ is true we can output a good solution with \solGenerator, (ii)
if the condition is false, then $\phi_{t+1}$ decreases multiplicatively
$\phi_{t+1} \leq (1-\gamma) \phi_t$ on average, (iii) if $\phi_t \leq
	\|\s\|_\op^{2p+1}/\poly(d)$ then the condition from (i) is necessarily true.

We have already shown (i) in the previous subsection, and we will state and prove (iii) in the \Cref{sec:small_potential} (\Cref{lem:RHSbound}).
In this section, we focus on establishing (ii) in the form of \Cref{cl:pot-decrease} below.

Before proving \Cref{cl:pot-decrease}, we outline how these claims imply a
version of \Cref{thm:near-linear-main} with runtime
$\widetilde{O}(nd/\gamma^3)$:
By the pruning\footnote{\label{ft:pruning}Na\"ive pruning removes only a few inliers; see
		\Cref{lem:stability-implications}.\ref{it:initialization}} of Line \ref{line:naive-prune-basic},
$\phi_0 \leq (d/\eps)^{O(p)} \|\s\|_\op^{2p+1}$.
Thus, there exists
$t_\mathrm{end}=O( \log^2(d/\eps)/\gamma^2)$ such that after
$t_\mathrm{end}$ rounds, $\phi_{t_\mathrm{end}}
	<(1-\gamma)^{t_\mathrm{end}}\phi_0 \leq \|\s\|_\op^{2p+1}/\poly(d)$, which
would cause \solGenerator to output a good solution.
Now, each iteration of the loop can be implemented in $\widetilde{O}(ndp)$
time: The calculation of $\vec M_t z$ involves $p$ multiplications of $z$ with
the empirical second moment matrix, each of which can be done in $O(nd)$ time
as $\s_t z= \sum_x x(x^\top
	z)$.
Thus, the total runtime is $\widetilde{O}(t_{\mathrm{end}} \cdot ndp) =
	\widetilde{O}(nd/\gamma^3)$. A detailed proof can be found in \Cref{sec:formal_proof}.
\begin{algorithm}[h]
	\caption{Robust PCA in Small Number of Iterations}
	\label{alg:few_iter-basic}
	\begin{algorithmic}[1]
		\State \textbf{Input}: $P,\eps,\gamma$.
		\State Let  $p {=} \frac{C\log(d/\gamma)}{\gamma},t_{\mathrm{end}}{=}\frac{C\log^2(d/\eps)}{\gamma^2}$ for large enough $C$.
		\State Find estimator $\widehat{\sigma}_{\op} \in (0.8\|\s\|_{\op} , 2 d\|\s\|_{\op})$.
		\label{line:naive-est-basic}
		\hfill \Comment{{c.f. \Cref{lem:stability-implications}.\ref{it:proj-var} with $\bU = \bI$}}
		\State Initialize  $w_{1,1}(x)= \Ind\left(\|x\|_2^2 \leq 10\widehat{\sigma}_{\op}(d/\eps)) \right)$.
		\label{line:naive-prune-basic}
		\For{ $t = 1,\ldots, t_{\mathrm{end}}$}
		\State Call  \solGenerator$(P,w_t,\eps,\gamma,\frac{1}{t_\mathrm{end}})$.
		\State Let $P_t$ be the distribution of $P$ weighted by $w_t$: $P(x)w_t(x)/ \E_{X \sim P}[w_t(X)] $.
		\State{ Let $\vec B_t := \E_{X\sim P} [w_t(X)XX^\top]$ and $\vec M_t := \vec B_t^{p}$.
		\label{line:definitions-main}}  \hfill \Comment{$\vec M_t$ does not need to be explicitly computed.}
		\State Let ${g}_t(x) :=\|\vec M_t x\|^2_2$.
		\State $v_t \gets \vec M_t z_t$, where $z_t \sim \cN(0,\bI)$. \label{line:z_t}
		\label{line:choosev-basic}
		\State Let $f_t(x) = (v_t^\top x)^2$. \label{line:f_t}
		\State Let $L_t$ be the $3\epsilon$-quantile of $f_t(\cdot)$ under $P_t$.
		\label{line:quantile-basic}
		\State $L_t \gets \max\{L_t, (0.1/d) \widehat{\sigma}_\op\|v_t\|_2^2 \}$ \footnotemark
		\label{line:max-basic}
		\State Let $\tau_t(x) = f_t(x) \Ind ( f_t(x) > L_t )$

		\State Find $\widehat{\sigma}_t$ such that $|\widehat{\sigma}_t - v_t^\top \s v_t| \leq 4\gamma v_t^\top \s v_t$.
		(e.g.,  $\widehat{\sigma}_t:=\E_{X \sim P}[w_t(X)f_t(X)\Ind ( f_t(X) \leq L_t )]$).
		\label{line:trimmed-mean-basic}

		\hfill \Comment{{c.f. \Cref{lem:stability-implications}.\ref{it:proj-var} with $\bU =  v_t^\top$}} 

		\State $\widehat{T}_t\gets 2.35 \gamma\widehat{\sigma}_t$.
		\State $w_{k,t+1} \gets \mathrm{HardThresholdingFilter}(P,w_t,\tau_t,\widehat{T}_t,R)$.
		\EndFor
		\State \textbf{return}
		FAIL.
		\label{line:filter_end-basic}

	\end{algorithmic}
\end{algorithm}
\footnotetext{This ensures a lower bound for $\tau(x)$ whenever it is non-zero.
	See second bullet in \Cref{sec:filtering-main-body} for why this is needed.
}

\begin{lemma}[Informal] In \Cref{setting:main}, if $\langle \s, \vec M_t^2 \rangle \geq (1-C\gamma)\langle \os_{t} ,
		\vec M_t^2 \rangle$, then  $\E_t[\phi_{t+1}] \leq (1 -\gamma) \phi_t$ in \Cref{alg:few_iter-basic}, where $\E_t$ is the conditional expectation up to $t$-th iteration.
	\label{cl:pot-decrease}
\end{lemma}
\textit{Proof Sketch of \Cref{cl:pot-decrease}.}
For simplicity, we assume that scores $f_t(x)$ used by the algorithm (Line
\ref{line:f_t}) are the same as the scores $g_t(x)$.
Observe that computing $g_t(x) = \| \vec M_t x \|_2^2$ for all $n$ points would
be computationally costly.
This is why we use the one-dimensional random projections $f_t(x)$, which are
unbiased estimates of $g_t(x)$.
A complete proof that  uses the estimates $f_t$ is significantly more involved and can be found in \Cref{sec:formal-pot-reduction}.

Using our definitions,
we calculate the decrease in potential at the ($t{+}1$)-th step:
\begin{align*}
    &\phi_{t+1} = \tr\left(\vec B_{t+1}^{2p+1}\right) \leq \tr\left(\vec B_t^p \vec B_{t+1} \vec B_t^p\right) = \tr(\vec M_t \vec B_{t+1} \vec M_t ) \\
    &= \E_{X \sim P}[w_{t+1}(X)g_t(X)] \\
    &=(1{-}\eps) \E_{X \sim G}[w_{t+1}(X)g_t(X)] {+} \eps \E_{X \sim B}[w_{t+1}(X)g_t(X)] \\
    &\leq\E_{X \sim G}[w_{t}(X)g_t(X)] + \eps \E_{X \sim B}[w_{t+1}(X)g_t(X)] \numberthis \label{eq:pot-decomposition}\;,
\end{align*}
where the first inequality uses that $\vec B_{t+1} \preceq \vec B_t$ is
decreasing in PSD order (along with \Cref{fact:trace_PSD_ineq2}) and the last
inequality uses $w_{t+1} \leq w_t$.  
We argue that the RHS above is $(1+O(\gamma))\langle \s, \vec M_t^2 \rangle$:
The first term in \Cref{eq:pot-decomposition}, which corresponds to inliers,
can be upper-bounded by $(1+2\gamma) \langle \s, \vec M_t^2 \rangle$ using
stability
(\Cref{lem:stability-implications}.\ref{it:stability_multidim})
For the second term in \Cref{eq:pot-decomposition},
we use $\tau_t(x) \leq g_t(x) + L_t \leq g_t(x) + 1.65(\gamma/\eps) \langle \s,
	\vec M_t^2 \rangle$, where the last inequality uses
\Cref{lem:stability-implications}.\ref{it:bound-quantile}.
Moreover, filtering ensures that $\eps \E_{X \sim B}[w_{t+1}(X)\tau_t(X)] \leq
	2.35\gamma \langle \s, \vec M_t^2 \rangle$, where we use
\Cref{lem:filter_guar-main-body} with $T = 2.35\gamma \langle \s, \vec M_t^2
	\rangle$,
with this choice of $T$ justified by
\Cref{lem:stability-implications}.\ref{it:goodscores}.
Combining these bounds, we obtain
\begin{equation}
	\label{eq:main-ineq}
	\phi_{t+1} \leq (1+O(\gamma)) \langle \s, \vec M_t^2 \rangle \;.
\end{equation}
Now, using our assumption $\langle \s, \vec M_t^2 \rangle <
	(1-C\gamma)\langle \os_{t}, \vec
	M_t^2 \rangle$, we complete the proof as follows:
\begin{align*}
    \phi_{t+1} &< (1{+}O(\gamma))  \langle \s, \vec M_t^2 \rangle
    < (1{+}O(\gamma))(1{-}C\gamma)  \langle \s_t, \vec M_t^2 \rangle\\
    &\leq \frac{(1+O(\gamma))(1-C\gamma)}{\E_{X \sim P}[w_{t}(X)]} \langle \vec B_{t}, \vec M_{t}^2 \rangle \leq (1-\gamma)\phi_t \;,
\end{align*}
using $C\gg 1$, $ \langle \vec B_{t}, \vec M_{t}^2 \rangle =\phi_t$, and
$\E_{X \sim P}[w_{t}(X)]\geq 1{-}O(\eps)$, as filtering does not delete more
than $O(\eps)$-mass from the inliers (see discussion below
\Cref{lem:filter_guar-main-body}).
\hfill\qedsymbol

The above sketch, while demonstrating the main idea of the argument, is rather simplistic, and the formal proof is given in \Cref{sec:formal-pot-reduction}.
We note that \Cref{cl:pot-decrease} holds for all $p \geq \log d$; however,
\Cref{lem:certificate1-main-body} requires large $p$.

\subsection{Improving the Dependence on $\epsilon$ in Runtime}
\label{sec:improvement}
We first describe the improvement in high-level and then move to its formal proof in \Cref{sec:formal_proof}.
In the simpler version of the algorithm, the main reason for a large runtime
was that (i) the initial potential was $\poly((d/\eps)^p)\|\s\|_{\op}^{2p+1}$,
(ii) the value of $p$ was $\frac{C}{\gamma}\log(\frac{d}{\gamma})$,
(iii) the potential
decreases by only $(1{-}\gamma)$ factor, and (iv) the algorithm terminates when
$\phi_t\leq \|\s\|_\op^{2p+1}/\poly(d) $.
Since (ii)-(iv) are most likely needed, we modify the algorithm to ensure that
the initial potential is much smaller:
$\poly((d/\eps)^{\log(d)})\|\s\|_{\op}^{2p+1}$.
If we are able to ensure this cheaply, the algorithm would need only
$\polylog(d/\eps)/\gamma$ rounds, saving one factor of $1/\gamma$ from the
runtime.

The idea is that while we do need (ii), it is only necessary when the algorithm
is about to produce a final solution.\footnote{\Cref{lem:certificate1-main-body} (Certificate lemma) requires $p>C\log(d/\gamma)/\gamma$.}
Suppose we run \Cref{alg:few_iter-basic} with $p=\log d$, where the initial potential is indeed $\poly((d/\eps)^{\log(d)})\|\s\|_{\op}^{2p+1}$.
In \Cref{sec:pot-reduction}, we guaranteed a $(1-\gamma)$-reduction in
potential until the condition $\langle \s, \vec M_t^2 \rangle \geq
	(1-C\gamma)\langle \os_{t}, \vec
	M_t^2 \rangle$ gets activated.
However, even if this condition gets activated, we do not know if
\Cref{alg:sol-gen} will succeed as $p \ll \log(d/\gamma)/\gamma$ (c.f.\ \Cref{lem:certificate1-main-body}).
Nonetheless, when this condition gets activated,
we see that the final potential has become much smaller: $\phi_t = \langle \vec
	B_t , \vec M_t^2\rangle \leq \langle \s_t , \vec M_t^2\rangle \leq (1 + C
	\gamma) \langle \s , \vec M_t^2\rangle \leq (1 + C \gamma)\|\s\|_\op \|\vec B_t
	\|_{2p}^{2p}$, where the second step uses that the condition is violated.
Using \Cref{fact:normReln} to relate $\|\vec B_t \|_{2p}^{2p}$ and $\phi_t =
	\|\vec B_t \|_{2p+1}^{2p+1}$,
we obtain $\phi_t \leq \poly(d/\eps)\|\s\|_\op^{2p+1}$.

So far, we have shown that starting with $\phi_0 {\leq}
	\poly((d/\eps)^{\log(d)})\|\s\|_\op^{2p+1}$, 
we can reduce $\phi_t$ to
$\phi_t \leq \poly(d)\|\s\|_\op^{2p+1}$ after $t \asymp \log^2(d/\eps)/\gamma$
rounds.
We can then double the value of $p$ to $p'=2p$ and restart the counter $t$.
The new potential $\phi'_0$ will be upper bounded using \Cref{fact:normReln} as follows:
\begin{equation*}
	\phi'_0 = \| \vec B_t \|_{2p'+1}^{2p'+1} \leq \| \vec B_t \|_{2p+1}^{2p'+1} =
	(\phi_t)^{\frac{2p'+1}{2p+1}} = (\phi_t)^{\frac{4p+1}{2p+1}},
\end{equation*}
i.e., it is still $\poly(d)\|\s\|_\op^{2p'+1}$.
We run the same procedure repetitively until $p$ reaches its final value of
$\Theta(\log^2(d/\gamma)/\gamma)$, where the analysis of the previous two sections
are applicable, returning a good vector.
This leads to the nested-loop algorithm outlined in \Cref{alg:few_iter-advanced}. This will be the algorithm realizing \Cref{thm:multi-pass}.

\begin{algorithm}[h]
	\caption{\textsc{RobustPCA} with improved runtime}
	\label{alg:few_iter-advanced}
	\begin{algorithmic}[1]
		\State \textbf{Input}: $P,\eps,\gamma$.
		\State Let $C,C'$ be sufficiently large absolute constants with their ratio $C/C'$ being sufficiently large. \label{line:constants}
		\State Let $k_\mathrm{end}:= \log((\log(d/\gamma)/\log(d))/\gamma)$, and $t_{\mathrm{end}}:=\frac{C  \log^2(d/\eps)}{\gamma}$.
		\State Find estimator $\widehat{\sigma}_{\op} \in (0.8\|\s\|_{\op} , 2 d\|\s\|_{\op})$.
		\label{line:naive-est}
		\hfill \Comment{{c.f. \Cref{it:proj-var} of \Cref{lem:stability-implications} with $\bU = \bI$}.}
		\State Initialize  $w_{1,1}(x)= \Ind\left(\|x\|_2^2 \leq 10\widehat{\sigma}_{\op}(d/\eps) \right)$.
		\label{line:naive-prune}
		\For {$k = 1,\ldots, k_{\mathrm{end}}$}
		\State  Let $p_k= 2^{k-1} p$, where $p = C'\log(d)$.\label{line:pk} \hfill \Comment{{$p_k$ ranges from $C'\log(d)$ to $C'\log(d/\gamma)/\gamma$}.}
		\For{ $t = 1,\ldots, t_{\mathrm{end}}$}
		\State Call  \solGenerator$(P,w_{k,t},\eps,\gamma,1/(k_\mathrm{end}\cdot t_\mathrm{end}))$. \hfill  \Comment{c.f.\ \Cref{alg:sol-gen}.}
		\State Let $P_{k,t}$ be the distribution of $P$ weighted by $w_{k,t}$: $P(x)w_{k,t}(x)/ \E_{X \sim P}[w_{k,t}(X)] $.
		\State{ Let $\vec B_{k,t} := \E_{X\sim P} [w_{k,t}(X)XX^\top]$ and $\vec M_{k,t} := \vec B_{k,t}^{p_k}$.
		\label{line:definitions}}  \hfill \Comment{$\vec M_{k,t}$ does not need to be explicitly computed.}
		\State Let ${g}_{k,t}(x) :=\|\vec M_{k,t} x\|^2_2$.
		\State $v_{k,t} \gets \vec M_{k,t} z_{k,t}$, where $z_{k,t} \sim \cN(0,\bI)$. \label{line:z_{k,t}}
		\label{line:choosev}
		\State Let $f_{k,t}(x) = (v_{k,t}^\top x)^2$. \label{line:f_{k,t}}
		\State Let $L_{k,t}$ be the $3\epsilon$-quantile of $f_{k,t}(\cdot)$ under $P_{k,t}$.
		\label{line:quantile}
		\State $L_{k,t} \gets \max\{L_{k,t}, (0.1/d) \widehat{\sigma}_\op\|v_{k,t}\|_2^2 \}$ 
		\label{line:max}
		\State Let $\tau_{k,t}(x) = f_{k,t}(x) \Ind ( f_{k,t}(x) > L_{k,t} )$

		\State Find $\widehat{\sigma}_{k,t}$ such that $|\widehat{\sigma}_{k,t} - v_{k,t}^\top \s v_{k,t}| \leq 4\gamma v_{k,t}^\top \s v_{k,t}$ 
		(e.g.,  $\widehat{\sigma}_{k,t}:=\E_{X \sim P}[w_{k,t}(X)f_{k,t}(X)\Ind ( f_{k,t}(X) \leq L_{k,t} )]$).
		\label{line:trimmed-mean}

		\hfill \Comment{c.f. \Cref{it:proj-var} of \Cref{lem:stability-implications} with $\vec U = v_{k,t}^\top$}

		\State $\widehat{T}_{k,t}\gets 2.35 \gamma\widehat{\sigma}_{k,t}$.
		\State $w_{k,t+1} \gets \mathrm{HardThresholdingFilter}(P,w_{k,t},\tau_{k,t},\widehat{T}_{k,t},R,0)$. \hfill \Comment{c.f.\ \Cref{alg:few_iter-advanced}.}

		\EndFor
		\State Set $w_{k+1,0} \gets w_{k,t+1}$.
		\EndFor
		\State \textbf{return}
		FAIL.
		\label{line:filter_end-basic}

	\end{algorithmic}
\end{algorithm}

\subsection{Proof of \Cref{thm:multi-pass}}\label{sec:formal_proof}
We now provide the formal proof of \Cref{thm:multi-pass} and provide the details omitted from the earlier sketches, such as the fact that the scores $f_t(x)$ computed by the algorithm do not coincide with $g_t(x)$. We will also incorporate the arguments of \Cref{sec:improvement} for the improved runtime.  

As discussed earlier, the proof consists of three main claims: (i) whenever the condition $\langle \s,
	\vec M_t^2 \rangle \geq (1-C\gamma)\langle \os_{t}, \vec
	M_t^2 \rangle$ is true (for a large enough value of $p$), we can output a good solution with \solGenerator, (ii)
if the condition is false, then $\phi_{t+1}$ decreases multiplicatively
$\phi_{t+1} \leq (1-\gamma) \phi_t$ on average, and (iii) if $\phi_t \leq
	\|\s\|_\op^{2p+1}/\poly(d)$, then the condition from (i) is necessarily true.
 The first part has been shown in \Cref{lem:certificate1-main-body}, the last part is a fairly straight-forward implication of stability, which is shown in \Cref{sec:small_potential} (\Cref{lem:RHSbound}), and the result that corresponds to (ii) is shown in \Cref{sec:formal-pot-reduction}. These pieces are then put together in \Cref{sec:combine} to complete the proof of the theorem. A detailed runtime analysis is provided in \Cref{sec:runtime}.

In terms of notation, we will follow the notations defined in \Cref{alg:few_iter-advanced}. In particular, many relevant quantities will have two subscripts such as $\s_{k,t}$: here $k$ refers to the outer loop and $\ell$ refers to the inner loop.

For the formal proof, we will use that the scores $f_{k,t}(x)$ are a constant factor approximation of the scores $g_{k,t}(x)$ with constant absolute probability (a detail that we omitted in the sketch of the previous subsection); This follows from \Cref{fact:quadratics-approximation}. 
Moreover, to make our analysis of this section more flexible and easier to adapt to the streaming setting of the next section, we will assume a weaker approximation on $f_t(\cdot)$, which includes also an additive error term (this additive error will account for approximation of $\vec M$ in the streaming section).
\begin{assumption}
  \label{assumption:full_points}
  The random selection of the vector $v_{k,t}$ in Line \ref{line:choosev} of \Cref{alg:few_iter-advanced} is such
  that for every point $x$ in our domain,
  \begin{align*}
    \pr_{v_{k,t}}\left[ f_{k,t}(x)>\frac{g_{k,t}(x)}{10} - 0.01\frac{\gamma}{\eps} \|\vec M_{k,t}\|_\fr^2 \|\s\|_\op \right] \geq 0.4 \;.
\end{align*}
\end{assumption}
We emphasize again that \Cref{fact:quadratics-approximation} implies that \Cref{assumption:full_points} holds for \Cref{alg:few_iter-advanced}.

\subsubsection{Small Potential Implies that Stopping Condition Holds}\label{sec:small_potential}

In this section, we formally show the result that whenever the potential function $\phi := \tr(\vec B^{2p+1} )$ is smaller than $\|\s\|_\op^{2p+1}/\poly(d/\eps)$ for a large enough value of $p$, then the condition of the previous lemma gets satisfied. Note that the factor $(1-2\gamma)^{2p}$ mentioned in the conclusion of the following lemma is indeed bigger than $1/\poly(d/\eps)$ since $p  = O (\log(d/\gamma)/\gamma)$.

\begin{restatable}{lemma}{CERTIF}
  \label{lem:RHSbound}
  Let $0<20\eps\leq \gamma < 1/20$, and a positive integer $p$.
  Let $P=(1-\eps)G+ \eps B$, where $G$ is a $(20\eps,\gamma)$-stable distribution
  with respect to $\s$, and let $w: \R^d \to [0,1]$ with $\E_{X \sim G}[w(X)]\geq 1-3\eps$. Define $\vec B:=\E_{X \sim P}[w(X) XX^\top]$, $\vec M := \vec B^{p}$, and $\phi := \tr(\vec B^{2p+1} )$.  If $\phi \leq \E_{X \sim
      P}[w(X)]\frac{(1-2\gamma)^{2p }}{1-250\gamma} \| \s \|_\op^{2p +1}$,
  then $\langle \s, \vec M^2 \rangle \geq (1-250\gamma)\langle \s_{P_w},
    \vec M^2 \rangle$.
\end{restatable}

\begin{proof}
We claim that for the statement to be true, it suffices to show that 
\begin{align} \label{eq:sufficestoshow}
   \langle
    \vec M^2, \s \rangle \geq (1-2\gamma)^{2p}\| \s \|_{\op}^{2p + 1} \;.
\end{align}
This is indeed sufficient, because if we had \Cref{eq:sufficestoshow} at hand, then whenever it holds $\phi \leq \E_{X \sim
      P}[w(X)]\frac{(1-2\gamma)^{2p }}{1-250\gamma} \| \s \|_\op^{2p +1}$, then we have that
    \begin{align*}
        \langle \vec M^2, \s_{P_w} \rangle = \frac{\phi}{\E_{X \sim
      P}[w(X)]} \leq \frac{(1-2\gamma)^{2p }}{1-250\gamma} \| \s \|_\op^{2p +1} \leq \frac{1}{1-250\gamma}\langle \vec M^2, \s \rangle \;,
    \end{align*}
    where the first step uses the definitions $\phi = \tr(\vec B^{2p+1}) = \langle\vec M^2, \vec B  \rangle = \E_{X \sim
      P}[w(X)]\langle\vec M^2, \s_{P_w}  \rangle$ , the second step uses $\phi \leq \E_{X \sim
      P}[w(X)]\frac{(1-2\gamma)^{2p }}{1-250\gamma} \| \s \|_\op^{2p +1}$, and the last step uses \Cref{eq:sufficestoshow}.
      
In the reminder, we establish \Cref{eq:sufficestoshow}.
  Consider the spectral decomposition $\s = \sum_{i=1}^d \lambda_i u_i u_i^\top$ with $\lambda_1\geq \lambda_2\geq \dots\geq 0$.
  We will lower bound $\langle \vec M^2, \s \rangle$ using $\langle \vec M^2, \lambda_1 u_1u_1^\top \rangle$ as follows:
  \begin{align}
    \langle \vec M^2, \s \rangle &= \left\langle \vec M^2,  \sum_{i=1}^d \lambda_i u_i u_i^\top  \right\rangle 
    = \sum_{i=1}^d \lambda_i  u_i^\top  \vec M^2 u_i
    \geq \lambda_1  u_1^\top \vec M^2 u_1 =  \|\s \|_\op  u_1^\top \vec M^2 u_1 \;.  \label{eq:last-step}
\end{align}
In the remainder of this proof, we will show that $u_1^\top \vec M^2 u_1  
    \geq (1-2\gamma)^{2p}\| \s \|_\op^{2p+1}$, which will complete the proof.
  We begin by lower bounding $\vec B$ by $\s$ in the PSD sense below:
\begin{align}
    \vec B &= \E_{X \sim P}[w(X)XX^\top] \succeq (1-\eps)\E_{X \sim G}[w(X)] \frac{ \E_{X \sim G}[w(X)XX^\top]}{\E_{X \sim G}[w(X)]} \notag\\
    &\succeq (1-\eps)(1-3\eps)(1-\gamma)\s \succeq (1-2\gamma) \s \;, \label{eq:helping-ineqq}
\end{align}
  where the first line uses that $P=(1-\eps)G + \eps B$, and the last line uses
  $\E_{X \sim G}[w(X)]\geq 1-3\eps$,  and stability
  condition for the second moment of the normalized distribution $G_w$, and the
  last step uses that $\eps<\gamma/20$.
  Since $\vec M = \vec B^p$ and $\vec B$ is at least $\vec \s$, we obtain the following lower bound on $u_1^\top \vec M^2 u_1$:
  \begin{align*}
   u_1^\top  \vec M^2 u_1  =  u_1^\top  \vec B^{2p} u_1 \geq (u_1^\top  \vec B u_1)^{2p}  \geq ( (1 - 2\gamma) u_1^\top  \s u_1)^{2p } = (1 -2 \gamma)^{2p }\| \s \|_\op^{2p } \;,
\end{align*}
  where the first inequality uses \Cref{fact:psd-power}, and the second
  inequality uses \Cref{eq:helping-ineqq}. 
\end{proof}

\subsubsection{Formal Proof of Multiplicative Decrease of the Potential}\label{sec:formal-pot-reduction}
In this section, we provide a formal proof that establishes the following: If the condition $\langle \s, \vec M_{k,t}^2 \rangle \geq (1-250\gamma)\langle \os_{k,t}, \vec M_{k,t}^2 \rangle$ is not satisfied, then, on average, the potential function decreases multiplicatively; Recall that $(k,t)$ denotes the double index of the iteration under analysis. First, we establish \Cref{lem:one-round}, which analyzes a single iteration of our algorithm's inner loop, conditioned on the previous ones. Then, 
  we state \Cref{cor:entire-inner-loop} to show that, when the entire loop finishes, the potential function is reduced by a factor of $(1-\gamma)^t$ on expectation.

\Cref{lem:one-round} below stipulates that if for each prior iteration $(k',t')$ we have $\langle \s, \vec M_{k',t'}^2 \rangle < (1-250\gamma)\langle \os_{k',t'}, \vec M_{k',t'}^2 \rangle$ (this corresponds to event $\cE_{k,t}^{(2)}$ of the lemma statement),\footnote{
Note that event $\cE_{k,t}^{(2)}$ 
implies that the algorithm has not ended yet (since the first time it happens that $\langle \s, \vec M_{k,t}^2 \rangle \geq (1-250\gamma)\langle \os_{k,t}, \vec M_{k,t}^2 \rangle$, we know by \Cref{lem:certificate1} that the algorithm will stop w.h.p.).}
and the fraction 
of outliers is still $O(\eps)$ (as given in event $\cE_{k,t}^{(1)}$),\footnote{We will show later on that this event indeed holds through the algorithm, as an invariant condition.}  then, in expectation we have $\phi_{k,t+1} \leq (1-\gamma) \phi_{k,t}$;  The expectation  is taken with respect to the randomness coming from the random directions $v_{k,t}$ (c.f.\ Line \ref{line:choosev} of \textsc{RobustPCA}) used for defining the scores $f_{k,t}$. The proof sketch of \Cref{sec:pot-reduction} made the simplification that $f_{k,t}$ is always equal to its expected value and hence there we showed that $\phi_{k,t+1} \leq (1-\gamma) \phi_{k,t}$ deterministically, but as we show in the formal version of that proof below, this happens in expectation with respect to the $f_{k,t}$'s randomness. The randomness from the \textsc{HardThresholdingFilter}, denoted by $\mathcal{F}_{k,t}$ in the lemma statement, does not play any role in this statement (it only affects the runtime of the algorithm, which is discussed in \Cref{sec:runtime}), and thus the expected potential decreases multiplicatively (\Cref{eq:pot_reduction}) under any conditioning of $\mathcal{F}_{k,t}$.

\begin{lemma}
  \label{lem:one-round}
  Consider the notation of \Cref{alg:few_iter-advanced}, where the input distribution is the
  mixture $P=(1-\eps)G + \eps B$, with $G$ being a $(20\eps,\gamma)$-stable
  distribution with respect to $\s$ for $0<20\eps\leq \gamma < \gamma_0$ for a sufficiently small positive constant $\gamma_0$.
  Make \Cref{assumption:full_points}.
  Let $\cE_{k,t} = \cE_{k,t}^{(1)} \cap \cE_{k,t}^{(2)}$ denote the
  intersection of the following two events:
    \begin{enumerate}[label = (\roman*)]
      \item $\cE_{k,t}^{(1)} $: $(1-\eps) \E_{X \sim G}[1-w_{k',t'}(X)] + \eps \E_{X
                \sim B}[w_{k',t'}(X)] \leq 3 \eps$, for every iteration
              $(k',t')$ prior to (and
            including) $(k,t)$.
      \item $\cE_{k,t}^{(2)} $:  $\langle \s, \vec M_{k',t'}^2  \rangle < (1-250\gamma)\langle \os_{k',t'}, \vec M_{k',t'}^2  \rangle$ for every iteration $(k',t')$ from $(k,1)$ up to (and including) $(k,t)$.
    \end{enumerate}
  
  Define the potential function $\phi_{k,t}:= \tr(\vec B_{k,t}^{2p_k + 1})$.
  Let $F_{k,t}$ be the randomness used by \textsc{HardThresholdingFilter}
  during the $(k,t)$-th loop of \textsc{RobustPCA} (i.e., $F_{k,t}$ is the
  collection of the random variables $r_1,r_2,\ldots$ used by the filter), and let 
  $v_{k,t}$ be the random vectors used in Line \ref{line:choosev} of
  \textsc{RobustPCA} (\Cref{alg:few_iter-advanced}).
  Also, define $\cF_{k,t} = \{F_{k',t'} : k'\leq k-1, t' \leq t_{\mathrm{end}}\}
    \cup \{F_{k,t'} : t' \leq t \}$, i.e., the entire history of filters up to (and
  including) the iteration $(k,t)$.
  Define $\cV_{k,t}$ similarly for $v_{k,t}$'s.
  
  Then, the following is true for every conditioning on $\cF_{k,t}, \cV_{k,t-1}$:
  \begin{align}
     \E_{v_{k,t}}  \left[\phi_{k,t+1} \Ind(\cE_{k,t}) \mid \cF_{k,t}, \cV_{k,t-1}   \right] \leq     (1-\gamma)\phi_{k,t} \Ind(\cE_{k,t-1}) \;. \label{eq:pot_reduction}
\end{align}
  
\end{lemma}
\begin{proof}
  We analyze the $(k,t)$-th iteration of \textsc{RobustPCA}, that is, the outer
  loop with index $k$ and the inner loop with index $t$.
  We condition on everything that has happened previously, and we are only interested
  in analyzing what happens in the current round of the inner loop with respect
  to randomness coming from $v_{k,t}$.
  In particular, the
  expectation will be conditional on the past history $\cF_{k,t}, \cV_{k,t-1}$ as
  in \Cref{eq:pot_reduction} but we omit explicitly writing them for notational convenience.

  In the case that $\Ind(\cE_{k,t}) =0$, the conclusion of our lemma holds trivially. Thus, in the reminder of the proof, we analyze the case $ \Ind(\cE_{k,t})=1$.
  
  We will call a point $x$ \emph{full} if  $f_{t,k}(x) > g_{k,t}(x)/10 -
    0.01(\gamma/\eps)\|\vec M_{k,t}\|_\fr^2\|\s\|_\op$, otherwise we will call it
  \emph{empty}.
  Note that $x$ being full implies that
  \begin{align} 
    \tau_{k,t}(x) &\geq f_{k,t}(x) - L_{k,t} 
    \geq \frac{g_{k,t}(x)}{10} - L_{k,t} -  0.01\frac{\gamma}{\eps}\|\vec M_{k,t}\|_\fr^2\|\s\|_\op \notag \\
    &\geq \frac{g_{k,t}(x)}{10} - 1.65\frac{\gamma}{\eps} \| v_{k,t} \|_2^2 \| \s \|_\op - 0.01\frac{\gamma}{\eps}\|\vec M_{k,t}\|_\fr^2\|\s\|_\op\;, \label{eq:fullpoint}
\end{align}
  where the first inequality uses the definition of $\tau_{k,t}$, and the last inequality uses \Cref{it:bound-quantile} of \Cref{lem:stability-implications}.

  We now use 
  \Cref{lem:filter_guar} to analyze the effect of filtering via \textsc{HardThresholdingFilter} (\Cref{alg:filter-appendix}): 
  We apply that lemma with $T= 2.35\gamma v_{k,t}^\top\s v_{k,t}$,
  $\widehat{T}=2.35 \gamma \widehat{\sigma}_{k,t}$, and $\delta=0$; we will now justify the choice of these parameters.
  The first requirement of \Cref{lem:filter_guar} is that $(1-\eps)\E_{X \sim
      G}\left[w_{k,t}(x)\tau_{k,t}(x)\right] < T$, which is satisfied by
  \Cref{it:goodscores} of \Cref{lem:stability-implications} (specifically, the special case that uses $\vec U =
  v_{k,t}^\top$).
  For its second requirement of $\widehat{T}>T/5$, we have that $|\widehat{T}-T|=2.35
    \gamma |\widehat{\sigma}_{k,t}-v_{k,t}^\top \s v_{k,t}| \leq 2.35 \gamma \cdot
    4 \gamma\cdot v_{k,t}^\top \s v_{k,t} \leq (1/5)\cdot 2.35 \gamma
    v_{k,t}^\top\s v_{k,t} = T/5$, where the second step uses \Cref{it:proj-var} of \Cref{lem:stability-implications}
  (with $\vec U = v_{k,t}^\top$) and the last inequality uses that $\gamma<1/20$.
  Thus, the lemma is applicable, and it yields that
  \begin{align} \label{eq:after_filter}
   \E_{X \sim P}[w_{k,t+1}(X) \tau_{k,t}(X)] \leq 3T  \leq 7.1\gamma \| v_{k,t} \|_2^2 \| \s \|_\op \;,
\end{align}
  with probability one.
  Therefore, for the bad points that are full, we get that their updated weights after the filter in the current iteration satisfy the following:
  \begin{align}
    \eps\E_{X \sim B}[w_{k,t+1}(X) g_{k,t}(X) \Ind(\text{$X$ full} )] &\leq 10\eps\E_{X \sim B}[w_{k,t+1}(X) \tau_{k,t}(X)] + 16.5 \gamma \| v_{k,t} \|_2^2 \| \s \|_\op  + 0.1\gamma \left\| \vec M_{k,t} \right\|_\fr^2 \notag \\ 
    &< 88 \gamma \| v_{k,t} \|_2^2 \| \s \|_\op + 0.1\gamma \left\| \vec M_{k,t} \right\|_\fr^2 \| \s \|_\op \;. \label{eq:badfull}
\end{align}
  where the first step uses \Cref{eq:fullpoint} and the second uses
  \Cref{eq:after_filter}.

  By \Cref{assumption:full_points}, every point is full with probability $0.4$.
  We will now take expectation with respect to $v_{k,t}$ (we again remind the reader that the
  expectation is conditioned on the past history $\cF_{k,t}, \cV_{k,t-1}$ as
  in \Cref{eq:pot_reduction} but we omit explicitly writing them for notational convenience)).
  \begin{align}
    &\E_{v_{k,t}} \left[ \eps \E_{X \sim B}[w_{k,t+1}(X) g_{k,t}(X) ] \right] \notag \\
    &\quad=  \E_{v_{k,t}} \left[ \eps \E_{X \sim B}[w_{k,t+1}(X) g_{k,t}(X) \Ind(\text{$X$ full}) ]\right]   + \E_{v_{k,t}}\left[  \eps \E_{X \sim B}[ w_{k,t+1}(X) g_{k,t}(X) \Ind(\text{$X$ empty}) ] \right] \notag \\
    &\quad\leq \E_{v_{k,t}} \left[\eps \E_{X \sim B}[w_{k,t+1}(X) g_{k,t}(X) \Ind(\text{$X$ full}) ] \right]  + \eps \E_{X \sim B}\left[  w_{k,t}(X) g_{k,t}(X) \E_{v_{k,t}} [\Ind(\text{$X$ empty})] \right]  \notag \\
    &\quad\leq \E_{v_{k,t}} \left[ \left( 88 \gamma \| v_{k,t} \|_2^2 \| \s \|_\op + 0.1\gamma \left\| \vec M_{k,t} \right\|_\fr^2 \| \s \|_\op\right) \right]  + \eps \E_{X \sim B}\left[  w_{k,t}(X) g_{k,t}(X) (0.6) \right]  \notag \\
    &\quad=  88.1 \gamma \|\vec M_{k,t}\|_\fr^2 \| \s \|_\op + 0.6 \eps \E_{X \sim B}[  w_{k,t}(X) g_{k,t}(X) ] \;, \label{eq:badpoints}
\end{align}
  where the second step uses $w_{k,t+1} \leq  w_{k,t}$ for the  second term, the third step uses \Cref{eq:badfull} and the fact that the probability of a point being empty is at most $0.6$,and the last step uses that $\E_{v_{k,t}}\left[\|v_{k,t}\|_2^2 \right] = \E_{z_{k,t} \sim \cN(0,\vec I)}[z_{k,t}^\top \vec M_{k,t}^2 z_{k,t} ] = \tr(\vec M_{k,t}^2)=\|\vec M_{k,t}\|_\fr^2$.
  
  We now upper bound the expectation of the potential function using following series of inequalities, which are explained below:\footnote{Recall that we are in the setting when $ \Ind(\cE_{k,t-1})$=1 , and thus we omit writing the indicator inside the expectations  explicitly.}
  \begin{align}
   \E_{v_{k,t}} \left[\tr\left( \vec  B_{k,t+1}^{2p_k+1} \right) \right] 
   &\leq \E_{v_{k,t}} \left[ \tr\left(\vec B_{k,t}^{p_k}\vec B_{k,t+1} \vec B_{k,t}^{p_k}   \right) \right] \label{eq:s1} \\
   &= \E_{v_{k,t}} \left[ \tr\left(\vec M_{k,t} \vec B_{k,t+1} \vec M_{k,t}    \right) \right] \notag \\ 
   &=   \E_{v_{k,t}} \left[   \E_{X \sim P}[w_{k,t+1}(X) g_{k,t}(X)]   \right] \label{eq:s3}\\
   &= \E_{v_{k,t}} \left[(1-\eps)  \E_{X \sim G}[w_{k,t+1}(X) g_{k,t}(X)]  + \eps \E_{X \sim B}[w_{k,t+1}(X)g_{k,t}(X)]   \right] \label{eq:s4}\\
   &\leq (1-\eps)  \E_{X \sim G}[w_{k,t}(X) g_{k,t}(X)] + \E_{v_{k,t}} \left[\eps \E_{X \sim B}[w_{k,t+1}(X) g_{k,t}(X)] \right] \label{eq:s5} \\
   &\leq  \left(1+  2\gamma  \right) \langle \s, \vec M_{k,t}^2 \rangle +89  \gamma \| \vec M_{k,t} \|_\fr^2 \| \s \|_\op  + 0.6 \eps  \E_{X \sim B}[  w_{k,t}(X) g_{k,t}(X) ] \label{eq:s6} \\
  &= \left(1+  2\gamma  \right) \langle \s, \vec M_{k,t}^2 \rangle + 89  \gamma \| \vec M_{k,t} \|_\fr^2 \| \s \|_\op  
   \notag\\
   &\quad\quad + 0.6\left(  \E_{X \sim P}[  w_{k,t}(X) g_{k,t}(X) ]  -  (1-\eps) \E_{X \sim G}[  w_{k,t}(X) g_{k,t}(X) ] \right) \label{eq:s8} \\
   &= \left(1+ 2 \gamma  \right) \langle \s, \vec M_{k,t}^2 \rangle + 89  \gamma \| \vec M_{k,t} \|_\fr^2 \| \s \|_\op  \notag\\
   &\quad\quad + 0.6\left( \phi_{k,t} -  (1-\eps) \E_{X \sim G}[  w_{k,t}(X) g_{k,t}(X) ] \right) \label{eq:s9} \\
   &\leq \left(1+  2\gamma  \right) \langle \s, \vec M_{k,t}^2 \rangle +89  \gamma \| \vec M_{k,t} \|_\fr^2 \| \s \|_\op  + 0.6 \phi_{k,t} - 0.6\left(1- 3 \gamma \right)\langle \s, \vec M_{k,t}^2 \rangle  \label{eq:s10} \\
   &= 0.4\langle \s, \vec M_{k,t}^2 \rangle + 3.8\gamma\langle \s, \vec M_{k,t}^2 \rangle  +89  \gamma \| \vec M_{k,t} \|_\fr^2 \| \s \|_\op  + 0.6 \phi_{k,t}  \notag \\
  &\leq 0.4\langle \s, \vec M_{k,t}^2 \rangle  + 93  \gamma \| \vec M_{k,t} \|_\fr^2 \| \s \|_\op  + 0.6 \phi_{k,t} \;,
  \label{eq:potential_one_round}
\end{align}
where
  \Cref{eq:s1} uses that $\vec B_{k,t+1} \preceq \vec B_{k,t}$ along with \Cref{fact:trace_PSD_ineq2} and the cyclic property of trace operator,  \Cref{eq:s3} uses the definition $\vec B_{k,t} = \E_{X\sim P} [w_{k,t}(X)] \os_{k,t} = \E_{X\sim P} [w_{k,t}(X)XX^\top] $, \Cref{eq:s4} uses that $P=(1-\eps)G+\eps B$, \Cref{eq:s5} uses $w_{k,t+1}(x)\leq w_{k,t}(x)$ for the first term, \Cref{eq:s6} uses $1-\eps\leq 1$ and \Cref{it:stability_multidim} of \Cref{lem:stability-implications} for the first term and \Cref{eq:badpoints} for the second term, \Cref{eq:s8} uses that $P=(1-\eps)G+\eps B$, \Cref{eq:s9} uses
   the definition of $\vec B_{k,t}$ to get $ \E_{X \sim P}  [w_{k,t}(X) g_{k,t}(X) ] =   \E_{X \sim P}[  w_{k,t}(X) \tr(\vec M_{k,t}^2 XX^\top) ]=  \tr(\vec M_{k,t}^2 \E_{X \sim P}[w_{k,t}(X)XX^\top]) = \tr(\vec B_{k,t}^{2p_k+1})=\phi_{k,t}$.
  We now explain the last three steps. Staring with \Cref{eq:s10}, recall that we have conditioned on the event  that
  $(1-\eps) \E_{X \sim G}[1-w_{k,t}(X)] + \eps \E_{X \sim B}[w_{k,t}(X)] \leq
    2.9\eps$. Thus, we use \Cref{it:stability_multidim} of \Cref{lem:stability-implications} to get
  $
    (1-\eps) \E_{X \sim G}[w_{k,t}(X)g_{k,t}(X)] \geq (1-\eps)(1-2\gamma)
    \langle \s, \vec M^2_{k,t} \rangle \geq (1-3\gamma)\langle \s, \vec M^2_{k,t}
    \rangle
  $ (where the last inequality uses $\eps<\gamma/20$).
Finally,  \Cref{eq:potential_one_round} upper bounds the terms $\gamma \langle \s, \vec M_{k,t}^2\rangle $ by $\gamma \|\s\|_\op \|\vec M\|_\fr^2 $.
We will now relate both $\langle \s, \vec M_{k,t}^2 \rangle$ and  $\|\s\|_\op \|\vec M\|_\fr^2$ with $\phi_{k,t}$.

We begin with the term $\langle \s, \vec M_{k,t}^2 \rangle$: Since the event $\cE_{k,t}^{(2)}$ holds,
  we have that
  \begin{align}
  \label{eq:the-first-term}
    \langle \s, \vec M_{k,t}^2 \rangle &< (1-250\gamma)\langle \os_{k,t}, \vec M_{k,t}^2 \rangle = \frac{1-250\gamma}{\E_{X \sim P}[w_{k,t}(X)]}  \langle \vec B_{k,t}, \vec M_{k,t}^2 \rangle \\
    &\leq \frac{1-250\gamma}{(1-\eps)(1-3\eps)}\phi_{k,t}\leq (1-240\gamma) \phi_{k,t} \;,
\end{align}
  where the first inequality uses the definition of the event $\cE_{k,t}^{(2)}$,
  the next steps use that $\E_{X \sim P}[w_{k,t}(X)]\geq (1-\eps)\E_{X \sim
      G}[w_{k,t}(X)]\geq (1-\eps)(1-3\eps)$, where the last inequality here used the
  definition of the event $\cE_{k,t}^{(1)}$.

    We now state the following result relating $\phi_{k,t}$ with $ \|\s\|_\op\|\vec M_{k,t}\|_\fr^2$  for $p_k$ sufficiently large:
  \begin{claim}
    \label{cl:final_iter}
    If $p_k \geq C\log(d)$ for a sufficiently large constant $C$ and the event $\cE_{k,t}$ holds, then
    \begin{align*}
    \tr\left(\vec B_{k,t}^{2p_{k}}   \right) \|\s \|_\op \leq 1.01\cdot  \tr\left(\vec B_{k,t}^{2p_k+1} \right)\;.
\end{align*}
  \end{claim}
  \begin{proof}
    We will prove the result using the relations between the Schatten $p$-norms and relating $\s$ with $\s_t$ using stability  as follows:
    \begin{align}
    \tr\left( \vec B_{k,t}^{2p_k} \right)\|\s\|_\op  \tag{using  stability as in \Cref{eq:helping-ineq}}
    &< (1+3\gamma) \tr\left( \vec B_{k,t}^{2p_k} \right)\|\vec B_{k,t}\|_\op \notag \\
    &= (1+3\gamma) \| \vec B_{k,t} \|_{2p_k}^{2p_k} \| \vec B_{k,t} \|_{\infty} \notag \\
    &\leq (1+3\gamma) \| \vec B_{k,t} \|_{2p_k}^{2p_k+1} \tag{$\|\vec A\|_{\infty} \leq \|\vec A\|_q$ $\forall q$}\\
    &\leq  (1+3\gamma) \left(  d^{\frac{1}{2p_k(2p_k+1)}} \|\vec B_{k,t} \|_{2p_k+1} \right)^{2p_k+1} \tag{by \Cref{fact:normReln}}\\
    &= (1+3\gamma) d^{1/{2p_k}} \|\vec B_{k,t} \|_{2p_k+1}^{2p_k+1} \notag \\
    &\leq (1+3\gamma)\cdot 1.001\cdot  \tr\left(\vec B_{k,t}^{2p_k+1} \right)\tag{$p_k \geq  C \log(d)$ for $C$ large enough } \\
    &\leq 1.01 \cdot  \tr\left(\vec B_{k,t}^{2p_k+1} \right) \;.\tag{$\gamma<\gamma_0$ for a sufficiently small $\gamma_0$}
\end{align}
    This completes the proof of \Cref{cl:final_iter}.
  \end{proof}
  We can now upper bound the RHS of \Cref{eq:potential_one_round} using \Cref{eq:the-first-term} and \Cref{cl:final_iter}:
  \begin{align*}
    0.4\langle \s, \vec M_{k,t}^2 \rangle  + 93  \gamma \| \vec M_{k,t} \|_\fr^2 \| \s \|_\op  + 0.6 \phi_{k,t}
    &\leq 0.4(1-240\gamma) \phi_{k,t} + 93\cdot 1.01\cdot \gamma \phi_{k,t} + 0.6 \phi_{k,t} \\
    &\leq (1-\gamma)\phi_{k,t} \;.
\end{align*}
    This completes the proof.

\end{proof}

\begin{corollary}
  \label{cor:entire-inner-loop}
  In the context of \Cref{lem:one-round}, assume that
  $\E_{\cV_{k-1,t_\mathrm{end}}} [\phi_{k,1} \mid \cF_{k-1,t_\mathrm{end}} ] \leq
    R$.
  Then, under every conditioning $\cF_{k,t-1}$ of the form $\cF_{k-1,t_\mathrm{end}} \cup \{
    F_{k,t'} : t' \leq t-1\}$ (i.e., every conditioning for the filter up to the $(k,t-1)$-th iteration that agrees with $\cF_{k-1,t_\mathrm{end}}$ on the first iterations up to the $(k-1,t_\mathrm{end})$-th), we have that 
  \begin{align*}
    \E_{\cV_{k,t-1} } [\phi_{k,t}\Ind(\cE_{k,t-1}) \mid \cF_{k,t-1}] \leq (1-\gamma)^{t-1} R  \;.
\end{align*}
\end{corollary}
\begin{proof}
  We prove this by induction on $t$.
  The base case $t=1$ holds trivially by assumption.
  Now, we assume that the claim holds for the index $(k,t)$ and we will show that
  it will continue to hold for the index $(k,t+1)$.
  That is, suppose that $\E_{ \cV_{k,t-1} } [\phi_{k,t} \Ind(\cE_{k,t-1}) \mid \cF_{k,t-1}] \leq
    (1-\gamma)^{t-1} R$.
  Then, we know by \Cref{lem:one-round} that
  \begin{align*}
    \E_{v_{k,t}}  \left[\phi_{k,t+1} \Ind(\cE_{k,t}) \mid \cF_{k,t}, \cV_{k,t-1}   \right] &\leq   (1-\gamma)\phi_{k,t} \Ind(\cE_{k,t-1})  \;.
\end{align*}
  Taking expectation over $\cV_{k,t-1}$ of both sides yields
  \begin{align*}
    \E_{\cV_{k,t} } [\phi_{k,t+1}\Ind(\cE_{k,t}) \mid \cF_{k,t}] 
    &\leq    (1-\gamma) \E_{ \cV_{k,t-1} }  \left[\phi_{k,t} \Ind(\cE_{k,t-1})\mid \cF_{k,t-1}  \right]
    \leq   (1-\gamma)^{t} R   \;.
\end{align*}
\end{proof}

\subsubsection{Combining Everything Together}\label{sec:combine}
We now put together the ingredients of the previous subsections to complete the proof of \Cref{thm:multi-pass}. 

\begin{proof}[Proof of \Cref{thm:multi-pass}]
It suffices to show that the conclusion of the theorem holds with a small constant probability, 
as this probability can be boosted arbitrarily by repeating the algorithm and selecting the output $u$ that maximizes $u^\top \s u$ (more precisely, 
an estimate of this variance that can be obtained by \Cref{it:proj-var} of \Cref{lem:stability-implications}).
For completeness, we give the bounds on runtime in \Cref{sec:runtime}.

First of all, observe that if the algorithm returns a vector using \solGenerator, then the resulting vector satisfies the desired guarantees with high probability.
We use the same notation as in the statement of \Cref{lem:one-round}.
We will show that at the start of any iteration of the outer loop (i.e., for
every $k\leq k_\mathrm{end}$ and $t=1$) we have that
$\E_{\cV_{k-1,t_\mathrm{end}}}[ \phi_{k,1} \Ind(\cE^{(1)}_{k-1,t_{\mathrm{end}}}) ]
  \leq (d/\eps)^{C_2 \log d} \|\s\|_\op^{2p_k +1}$ for a large enough positive constant $C_2$.
We do this by induction on $k$.
At the start of the algorithm ($k=1$), because of Line \ref{line:naive-prune}, every
point has norm $\|x\|_2^2 \leq 10(d/\eps)\widehat{\sigma}_\op \leq (20(d^2/\eps))\|\s\|_\op$, the potential is bounded
as follows
\begin{align*}
    \phi_{1,1} &\leq d \|  \s_{k,t} \|_\op^{2p+1} 
    \leq \left(20(d^2/\eps)\|\s\|_\op\right)^{2p+1}
    \leq (d/\eps)^{10p}  \|\s\|_\op^{2p+1} \\
    &=(d/\eps)^{10C' \log d}\|\s\|_\op^{2p_1+1} \leq  (d/\eps)^{C_2\log d}\|\s\|_\op^{2p_1+1}\;.
\end{align*}
Under the induction hypothesis, suppose the desired conclusion holds up to (and including) some $k\geq 1$.
Let us first relate $\phi_{k,t_\mathrm{end}+1}$ and $\phi_{k+1,1}$ as follows:
\begin{align}
    \phi_{k+1,1} &= \tr\left( \vec B_{k, t_\mathrm{end}+1}^{2p_{k+1}+1} \right)
    = \| \vec B_{k, t_\mathrm{end}+1} \|_{2p_{k+1}+1}^{2p_{k+1}+1} 
    \leq \| \vec B_{k, t_\mathrm{end}+1} \|_{2p_{k}+1}^{2p_{k+1}+1} \notag \\
    &= \tr\left( \vec B_{k, t_\mathrm{end}+1}^{2p_k+1} \right)^{{(2p_{k+1}+1)}/{(2p_{k}+1})}
    = \phi_{k,t_{\mathrm{end}+1}}^{{(2p_{k+1}+1)}/{(2p_{k}+1})} \;. \label{eq:doubling}
\end{align}
By induction hypothesis, at the beginning of the $k$-th iteration of the outer loop, it
holds that $\E_{\cV_{k-1,t_\mathrm{end}}}[ \phi_{k,1}
    \new{\Ind(\cE^{(1)}_{k-1,t_{\mathrm{end}}})} ] \leq (d/\eps)^{C_2 \log d}\|\s\|_\op^{2p_{k}+1}$.
In the $k$-th iteration of the outerloop,
we will show that the end of the inner loop (i.e., at
$t=t_{\mathrm{end}}+1$), the potential will be much smaller.
In particular, we will show that $    \E_{\cV_{k,t_\mathrm{end}} } \left[\phi_{k,t_\mathrm{end}+1}\Ind(\cE^{(1)}_{k,t_\mathrm{end}})  \right] \leq  (d/\eps)^{0.1 C_2 \log d}\|\s\|_\op^{2p_{k-1}+1}$ below.
Note that in the $k$-th iteration of the algorithm, two cases might happen: (i) either the condition $\langle \s_t, \vec M^2_{k,t} \rangle < (1 - 250 \gamma) \langle \s_{k,t}, \vec M^2_{k,t} \rangle$ holds for all $t_\mathrm{end}$ iterations of the inner loop, which would lead to a $(1 - \gamma)^{t_\mathrm{end}}$ decrease in potential, or this condition fails to hold for some $t$, which itself implies that the potential is small.
Recall that $\cE_{k,t}^{(2)}$ corresponds to these two cases.
Thus, we have the following decomposition:
\begin{align}
    \E_{\cV_{k,t_\mathrm{end}} } \left[\phi_{k,t_\mathrm{end}+1}\Ind(\cE^{(1)}_{k,t_\mathrm{end}})  \right] 
    &=\E_{\cV_{k,t_\mathrm{end}} } \left[\phi_{k,t_\mathrm{end}+1}\Ind(\cE^{(1)}_{k,t_\mathrm{end}}) \Ind(\cE^{(2)}_{k,t_\mathrm{end}}) \right] \notag \\
    &\quad + \E_{\cV_{k,t_\mathrm{end}} } \left[\phi_{k,t_\mathrm{end}+1}\Ind(\cE^{(1)}_{k,t_\mathrm{end}}) \Ind\Big(\,\overline{\cE^{(2)}_{k,t_\mathrm{end}}}\,\Big) \right] \label{eq:stepp1} 
\end{align}
The first term in \Cref{eq:stepp1} van be bounded using \Cref{cor:entire-inner-loop} as follows:
\begin{align}
    \E_{\cV_{k,t_\mathrm{end}} } \left[\phi_{k,t_\mathrm{end}+1}\Ind(\cE^{(1)}_{k,t_\mathrm{end}}) \Ind(\cE^{(2)}_{k,t_\mathrm{end}}) \right] &\leq (1-\gamma)^{t_\mathrm{end}+1} (d/\eps)^{C_2\log d} \|\s\|_\op^{2p_k+1} \\
    &\leq e^{-0.1 \cdot t_\mathrm{end}\cdot \gamma} (d/\eps)^{\new{C_2 \log d}} \|\s\|_\op^{2p_k+1},
    \label{eq:term1-finalproof}
    \end{align}
    which is less than $(d/\eps)^{0.05 C_2 \log d}$ after  $t_\mathrm{end}:=C\log^2(d/\eps)/\gamma$ for
sufficiently large $C$ (\new{note that we have picked $C$ to be much larger than $C_2$}).

We now turn our attention to the second term in \Cref{eq:stepp1}, where we show that if $\cE^{(2)}_{k,t_\mathrm{end}}$ does not hold, then the potential will be $\poly(d/\eps)\|\s\|_\op^{2p_k+1}$.
More formally, if for some $(k,t)$ we have that $\langle \s, \vec M_{k,t}^2  \rangle \geq (1-250\gamma)\langle \os_{k,t}, \vec M_{k,t}^2  \rangle$, there exists $c$ such that 
\begin{align*}
  \phi_{k,t} =  \| \vec B_{k,t} \|_{2p_k + 1}^{2p_k+1} &= \langle \vec B_{k,t},  \vec M_{k,t}^2 \rangle
    \leq \langle \vec \s_{k,t},  \vec M_{k,t}^2 \rangle 
    \leq \frac{1}{1-250\gamma} \langle \vec \s,  \vec M_{k,t}^2 \rangle \\
    &\leq e^{c \gamma} \| \s\|_\op \|\vec B_{k,t} \|_{2p_k}^{2p_k} \tag{using $\gamma \ll 1$}\\
    &\leq e^{c \gamma} \| \s\|_\op \|\vec B_{k,t} \|_{2p_k+1}^{2p_k} d^{\frac{1}{2p_k + 1}} \;. \tag{using \Cref{fact:normReln}}
\end{align*}
Rearranging the inequality above implies that $\|\vec B_{k,t}\|_{2p_k+1}^{2p_k+1} \leq e^{c \gamma p_k} d \|\s\|_\op^{2p_k+1} \leq (d/\gamma)^{O(1)} \new{\|\s\|_{\op}^{2p_k+1}}$, where the last inequality used that $p_k=O(\log(d/\gamma)/\gamma)$ for all $k$.
Combining this with \Cref{eq:stepp1} and \Cref{eq:term1-finalproof},
we obtain that $    \E_{\cV_{k,t_\mathrm{end}} } \left[\phi_{k,t_\mathrm{end}+1}\Ind(\cE^{(1)}_{k,t_\mathrm{end}})  \right] 
 \leq (d/\eps)^{0.1C_2 \log d}\new{\|\s\|_\op^{2p_k+1}}$.
Combining this with \Cref{eq:doubling} and the observation that ${(2p_{k+1}+1)}/{(2p_{k}+1}) \leq 2$ by definition of $p_k$ (Line \ref{line:pk} in \Cref{alg:few_iter-advanced}), we conclude
that
\begin{align*}
    \E_{\cV_{k,t_\mathrm{end}} } \left[\phi_{k+1,1}\Ind(\cE^{(1)}_{k,t_\mathrm{end}})  \right] 
    \leq \|\s\|_\op^{2p_{k+1}+1} (d/\eps)^{C_2 \log d} \;.
\end{align*}
Thus far, we have shown that $\E_{\cV_{k-1,t_\mathrm{end}} }
  [\phi_{k,1}{\Ind(\cE^{(1)}_{k-1,t_\mathrm{end}})} ]
  \leq \|\s\|_\op^{2p_{k}+1} (d/\eps)^{C_2 \log d}$ for $k=1,\ldots,
  k_{\mathrm{end}}$.
  
It remains to analyze the last iteration $k= k_{\mathrm{end}}$.
For that, we again use \Cref{cor:entire-inner-loop} to obtain the following:
\begin{align}
        \E_{\cV_{k_\mathrm{end},t_\mathrm{end}}} \left[ \phi_{k_\mathrm{end},t_\mathrm{end}+1} \Ind(\cE_{k_\mathrm{end},t_\mathrm{end}} ) \right] 
        &\leq (1-\gamma)^{t_\mathrm{end}+1 } (d/\eps)^{C_2 \log d} \|\s\|_\op^{1+2p_{k_\mathrm{end}}} \notag\\ 
        &\leq (d/\eps)^{-100 \new{C'}} \|\s\|_\op^{1+2p_{k_\mathrm{end}}} \;,
    \end{align}
where the last step uses that $t_\mathrm{end}:=C\log^2(d/\eps)/\gamma$ for
sufficiently large $C$ and recalling that the constant $C$ is large relative to the constant $C'$ (c.f.\ Line \ref{line:constants}).

Recall the definition of the event $\cE_{k,t}$ given in the statement
of \Cref{lem:one-round} as intersection of the two events:
\begin{enumerate}
  \item $\cE_{k,t}^{(1)} $: $(1-\eps) \E_{X \sim G}[1-w_{k',t'}(X)] + \eps \E_{X
            \sim B}[w_{k',t'}(X)] \leq 3 \eps$, for every iteration $(k',t')$ prior to (and
        including) $(k,t)$.
  \item $\cE_{k,t}^{(2)}: \langle \s, \vec M_{k',t'}^2  \rangle < (1-250\gamma)\langle \os_{k',t'}, \vec M_{k',t'}^2  \rangle$ for every iteration $(k',t')$ from $(k,1)$ up to (and including) $(k,t)$.
\end{enumerate}
By \Cref{lem:martingale} we have that
$\pr[\cE_{k_\mathrm{end},t_\mathrm{end}}^{(1)}] > 0.2$.
  Thus
\begin{align*}
    \E_{\cV_{k,t_\mathrm{end}}} \left[ \phi_{k_\mathrm{end},t_\mathrm{end}+1} \Ind(\cE^{(2)}_{k_\mathrm{end},t_\mathrm{end}} ) \mid \cE_{k_\mathrm{end},t_\mathrm{end}}^{(1)}\right] 
    =    \frac{\E_{\cV_{k,t_\mathrm{end}}} \left[ \phi_{k_\mathrm{end},t_\mathrm{end}+1} \Ind(\cE_{k_\mathrm{end},t_\mathrm{end}} ) \right]}{\pr[\cE_{k_\mathrm{end},t_\mathrm{end}}^{(1)}]} 
    \leq 5 (d/\eps)^{-100 C'}  \|\s\|_\op^{1+2p_{k_\mathrm{end}}}
\end{align*}
Therefore, we have that
\begin{align}
    &\pr\left[  \phi_{k_\mathrm{end},t_\mathrm{end}+1} \Ind(\cE^{(2)}_{k_\mathrm{end},t_\mathrm{end}} ) > 500 (d/\eps)^{-100 C'}  \|\s\|_\op^{1+2p_{k_\mathrm{end}}} \; \text{or} \; \Ind(\cE_{k_\mathrm{end},t_\mathrm{end}}^{(1)})=0 \right] \notag \\
    &\leq \pr[\Ind(\cE_{k_\mathrm{end},t_\mathrm{end}}^{(1)})=0] + \pr\left[  \phi_{k_\mathrm{end},t_\mathrm{end}+1} \Ind(\cE^{(2)}_{k_\mathrm{end},t_\mathrm{end}} ) > 500 (d/\eps)^{-100 C'}  \|\s\|_\op^{1+2p_{k_\mathrm{end}}} \mid \cE_{k_\mathrm{end},t_\mathrm{end}}^{(1)} \right] \notag \\
    &\leq 0.2 + 1/100 \;, \label{eq:twoevents}
\end{align}
where the last step uses Markov's inequality.%

Let $t^*$ be the first time during the $k_\mathrm{end}$-th outer loop iteration
of \Cref{alg:few_iter-advanced} such that $\langle \s, \vec M_{k_\mathrm{end},t}^2
  \rangle \geq (1-250\gamma)\langle \os_{k_\mathrm{end},t}, \vec
  M_{k_\mathrm{end},t}^2 \rangle$ (or equivalently $\Ind
  (\cE_{k_\mathrm{end},t}^{(2)})=0$).
In the following, we argue that $t^* \leq t_\mathrm{end}$.
We prove this by contradiction.
Assume that $\Ind (\cE^{(2)}_{k_\mathrm{end},t_\mathrm{end}} ) = 1$.

In the event that $\phi_{k_\mathrm{end},t_\mathrm{end}+1}
  \Ind(\cE^{(2)}_{k_\mathrm{end},t_\mathrm{end}} ) < 500 (d/\eps)^{-100 C'}$ and
$\Ind(\cE_{k_\mathrm{end},t_\mathrm{end}}^{(1)})=1$ simultaneously (by
\Cref{eq:twoevents} the two events will hold simultaneously with probability at
least $0.79$), we have that
\begin{align}
    \phi_{k_\mathrm{end},t_\mathrm{end}+1} &= \phi_{k_\mathrm{end},t_\mathrm{end}+1}  \Ind(\cE^{(2)}_{k_\mathrm{end},t_\mathrm{end}} ) \notag \\
    &< 500 (d/\eps)^{-100 C'} \left\| \s \right\|_{\op}^{2p_{k_{\mathrm{end}}}+1} \notag \\
    &< 0.5 (1-2\gamma)^{2p_{k_{\mathrm{end}}}} \left\| \s \right\|_{\op}^{2p_{k_{\mathrm{end}}}+1}  \notag \\
    &\leq \E_{X \sim P}[w_{k,t}(X)]  (1-2\gamma)^{2p_{k_{\mathrm{end}}}}\frac{1}{1-250\gamma} \left\| \s \right\|_{\op}^{2p_{k_{\mathrm{end}}}+1} \;, \label{eq:pot-small}
    \end{align}
where the first line uses the assumption
$\Ind(\cE^{(2)}_{k_\mathrm{end},t_\mathrm{end}} ) = 1$, the third line uses
$p_{k_{\mathrm{end}}} = C'\log(d/\gamma)/\gamma$, and the fourth line uses that $\E_{X
    \sim P}[w_{k,t}(X)]\geq 1/2$ because the event
$\cE_{k_\mathrm{end},t_\mathrm{end}}^{(1)}$ holds.

\Cref{lem:RHSbound} and \Cref{eq:pot-small} imply that $\langle \s, \vec M_{k_\mathrm{end},t}^2  \rangle \geq (1-250\gamma)\langle \os_{k_\mathrm{end},t}, \vec M_{k_\mathrm{end},t}^2  \rangle$, which yields a contradiction.
Therefore, $t^* \leq t_\mathrm{end}$.

To complete the proof it remains to show that if the algorithm hasn't terminated earlier with a good solution, then, during the ($k_\mathrm{end},t^*$)-th
iteration, the algorithm will return a good approximation of the top eigenvalue
of $\s$ and terminate.
We use \Cref{lem:certificate1} for that (the lemma is applicable since its
requirement $\E_{X \sim G}[w_{k,t}(X)]\geq 1 - 3\eps$ follows by the fact that
we have conditioned on the event $\cE_{k_\mathrm{end},t_\mathrm{end}}$, which
as we saw earlier happens with constant probability).
The conclusion of the lemma implies that there is probability 0.9 that when $t = t^*$, the output of \solGenerator satisfies $u^\top \s u > (1-O(\gamma))
  u^\top \s_{k,t} u$ and $ \frac{u^\top \os_{k,t} u}{\|u\|_2^2} \geq (1-\gamma)\|
  \s_{k,t} \|_\op$.
Conditioning on this event and by noting that the estimator $\widehat{r}_t$
of Line \ref{line:strong-estimator-basic} will satisfy $\widehat{r}_t \geq
  (1-\gamma)\|\s\|_\op$ with high probability over all the iterations, the check of Line \ref{line:stoppingcond} will be activated, and by \Cref{lem:certificate},
the algorithm will return a vector $u$ such that $u^\top \s u/\|u\|_2^2 \geq
  (1-O(\gamma)) \| \s \|_\op$.

\subsubsection{Runtime Analysis}\label{sec:runtime}

Let the input distribution $P$ be the uniform distribution over $n$ points in
$\R^d$.
The outer loop is repeated $k_\mathrm{end}$ times and the inner loop is
repeated $t_\mathrm{end}$ times, where $k_\mathrm{end}=O(\log(1/\gamma))$, and
$t_\mathrm{end}=O(\log^2(d/\eps)/\gamma)$.

Inside each loop, the runtime is determined by the following:
Calculating $v_{k,t}$ in \ref{line:choosev} can be implemented in time $O(n d
  p_k)$ by starting with $z_{k,t}$ and repeatedly multiplying it by $\vec
  B_{k,t}$ (multiplication of a vector $z$ with a second moment matrix $\sum_{x}
  x x^\top$ can be implemented in $O(nd)$ time by first calculating the inner
product inside the parenthesis $\sum_{x} x (x^\top z)$, and then calculating
the average of the resulting vectors).
The power iteration estimator of Line \ref{line:strong-estimator-basic} in \Cref{alg:sol-gen}  runs in time $O(\frac{n
    d}{\gamma} \log(d \cdot k_\mathrm{end}\cdot  t_{\mathrm{end}}/\gamma ))$ (and is being called at most
$k_\mathrm{end}\cdot t_\mathrm{end}$ many times).
It remains to analyze the runtime of
$\textsc{HardThresholdingFilter}$ (\Cref{alg:filter-appendix}): We have that the scores $\tau_{k,t}(x)$ that are not zeroed out by the thresholding satisfy the following upper and lower bound: $\poly(\eps/d)\|\s\|_\op \|v_{k,t} \|_2^2 \leq 0.1 (\new{\widehat{\sigma}_\op/d)} \| v_{k,t} \|_2^2 \leq  \tau_{k,t}(x)\leq \|x\|_2^2  \|v_{k,t} \|_2^2 \leq 2 \widehat{\sigma}_\op (d^4/\eps) \|v_{k,t}\|_2^2 \leq \poly(d/\eps) \| \s \|_\op \|v_{k,t}\|_2^2 $, where we used the  pruning of Line \ref{line:naive-prune}, the definition of Line \ref{line:max}, and the na\"ive estimator of Line \ref{line:naive-est}.
Since the \textsc{HardThresholdingFilter} in expectation halves the maximum
value of $\tau(x)$, the number of filtering steps it performs before
terminating will be in expectation $O(\log(d/\eps))$, and thus by Markov's
inequality, the contribution to the runtime from all the $k_\mathrm{end}\cdot
  t_\mathrm{end}$ executions of \textsc{HardThresholdingFilter} will be $O(nd
  \cdot k_\mathrm{end}t_\mathrm{end}\log(d/\eps))$ with high constant
probability.
\end{proof}

Therefore, the total runtime is
\begin{align*}
T 
&=  O\left(\frac{n d}{\gamma} k_\mathrm{end} t_\mathrm{end}\log( d \cdot k_\mathrm{end} \cdot t_{\mathrm{end}}/\gamma  )\right) +  O(nd \cdot k_\mathrm{end}t_\mathrm{end}\log(d/\eps)) +O\left( n d \cdot  t_{\mathrm{end}} \sum_{k=1}^{k_\mathrm{end}} p_k \right)   \\
&= O \left( \frac{nd }{\gamma^2}\log(1/\gamma)\log^2(d/\eps)  \log\left( \frac{d}{\gamma}\log(d/\eps)  \right)   + \frac{nd }{\gamma}\log(1/\gamma)\log^3(d/\eps)  +   \frac{nd }{\gamma^2}\log^2(d/\eps)\log(d/\gamma)     \right)\\
&= O \left( \frac{nd }{\gamma^2}\log^4(d/\eps)  \right) \;.
\end{align*}

\section{A Streaming Algorithm for Robust PCA}
\label{sec:streaming-main}

\paragraph{Organization} In this section, we present the streaming version of our nearly-linear time robust PCA algorithm. We provide the statement in \Cref{thm:streaming} and discuss the differences in comparison to the setting before and the adaptation needed to be made in our proof technique. This results in a set of additional deterministic conditions listed in \Cref{cond:deterministic}, which are established in \Cref{sec:estimators} and \Cref{sec:streaming}.

We work under the standard single-pass streaming model, where instead of having a fixed dataset, the algorithm draws samples from the distribution in an online manner. 
Let $G$ be an $(20\eps,\gamma)$-stable distribution, which will be the underlying distribution of inliers. 
Let the ``contaminated'' distribution be $P$, which is assumed to be an $\eps$-corrupted version of $G$ in total variation (c.f. \Cref{def:oblivious}). 
Note that up to some small change in the constant in front of $\eps$, we can equivalently use a mixture representation $P = (1-\eps)G + \eps B$ (see discussion below \Cref{def:stability}). 
In contrast to the previous section, where the input dataset was stored in essentially ``free'' memory, now the data access model is that $n$ samples are drawn i.i.d.\ from $P$ and presented to the algorithm one by one (\Cref{def:streaming}).

The algorithm is still allowed to maintain a local memory, but anything stored in the local memory will be counted against its space complexity.
The goal is again to find a unit vector $u$ such that $u^\top \s u \geq (1-O(\gamma))\| \s \|_\op$.

\begin{restatable}{theorem}{Streaming}
  \label{thm:streaming}
  Let an integer $d>2$, and reals $0<20 \eps<\gamma < \gamma_0$, for a
  sufficiently small $\gamma_0$. Let $G$ be $(20\eps,\gamma)$-stable distribution (\Cref{def:stability}) with respect to a
  PSD matrix $\s \in \R^{d \times d}$, and $r \geq 1$ be a radius such that $\pr_{X \sim G}[\|X\|_2 > r \sqrt{d\|\s\|_\op}]\leq \eps$. Let $P$ be a distribution with $\dtv(P,G) \leq \eps$.
  There exists an algorithm takes $\eps,\gamma,r$ as input, uses a stream of
  \begin{align*}
      n \lesssim \left(\new{\frac{r^2 d^2 }{\gamma^5 }}  + \frac{1}{\eps \gamma}   \right) \polylog(d/\eps) \;.
  \end{align*}
  i.i.d.\ samples from $P$ (cf.\ \Cref{def:streaming}), 
  uses additional memory of storing $O\left( (d/\gamma+1/\eps)\polylog(d/\eps) \right)$ many real numbers (or a bit complexity of $(d/\gamma^2)\polylog(d/\eps)$ in the word RAM Model),
  runs for $O(\frac{nd}{\gamma^2}  \polylog(d/\eps))$ time, and with probability at least $0.99$ outputs a
  vector $u$ such that $u^\top \s u \geq (1-O(\gamma))\| \vec \Sigma \|_\op$.
\end{restatable}

The algorithm realizing the above theorem is outlined in \Cref{alg:streaming,alg:sol-gen2}.
The specialization of the above theorem for subgaussians, where $\gamma = O(\eps \log(1/\eps))$ and $r = O(\sqrt{\log(1/\eps)})$ yields sample complexity \new{$O((d^2/\eps^5)\polylog(d/\eps))$.}
Moreover, the memory usage stated in \Cref{thm:streaming} is in terms of the number of \emph{real numbers} that need to be stored in the algorithm's memory. See \Cref{sec:bit-complexity} for how our algorithm can be implemented with bounded precision (i.e., in the standard word RAM model) using $(d/\gamma)\polylog(d/\eps)$ registers of size $(1/\gamma)\polylog(d/\eps)$ bits each (for a total bit complexity of $(d/\gamma^2)\polylog(d/\eps)$).

\paragraph{Differences in Streaming Setting} 
Drawing inspiration from techniques in \cite{DKPP22}, we see that the
algorithm from the previous section is partly already amenable to the
streaming setting: The filters used are of the form $\Ind(v^\top x > L)$ which have a compact representation of $O(d)$ space (it suffices to store the vector $v$ and the threshold $L$). The potential-based analysis also showed that we create at most $O(\polylog(d/\eps)/\gamma)$ many such filters. The remaining adaptation that needs to be done is to deal with the fact that there is no fixed dataset to iterate over.
 The new algorithm is given in \Cref{alg:streaming}. Regarding notation, the quantities $P_{k,t},\os_{k,t},\vec B_{k,t}, \vec M_{k,t}$ as well as the score functions $g_{k,t}(x) = \|\vec M_{k,t} x\|_2^2  , f_{k,t}(x) = ( v_{k,t}^\top x)^2, \tau_{k,t}(x)= f_{k,t}(x)\Ind(f_{k,t}(x) > L_{k,t}) $ are all population-level quantities, i.e., they are unknown to the algorithm. However, the algorithm can approximate them by drawing samples and forming estimators (where the exact form of the estimators will have to be easily computable in the streaming setting and will be discussed later on). We will use ``hat'' to denote the relevant sample-based approximations, i.e., $\widehat{\s}_{k,t}, \widehat{\vec B}_{k,t}, \hatm_{k,t}$ and the induced scores $\widehat{g}_{k,t}(x) = \|\widehat{\vec M}_{k,t} x\|_2^2  , \widehat{f}_{k,t}(x) = ( \widehat{v}_{k,t}^\top x)^2, \widehat{\tau}_{k,t}(x)= \widehat{f}_{k,t}(x)\Ind(\widehat{f}_{k,t}(x) > \widehat{L}_{k,t}) $. The approach that we follow is that if the sample-based quantities are sufficiently close to their population-level counterparts,  then the proof of correctness from the previous section (which involves the population-level quantities)  still applies. One can easily go through the proofs of the previous section and verify that there is enough slack in all bounds to allow for converting between population-level quantities and their sample-based approximations. In this section, we avoid repeating all the correctness proofs since the only changes will be in the constants used in some inequalities. Instead, we mostly focus on deriving the sample complexity needed for obtaining fine enough approximations.

\paragraph{Sample-based Estimators } Many steps of the algorithm of the previous section involved multiplying a vector $x$ by $\vec M_{k,t} := \vec B_{k,t}^{p_k}$, for example the scores $g_{k,t}(x)$ involve the norm of $\vec M_{k,t}x$. This was done by starting with $x$ and repeatedly multiplying by $\vec B_{k,t}$. Instead of $\vec B_{k,t}$ the new algorithm will now use an empirical moment matrix $\widehat{\vec B}_{k,t}$. As we have seen in the previous section, the multiplication of an empirical second moment matrix with a vector can be performed with $O(d)$ memory usage in a single pass over the samples and in linear time. However, since the algorithm cannot store this matrix, it will repeatedly multiply $x$ by fresh estimators $\widehat{\vec B}_{k,t,1},\ldots, \widehat{\vec B}_{k,t,p_k}$ each time. Thus, the result will be of the form $\widehat{\vec M}_{k,t} z$ where $\widehat{\vec M}_{k,t} = \prod_{\ell=1}^{p_k} \widehat{\vec B}_{k,t,\ell}$ (also see \Cref{alg:matrix-power-est} for more detail).

\begin{algorithm}
  \caption{Estimator of $\vec B_{k,t}^p$ from minibatches.}
  \label{alg:matrix-power-est}
  \begin{algorithmic}[1]
    
    \State  Draw a batch $S_0$ of $\tilde{n}$ samples from $P$ 
and let the estimate $\widehat{W}_{k,t} = \E_{X \sim \cU(S_0)}[w_{k,t}(X)]$. \label{step:1}
    \State Draw $p$ batches $S_1,\ldots, S_{p}$ of $\tilde{n}$ samples, each from $P_{k,t}$. \label{it:draw_samples}
\For{$\ell \in [p]$}
\State Let $\widehat{\vec{\Sigma}}_{k,t,\ell} = \frac{1}{\tilde{n}}\sum_{x \in S_\ell} x x^T$. \label{step:sigma}
        \State  Let $\widehat{\bB}_{k,t,\ell} = \widehat{W}_{k,t}^2\widehat{\vec{\Sigma}}_{k,t,\ell}$.
\EndFor
\State  \textbf{return} $\widehat{\bM}_{k,t,\ell} = \prod_{\ell=1}^{p} \widehat{\bB}_{k,t,\ell}$. \label{step:5}
  \end{algorithmic}
\end{algorithm}

We will additionally need two new memory-efficient estimators that work in the streaming model.
First, we need an estimator for the quantiles of the underlying distribution to replace Line \ref{line:quantile} of the old algorithm. These will be computed by empirical quantiles. Second, we also need a memory-efficient way to evaluate the stopping condition inside the \textsc{HardThresholdingFilter} (\Cref{alg:filter-appendix}) because the stopping condition, $\E_{X \sim P}[w(x)\tau(x)]$, is a population-level quantity; a similar estimator is also needed to adapt the Line \ref{line:trimmed-mean} of our old algorithm. 
This completes the short informal overview.

\begin{algorithm}[h]
	\caption{\textsc{RobustPCA} in streaming model}
	\label{alg:streaming}
	\begin{algorithmic}[1]
		\State \textbf{Input}: $P,\eps,\gamma$.
		\State \new{Let $C,C'$ be sufficiently large absolute constants with their ratio $C/C'$ being sufficiently large.} \label{line:constants}
		\State \new{Let an estimation $R$ be such that $|\pr_{X \sim G}[\|X\|_2 \geq R] - \eps | \leq 2\eps$ (note that $R \leq r \sqrt{d \| \s \|_\op}$)}  \label{line:radius-est}
        \State Initialize  $w_{1,1}(x)= \Ind\left(\|x\|_2 \leq R \right)$. \label{line:weights-init}
		\State Let $k_\mathrm{end}:= \log((\log(d/\gamma)/\log(d))/\gamma)$, and $t_{\mathrm{end}}:=\frac{C  \log^2(d/\eps)}{\gamma}$.
		\State Find estimator $\widehat{\sigma}_{\op} \in (0.8\|\s\|_{\op} , 2 d\|\s\|_{\op})$.
		\label{line:naive-est}
		\hfill  \Comment{\new{{c.f. \Cref{it:proj-var} of \Cref{lem:stability-implications} with $\bU = \bI$}}.}
		\For {$k = 1,\ldots, k_{\mathrm{end}}$}
		\State \new{Let $p_k= 2^{k-1} p$, where $p = C'\log(d)$.}\label{line:pk} \hfill  \Comment{\new{$p_k$ ranges from $C'\log(d)$ to $C'\log(d/\gamma)/\gamma$}.}
		\For{ $t = 1,\ldots, t_{\mathrm{end}}$}
		
		\State Let $P_{k,t}$ be the distribution of $P$ weighted by $w_{k,t}$: $P(x)w_{k,t}(x)/ \E_{X \sim P}[w_{k,t}(X)] $.
		\State{ Let $\vec B_{k,t} := \E_{X\sim P} [w_{k,t}(X)XX^\top]$ and $\vec M_{k,t} := \vec B_{k,t}^{p_k}$.
		\label{line:definitions}}  
		\State Let ${g}_{k,t}(x) :=\|\vec M_{k,t} x\|^2_2$.
		    \State Let $\widehat{\vec M}_{k,t}$ be a sample-based version of $\vec M_{k,t}$ as defined in \Cref{alg:matrix-power-est}. \label{line:matrix-power-est}
    \State $\widehat{v}_{k,t} \gets \vec \hatm_{k,t} z_{k,t}$, where $z_{k,t} \sim \cN(0,\bI)$.
    \label{line:choosev2}
    \State Let $\widehat{f}_{k,t}(x) = (\widehat{v}_{k,t}^\top x)^2$.
    \State{Let $\widehat{L}_{k,t}$ be estimator for the $3\epsilon$-quantile of $\widehat{f}_{k,t}(\cdot)$ under $P_{k,t}$ such that $|\pr_{X \sim P_{k,t}}[\widehat{f}_{k,t}(X) > \widehat{L}_{k,t}] - 3\eps|\leq 0.01\eps$, and then do $\widehat{L}_{k,t} \gets \max\{\widehat{L}_{k,t}, \frac{0.1}{d} \widehat{\sigma}_\op\|v_{k,t}\|_2^2 \}$.}\label{eq:L}
		\State Let $\widehat{\tau}_{k,t}(x) = \widehat{f}_{k,t}(x) \Ind ( \widehat{f}_{k,t}(x) > \widehat{L}_{k,t} )$.
    \State Call  \solGenerator$(P,w_{k,t},\eps,\gamma,1/(k_\mathrm{end}\cdot t_\mathrm{end}))$. \hfill  \Comment{c.f.\ \Cref{alg:sol-gen2}.}
		    \State Let $\widehat{\sigma}_{k,t}$ such that $|\widehat{\sigma}_{k,t} - \widehat{v}_{k,t}^\top \s \widehat{v}_{k,t}| \leq 4\gamma \widehat{v}_{k,t}^\top \s \widehat{v}_{k,t}$.
    \State (e.g.,  $\widehat{\sigma}_{k,t}:=\E_{X \sim P}[w_{k,t}(X)\widehat{f}_{k,t}(X)\Ind (\widehat{f}_{k,t}(X) \leq \widehat{L}_{k,t} )]$).    \hfill \Comment{c.f. \Cref{it:proj-var} of \Cref{lem:stability-implications}.}

    \State Find an estimator $\widehat{\sigma}_{k,t}'$ such that $|\widehat{\sigma}_{k,t}' - \widehat{\sigma}_{k,t}|\leq 0.01 \widehat{\sigma}_{k,t} + \new{\frac{0.01\gamma}{r^2 d}\|\widehat{v}_{k,t}\|_2^2 \| \s_{k,t} \|_\op}$.\label{line:sigma_prime}
    \State Let $\widehat{T}_{k,t}=2.35 \gamma\widehat{\sigma}_{k,t}'$.
    \State{$w_{k,t+1} \gets \mathrm{HardThresholdingFilter}(P,w_{k,t},\tau_{k,t},\widehat{T}_{k,t},R,\new{\frac{0.1\gamma}{r^2 d}\|\widehat{v}_{k,t}\|_2^2 \| \s_{k,t}})$, where $\widehat{\vec \Sigma}_{k,t}$ is a sample-based version of $\os_{k,t}$.} \hfill \Comment{c.f. \Cref{alg:filter-appendix}.}
		\EndFor
		\State Set $w_{k+1,0} \gets w_{k,t+1}$.
		\EndFor
		\State \textbf{return}
		FAIL.
		\label{line:filter_end-basic}

	\end{algorithmic}
\end{algorithm}

Formally,
we aggregate all the guarantees needed regarding the estimators in \Cref{cond:deterministic}. 
As a small technical note, the conditions in \Cref{cond:deterministic} are only needed to hold whenever $\E_{X \sim P}[w_{k,t}(X)] \geq 1- 3\eps $ and $\| \vec B_{k,t} \|_\op \geq 0.5 \| \s \|_\op$. 
We can assume the first condition holds throughout the course of the algorithm because we filter out mostly outliers. 
The second condition holds because we also know that if $\| \vec B_{k,t} \|_\op \geq 0.5 \| \s \|_\op$ is violated, then the potential function is small enough so that the algorithm must have already terminated  (c.f. \Cref{lem:RHSbound}) holds with high probability thought the algorithm  and dedicate \Cref{sec:estimators} to establishing it. 

\begin{condition}[Conditions for \Cref{alg:streaming}]\label{cond:deterministic}
In the context of \Cref{alg:streaming}, assume that the following are true. For every $k \in [k_\mathrm{end}]$ and $t \in [t_\mathrm{end}]$, if $\E_{X \sim P}[w_{k,t}(X)] \geq 1- 3\eps $ and $\| \vec B_{k,t} \|_\op \geq 0.5 \| \s \|_\op$, then: 
\begin{enumerate}
 
    \item $\hatg_{k,t}(x) \geq 0.5 g_{k,t}(x) - 0.01(\gamma/\eps)\| \vec M_{k,t} \|_\fr^2 \| \s \|_\op$. \label{it:g-scores}
    \item $\left|  \|\hatm_{k,t}\|_\fr^2 - \|\vec M_{k,t}\|_\fr^2   \right| \leq 0.01\|\vec M_{k,t}\|_\fr^2$.\label{it:fr-norm}

    \item The estimator $\widehat{L}_{k,t}$ of  Line \ref{eq:L} satisfies $\pr_{X \sim P_{k,t}}[\hatf_{k,t}(X) > \widehat{L}_{k,t}] \in (3.99\eps, 4.01\eps)$. \label{it:quantile_est}

    \item Recall the parameter $r$ as the radius such that $\pr_{X \sim G}[\|X\|_2 > r \sqrt{d\|\s\|_\op}]\leq \eps$. For any weight function $w: \R^d \to [0,1]$,  the algorithm has access to an estimator $\widehat{F}$ for the quantity $F_{k,t}:= \E_{X \sim P}[w(X) \widehat{f}_{k,t}(X)]$ that has accuracy $ | \widehat{F} - F_{k,t} | \leq 0.01\gamma F_{k,t} + \frac{0.01 \gamma}{d r^2} \| \widehat{v}_{k,t} \|_2^2 \|\os_{k,t}\|_\op$ across a total of $t_\mathrm{end}k_\mathrm{end}\cdot \log^5(d/\eps)$ calls. \label{it:scores-mean}
    
    \item The estimator $\widehat{r}_t$ of  Line \ref{line:strong-p-iteration} of \Cref{alg:sol-gen2} satisfies $\widehat{r}_t \geq (1-\gamma)\| \os_{k,t} \|_\op$. \label{it:strong-power-est}
    \item Every time   Line \ref{line:vectoru} of of \Cref{alg:sol-gen2} is executed, there is probability $0.9$ that $u^\top \os_{k,t} u/\|u\|_2^2 \geq (1-\gamma) \| \os_{k,t}\|_\op$. \label{it:power-est}
\end{enumerate}
\end{condition}
In particular, \Cref{it:g-scores} of \Cref{cond:deterministic}  shows that \Cref{assumption:full_points} that we used in proving the main theorem of the previous section is satisfied for \Cref{alg:streaming} in our current setting.
\Cref{it:fr-norm} is relevant to equation \Cref{eq:fullpoint} in our previous analysis. To adapt that in our  current setting, $v_{k,t}$  will be replaced by $\widehat{v}_{k,t}$ and after we take expectation over $\widehat{v}_{k,t}$ it's norm will become  $\|\hatm_{k,t}\|_\fr^2$ that we can  relate to $\| \vec M_{k,t} \|_\fr^2$.
\Cref{it:quantile_est} takes care of the quantile estimation.
\Cref{it:scores-mean} is the estimator that we will use to evalueate the stopping condition inside \textsc{HardThresholdingFilter} as well as adapting  line \ref{line:trimmed-mean} of the old algorithm.
\Cref{it:strong-power-est,it:power-est} are the adaptation of the power iteration guarantee when we do not have access to the target matrix $\os_{k,t}$ but instead we start with a random Gaussian vector and multiply it repetitively with fresh estimates $\widehat{\os}_{k,t}$.

\begin{algorithm}[h]
	\caption{\solGenerator2}
	\label{alg:sol-gen2}
	\begin{algorithmic}[1]
		\State \textbf{Input}: Distribution $P$, weights $w$, parameters $\eps,\gamma,\delta$.

		\State Let $\vec M:= (\E_{X \sim P}[w(X)XX^\top])^p$ for $p=
			C\frac{\log(d/\new{\gamma})}{\gamma}$.
		\State $\widehat{r} \gets 0$.
		\For{$j \in [\log(1/\delta)]$}
		\State Let $\hatm$ be a sample based estimator for $\vec \s_{P_w}^{p}$ calculated using \Cref{alg:matrix-power-est}.
		\State $y \gets \hatm' \cdot g$ for  $g \sim \cN(0,\bI)$.
		\State $\widehat{r} \gets \max(\widehat{r} \gets,y^\top \widehat{\s}_{P_w} y/ \|y\|_2^2)$. \label{line:strong-p-iteration}
		\EndFor
		\hfill\Comment{cf. \Cref{lem:approx-power-iter}}
\State Let $\hatm$ be a sample based estimator for $\vec \s_{P_w}^{p}$ calculated using \Cref{alg:matrix-power-est}.
		\State $u \gets \vec \hatm z$ for $z \sim \cN(0,\bI)$. \label{line:vectoru}
		\State{\new{Let $\widehat{\sigma}_u$ be any approximation such that $|\widehat{\sigma}_u - u^\top \s u| \leq 4 \gamma u^\top \s u$, for example $\widehat{\sigma}_u:=\E_{X \sim P}[w(X) (u^\top X)^2 \Ind((u^\top X)^2 \leq Q)]$, where $Q$ is within $0.01\eps$ from the $3\eps$-quantile of $(u^\top X)^2$ under $P_{w}$.}}
    \State \new{Find estimator $\widehat{\sigma}_u'$ such that $|\widehat{\sigma}_u'- \widehat{\sigma}_u| \leq 0.01\gamma \widehat{\sigma}_u + 0.01\gamma \|u\|_2^2 \|\s\|_\op$.} \label{line:avg-est}
    \State Let $\widehat{\vec \Sigma}_{P_w}$ be a sample-based version of $\os_{P_w}$.
    \If{$\widehat{\sigma}_{u}' \geq (1-C\gamma)u^\top \widehat{\vec \Sigma}_{P_w} u $ and $\frac{u^\top \widehat{\vec \Sigma}_{k,t} u}{\|u\|_2^2} \geq (1 - \gamma)  \widehat{r}_t $}    \label{line:when-we-return2}
		\State \textbf{return} $u/\|u\|_2$.  \hfill\Comment{c.f. \Cref{lem:basic-cert}}
	\EndIf
	\end{algorithmic}
\end{algorithm}

\subsection{Establishing \Cref{cond:deterministic} } \label{sec:estimators}

In this section, we focus on showing that \Cref{cond:deterministic} holds with high probability given enough samples.
The proofs of \Cref{it:g-scores,it:fr-norm,it:scores-mean} closely follow \cite{DKPP22} and are thus deferred to \Cref{sec:streaming}. \Cref{it:quantile_est} is a basic application of Chernoff-Hoeffding bounds and is also provided in Appendix. 
Thus, in this section, we focus on establishing \Cref{it:strong-power-est,it:power-est}, both of which are related to the power iteration method, as explained below.

For the standard power iteration method, we know that for a $d \times d$ PSD matrix $\vec A$,  $y^\top \vec A y / \|y\|_2^2$ is close to $\|\vec A\|_\op$ with high probability if $y = \vec A^p z$ for a large $p$ and $z \sim \cN(0,\vec I)$. 
Formally, we have that $y^\top \vec A y / \|y\|_2^2 \geq (1 - \gamma)\|\vec A\|_\op$ for $p = C \log(d/\gamma)/\gamma$;  cf.\ \Cref{fact:power-iter}.
However, in our analysis, we will be able to compute only $y = \widehat{\vec A} z$, where $\|\widehat{\vec A} - \vec A^p\|_\op \leq \delta \|\vec A\|_\op$ for a small $\delta$; Here, $\delta$ accounts for sampling error.
Despite having access to only $\widehat{\vec A}$, we still want that $y^\top \vec A y / \|y\|_2^2$ is at least $(1-O(\gamma)) \|\vec A\|_\op$.
This is obviously true for $\delta = 0$ by power iteration and we want to understand how large can $\delta$ be so that it is still true, since larger $\delta$ will lead to smaller sample complexity.
We state our result in \Cref{lem:approx-power-iter} below.

Once we have \Cref{lem:approx-power-iter}, then \Cref{it:strong-power-est,it:power-est} will follow as described next. For \Cref{it:power-est} we use the lemma below with $\vec A = \vec B_{k,t}$ and $\widehat{\vec A} = \hatm_{k,t}$, for which it is guaranteed by \Cref{lem:op-norm-closeness} that $\|  \widehat{\vec A} - \vec A^p \|_\op \leq \delta \| \vec A^p \|_\op \leq \delta \| \vec A^p \|_\fr$. \Cref{it:strong-power-est} also follows since we can boost the probability of success from $0.9$ to arbitrarily close to $1$ by repeating the procedure and returning the vector $y$ maximizing $y^\top \vec A y/\|y\|_2^2$. We now state and prove the lemma below:

\begin{lemma}\label{lem:approx-power-iter}
  Let $\delta>0$, $\gamma \in (0,1/2)$, $p \in \N$ and $d \times d$ PSD matrices $\vec A, \widehat{\vec A}$, 
  such that $\|  \widehat{\vec A} - \vec A^p \|_\op \leq \delta \| \vec A^p \|_\fr$. 
  Let $z \sim \cN(0,\bI)$ and $y = \widehat{\vec A} z$. 
  If $p > C \log(d/\gamma)/\gamma$ for a sufficiently large constant, and \new{$\delta < c \gamma/\sqrt{d}$} for sufficiently small positive constant, then with probability at least $0.9$, we have that  $\frac{y^\top \vec A y}{\|y\|_2^2} \geq (1-O(\gamma))\| \vec A \|_\op  \;.$
\end{lemma}

\begin{proof}
   Denote $\vec \Delta := \widehat{\vec A} - \vec A^p$, which has operator norm at most $\delta \|\vec A^p\|_\fr$.
   We first analyze the random variable \new{$U = z^\top \left(\vec A^{p} \vec \Delta +\vec \Delta \vec A^{p}  \right) z$.}
   Then 
   $$\E[U] = \trace \left(\vec A^{p} \vec \Delta +\vec \Delta \vec A^{p}  \right) = 2 \langle  \vec A^p, \vec \Delta \rangle \leq 2\| \vec A^p\|_\fr \| \vec \Delta\|_\fr \leq 2\sqrt{d} \|\vec A^p\|_\fr \| \vec\Delta\|_\op \leq 2\delta \sqrt{d}\|\vec A^p\|_\fr^2\,, $$
   since $\| \vec \Delta\|_\op \leq \delta \|\vec A\|_\fr$.
   Since $\vec A^{p} \vec \Delta +\vec \Delta \vec A^{p} $ is symmetric, by \Cref{fact:var-quadratic}, the variance of $U$ satisfies
   $\Var(U) \lesssim \| \vec A^{p} \vec \Delta +\vec \Delta \vec A^{p}  \|_\fr^2 \lesssim \|\vec \Delta\|_\op^2 \|\vec A^p\|_\fr^2 \leq \delta^2 \|\vec A^p\|_\fr^4$.
   By applying Chebyshev's inequality, we have that with probability at least $0.999$ 
   \begin{align}
       \label{eq:dummy-2}
   |U| =|z^\top \left(\vec A^{p} \vec \Delta +\vec \Delta \vec A^{p}  \right) z| \lesssim \delta \sqrt{d} \| \vec A^p\|_\fr^2.
   \end{align}
   By a similar argument, we obtain that with probability at least $0.999$, 
      \begin{align}
       \label{eq:dummy-3}
  |z^\top \left(\vec A^{p+1} \vec \Delta +\vec \Delta \vec A^{p+1}  \right) z| \lesssim \delta \sqrt{d} \|\vec A\|_\op \| \vec A^p\|_\fr^2. 
   \end{align}
    Now let $U' = z^\top \vec A^{2p} z$.
    Then $\E[U'] = \trace( \vec  A^{2p}) =  \| \vec A^p\|_\fr^2$.
    By Gaussian anticoncentration (\Cref{fact:quadratics-approximation}), the standard deviation of $U'$ is at most $O(\|\vec A^p\|_\fr)$, which gives an upper bound on $U'$ with high probability. 
    Combining this with \Cref{fact:quadratics-approximation}, we have that with probability $0.999$:
    \begin{align} \label{eq:gquadratic}
        z^\top  \vec  A^{2p} z \asymp \| \vec  A^p\|_\fr^2 \;.
    \end{align}
    Similarly, we obtain that with high probability, $\|z\|_2^2 \lesssim d$.
   Combining this with \Cref{eq:dummy-2,eq:dummy-3,eq:gquadratic}, we have that with probability at least $0.99$:
   \begin{align}
    \|y\|_2^2 &= z^\top \widehat{\vec A}^2 z \notag = z^\top   \left(  \vec A^p + \vec \Delta  \right)^2  z  \notag \\
    &= z^\top  \vec A^{2p} z +   z^\top  \left(\vec A^{p} \vec \Delta +\vec \Delta \vec A^{p}  \right) z  +  z^\top    \vec \Delta^2 z \notag \\
    &\leq z^\top  \vec A^{2p} z +  C \delta \sqrt{d} \| \vec  A^p\|_\fr^2   +  C d \| \vec  \Delta\|_\op^2 \notag \\
    &\leq   z^\top  \vec A^{2p} z + C\| \vec  A^p\|_\fr^2\left( \delta \sqrt{d} + d\delta^2\right) \notag\\
    &\leq   z^\top  \vec A^{2p} z + 2C \delta \sqrt{d}\| \vec  A^p\|_\fr^2 \;. \label{eq:denom}
\end{align}
Similarly, we obtain the following lower bound that with probability at least $0.99$: 
\begin{align}\label{eq:denom2}
    \|y\|_2^2 \geq  z^\top  \vec A^{2p} z - C\delta \sqrt{d} \| \vec A^p\|_\fr^2 \;.
\end{align}
By \Cref{fact:power-iter}, the standard power iteration implies that with probability $0.999$ we have that
\begin{align} \label{eq:std-power-iter}
    \frac{z^\top \vec A^{2p+1}z}{z^\top \vec A^{2p} z} \geq (1-\gamma) \|\vec A\|_\op \;.
\end{align}
At the intersection of all these events, which itself holds with probability at least $0.9$ by union bound, we obtain the following lower bound on $y^\top \vec A y$:
  \begin{align}
      y^\top \vec A y &= z^\top \widehat{\vec A} \vec A \widehat{\vec A} z \notag= z^\top \left(  \vec A^p + \vec \Delta  \right) \vec A \left(  \vec A^p + \vec \Delta  \right) z  \notag \\
      &= z^\top \vec A^{2p+1} z +  z^\top (\vec A^{p+1} \vec \Delta +\vec  \Delta \vec A^{p+1} )z  + z^\top (\vec \Delta  \vec A \vec \Delta  )z  \tag{using \Cref{eq:dummy-3} and $z^\top \Delta \vec A \Delta z \geq 0$} \\
      &\geq  z^\top \vec A^{2p+1} z - C \delta \sqrt{d} \| \vec  A\|_\op \| \vec  A^p\|_\fr^2  \notag \\
    &\geq (1-\gamma)\|\vec A \|_\op z^\top \vec A^{2p} z - C\delta \sqrt{d}  \|\vec A \|_\op \|\vec A^p\|_\fr^2  \tag{using \Cref{eq:std-power-iter}}\\
    &\geq (1-\gamma)\|\vec A \|_\op \left(  z^\top \vec A^{2p} z - 3C\delta \sqrt{d}  \|\vec A^p\|_\fr^2 \right) \tag{$\gamma<2/3$} \\
    &\geq (1-\gamma)\|\vec A \|_\op \left(  \|y\|^2 - 5C\delta \sqrt{d}  \|\vec A^p\|_\fr^2 \right) \;. \tag{using \Cref{eq:denom}}
  \end{align}
  Finally, dividing by $\|y\|_2^2$ both sides yields
  \begin{align*}
      \frac{y^\top \vec A y}{\|y\|_2^2} 
    &\geq (1-\gamma)\| \vec A \|_\op \left( 1 - 5 C\delta \sqrt{d}\| \vec A^p \|_\fr^2 \frac{1}{\|y\|_2^2} \right).
    \end{align*}
    We will now show that  $5 C\delta \sqrt{d}\| \vec A^p \|_\fr^2 \frac{1}{\|y\|_2^2} \lesssim \gamma$ on the intersection of the events mentioned above, which will complete the result: 
\begin{align*}
\delta \sqrt{d}\| \vec A^p \|_\fr^2 \frac{1}{\|y\|_2^2}
    &\leq 
    \frac{\delta \sqrt{d}\| \vec A^p \|_\fr^2} {z^\top \vec A^{2p}z - 2 C \delta \sqrt{d}\|\vec A^p\|_\fr^2}   \tag{using \Cref{eq:denom2}}\\
    &\lesssim \frac{\delta \sqrt{d}\| \vec A^p \|_\fr^2} { \|\vec A^p\|_\fr^2 - C' \delta \sqrt{d}\|\vec A^p\|_\fr^2} \tag{using \Cref{eq:gquadratic}} \\
    &=\frac{\delta \sqrt{d}} { 1 - C' \delta \sqrt{d}}  \lesssim \delta \sqrt{d} \lesssim \gamma. \tag{using $\delta < c \gamma /\sqrt{d}$}
  \end{align*}
\end{proof}

\subsection{Proof Sketch of \Cref{thm:streaming}} \label{sec:proof-of-streaming-thm}
The proof of correctness of \Cref{alg:streaming} is the same as in \Cref{sec:multi-pass} modulo small changes in the constants of the bounds to account for the usage of our sample-based estimators. Thus, here we only derive the sample complexity, memory, and runtime bounds.

\textbf{Sample complexity}: Consider the  $(k,t)$-th iteration of \Cref{alg:streaming} (that is, the $k$-th iteration of the outer loop and $t$-th iteration of the inner loop).
In order to ensure all parts of \Cref{cond:deterministic} we use \Cref{lem:op-norm-closeness} with $\delta =   \frac{0.01}{\sqrt{d}}\min\left(\frac{\sqrt{\gamma/\eps}}{r  },\gamma \right)$ and probability of failure, $\tau$, equal to a sufficiently large polynomial of $\eps/d$ so that with high probability the estimators are accurate across all iterations simultaneously.
Thus, each estimator $\widehat{\vec B}_{k,t,\ell}$ uses $(r^2   p_k^2 d/\delta^2)\polylog(d/\eps)$ samples. We use $p_k$ of these estimators in Line \ref{line:matrix-power-est} resulting in $(r^2  p_k^3 d/\delta^2)\polylog(d/\eps)$ samples. Moreoever, for the estimator of Line \ref{eq:L} we use \Cref{l:simple-partition} (again with $\tau = \poly(\eps/d)$) which contributes another $(1/\eps)\polylog(d/\eps)$ to the sample complexity. We also call  the estimator of \Cref{lem:averaging-est} which uses $O( (r^4 d^2/\gamma^2)\polylog(d/\eps))$ many samples.
Thus, the number of samples used during the $(k,t)$-th iteration of \Cref{alg:streaming} is $n_{k,t}:= O\left(r^2  p_k^3 d/\delta^2\polylog(d/\eps) + (1/\eps)\polylog(d/\eps) + (r^4 d^2/\gamma^2) \polylog(d/\eps) \right)$, and the total sample complexity is 
\begin{align*}
    n &= \sum_{k=1}^{k_\mathrm{end}} \sum_{t=1}^{t_\mathrm{end}} n_{k,t}  
    \lesssim t_\mathrm{end} \left(\sum_{k=1}^{k_\mathrm{end}} p^3_k \right)\frac{r^2 d}{\delta^2 }\polylog(d/\eps) +\frac{t_\mathrm{end} k_\mathrm{end}}{\eps}\polylog(d/\eps) + t_\mathrm{end} k_\mathrm{end}\frac{r^4 d^2}{\gamma^2}\polylog(d/\eps)\\
    &\lesssim  \frac{r^2 d^2\max(r^2\eps,1)}{\gamma^5 }  \polylog(d/\eps) + \frac{1}{\eps \gamma}\polylog(d/\eps) + \frac{r^4 d^2}{\gamma^3}\polylog(d/\eps) \;.
\end{align*}

\textbf{Memory Usage}: The estimator of  Line \ref{eq:L} in \Cref{alg:streaming} uses memory $O((1/\eps)\polylog(d/\eps))$, and that memory can be freed and reused the next time Line \ref{eq:L} is executed. The filters created by \textsc{HardThresholdingFilter} are of the form of a vector and a one dimensional threshold, meaning that they can be stored in $O(d)$ memory. The number of the filters created in each one of the $t_\mathrm{end}\cdot k_\mathrm{end}$ calls of \textsc{HardThresholdingFilter} is at most $O(\log(d/\eps))$. Thus, the memory used to store all of them is $O(t_\mathrm{end}\cdot k_\mathrm{end} \cdot d \log(d/\eps)) = O((d/\gamma) \polylog(d/\eps))$.
Every other operation done in \Cref{alg:streaming} is either a multiplication of an empirical second moment matrix with a vector (that as we have noted earlier can be implemented in $O(d)$ memory) or even more basic operation.

\textbf{Runtime}: The same runtime analysis from \Cref{sec:runtime} is applicable here as well.

\subsubsection{Bit Complexity}\label{sec:bit-complexity}

We briefly discuss how our algorithm can be implemented with bounded precision (i.e., in the standard word RAM model) using $(d/\gamma)\polylog(d/\eps)$ registers of size $(1/\gamma)\polylog(d/\eps)$ bits each. We assume that all but $\eps$ mass of the inlier distribution is supported in a ball of radius $R=(d/\eps)^{\polylog(d/\eps)}$ and we also assume that $\s \succeq (\eps/d)^{\polylog(d/\eps)} \vec I$. Our algorithm ignores all points $x$ outside of the ball and rounds the remaining ones to $x'$ such that $\|x-x'\|_2 \leq \eta$. We want $\eta$ to be small enough so that it does not affect the correctness of our algorithm. Picking $\eta =(\eps/d)^{\polylog(d/\eps)}$ means that the rounded distribution has covariance matrix close enough to the original one so that is satisfies the same stability condition (modulo a difference in the constants), thus our algorithm will be correct.  So far, by the aforementioned rounding, every point can be stored in $d$ words of size $O(\log(Rd/\eta)) = \polylog(d/\eps)$ bits). Regarding the bit complexity of intermediate calculations of our algorithm, the most expensive one is the computation of $\vec M z$ for $\vec M$ being the empirical covariance raised to the power $p=\gamma^{-1}\polylog(d/\eps)$ and $z$ a random Gaussian vector. Since the power $p$ includes a $1/\gamma$ factor, the final result may have bit complexity increased by $1/\gamma$ factor, thus we need $d$ registers of size $(1/\gamma)\polylog(d/\eps)$ to store the resulting vector $\vec M z$. The filters that we create are $(1/\gamma)\polylog(d/\eps)$ in number and the representation of each is a vector of the previous form (and a one-dimensional threshold). Thus, overall, we have to use $(d/\gamma)\polylog(d/\eps)$ many 
registers of size $(1/\gamma)\polylog(d/\eps)$ bits each.

\vspace{-0.5em}
\section{Conclusion}

We gave the first nearly-linear time algorithm for robust PCA that
attains near-optimal error guarantees, 
without eigenvalue gap assumptions in the
spectrum of $\Sigma$.
Additionally, we presented the first sub-quadratic space streaming algorithm
for robust PCA.

In terms of future improvements, one potential avenue for optimization is the
calculation of $\vec B^p$ for a PSD matrix $\vec B$ and $p =
	\Tilde{\Theta}(1/\eps)$.
It is likely that these matrix-dot products can be approximated faster through
the use of Chebyshev approximation~\cite[Chapter 10]{SacVis14}.
Additionally, it remains to be determined if the runtime of our algorithm
improves in the presence of a gap among the large eigenvalues of the true
covariance matrix.
Lastly, as highlighted by \cite{DKPP22}, designing streaming algorithms with
optimal sample complexities and robustness to strong contamination is an open
problem.

\newpage
\bibliographystyle{alpha}
\bibliography{allrefs,newrefs}

\newpage
\appendix

\section*{Appendix}

\paragraph{Organization} The Appendix is organized as follows:
\Cref{sec:prelim-appendix} includes the additional standard results (and their proofs) that were
omitted from the main paper.
Finally, \Cref{sec:streaming} focuses on the streaming algorithm achieving \Cref{thm:streaming-main}.

\section{Preliminaries: Omitted Facts}
\label{sec:prelim-appendix}

\subsection{Information-theoretic Error}
\label{sec:info-theoretic-error}
Let us briefly outline why the best approximate guarantee for subgaussian distributions is of the order $(1 - \tilde{\Theta}(\eps))$.
We will do so by focusing on Gaussian distributions and establishing a lower bound of $(1 - \Omega(\eps))$.
Let $P$ be the standard Gaussian distribution in $\R^d$.
For a unit vector $v$ and $\eps \in (0, 0.1)$, let $Q$ be the distribution $\cN(0, \vec I + \eps vv^\top)$.
Using simple calculations given below, it can be seen that $\dtv(P ,Q) \leq \eps/2$:
By Pinsker's inequality and the expression for KL divergence between two Gaussians, we obtain the following series of inequalities
\begin{align*}
    \dtv(Q,P) &\leq \sqrt{d_{\text{KL}} (Q||P)} \tag{Using Pinsker's inequality }\\
    &= \sqrt{ \frac{1}{2} \left( -\log|\vec \Sigma_Q| + \tr(\vec \Sigma_Q) - d \right)} \tag{KL divergence between two Gaussians}\\
    &= \sqrt{ \frac{1}{2} \left(- \log(1 + \eps) + \eps \right)}\tag{Using $\s_Q = \vec I + \eps vv^\top$}\\
    &\leq \eps/2 \tag{since $\log(1 + x) \geq x - x^2/2$ for $x\geq 0$}.
\end{align*}
Let $\widehat{v} \in \R^d$ be the output of any robust PCA algorithm, when given i.i.d.\ samples from $P$ as input.
The true variance along the direction $\widehat{v}$ under the  distribution $Q$ is $1 + \eps (\widehat{v}^\top v)^2$.
In particular, the approximation factor (on the distribution $Q$) obtained by the algorithm outputting $\widehat{v}$ is
\begin{align*}
\frac{\widehat{v}^\top \s_Q \widehat{v}}{\|\s_Q\|_\op} =    \frac{1 + \eps (\widehat{v}^\top v)^2}{1 + \eps} \leq (1 -0.5 \eps)(1 + \eps (\widehat{v}^\top v)^2).
\end{align*}
Since $P$ and $Q$ have total variation distance at most $\eps/2$,
an adversary, that is allowed to  corrupt an $\eps$-fraction of samples, can simulate i.i.d.\ samples of $P$ given i.i.d.\ samples from $Q$, and vice versa.
Now, suppose the true distribution was $Q$, but the adversary gives i.i.d.\ samples from $P$ (which is a valid adversary as mentioned above).
Since $P$ contains no information about $v$ whatsoever, no algorithm can identify $v$ from the set of  $\eps$-corrupted samples of $Q$. 
As $v$ could be an arbitrary unit vector, no algorithm can output a $\hat{v}$ such that  
$|v^\top \widehat{v}| \leq 0.01$ for $d $ larger than a constant.
As a result, it is not possible  to obtain an approximation better than $(1 - 0.1\eps)$ for Gaussians because $(1-0.5 \eps)(1 + 0.1 \eps) \leq (1 - 0.1 \eps)$.

\subsection{Comparison Between the Two Stopping Conditions}
\label{sec:comparison-two-stopping-conditions}
Consider the setting in \Cref{sec:advanced-cert-lemma}.
Let $w:\R^d \to \{0,1\}$ be weights such that probability of inliers is at least $1- 3 \eps$.
Let $r= \E_{X \sim P}[w_t(X)]$.
Recall that $\vec M = \left(\E_{X \sim P}[w_t(X)XX^\top]\right) = \left(\E_{X \sim P}[w_t(X)] \s_t\right)^p = r^p \s_t^p$.

Our algorithm stops when $\langle \s_t, \vec M^2  \rangle \leq (1 + C \gamma) \langle  \s , \vec M^2\rangle$.
We will now show that this stopping condition implies the stopping condition that depends solely on the Schatten norms for large $p$.
Using the definition of the Schatten norm, we begin as follows: 
\begin{align*}
 \|\s_t\|_{2p+1}^{2p+1} &= \langle \s_t, \s_t^{2p} \rangle = r^{-2p} \langle \s_t,  \vec M^2 \rangle \\
 &\leq (1 + C \gamma) r^{-2p}\langle \s, \vec M^2 \rangle \tag{Using stopping condition}\\
 &\leq (1 + C \gamma)r^{-2p} \|\s\|_\op \|\vec M\|_\fr^2 \\
 &= (1 + C \gamma) \|\s\|_\op \|\s_t\|_{2p}^{2p} \tag{Using definition of $\s_t$} \\
&\leq (1 + C \gamma) \|\s\|_\op \left( \|\s_t\|_{2p+1} d^{\frac{1}{2p(2p+1)}}\right)^{2p}  \tag{Using \Cref{fact:normReln}}\\
&\leq (1 + 2C \gamma) \|\s\|_\op \|\s_t\|_{2p+1}^{2p}, 
\end{align*}
where we use that $p \gg \log(d)/\gamma$ and thus $d^{1/{(2p+1)}} \leq (1 + C \gamma)$.
Rearranging it,
we obtain that
\begin{align*}
    \|\s_t\|_{2p+1} \leq (1 + 2C \gamma) \|\s\|_\op,
\end{align*}
and thus $\|\s_t\|_{2p+1}^{2p+1} \leq (1 + 2C \gamma)^{2p+1} \|\s\|_\op^{2p+1}$, which is the corresponding stopping condition that depends solely on the Schatten norms of $\s_t$ and $\s$.
Thus, the stopping condition of \Cref{lem:certificate1-main-body} is stronger.

\subsection{Omitted Proofs Regarding Stability: Proof of \Cref{lem:stability-implications}}\label{app:stability}

In this section we restate and prove \Cref{lem:stability-implications}. We have not optimized the constants (up to the choice of constants $C$).

\StabImplications*

We prove each part of the lemma separately. For the first one, we have the following, which holds under the slightly stronger condition $\E_{X \sim G}[w(X)] \geq 1-7\eps$ (this is because this version will be needed later on).

\begin{restatable}{lemma}{CorShift}
  \label{lem:stability_multidim}
  Let $0<20\eps\leq \gamma < 1$.
  Let $G$ be a $(20\eps,\gamma)$-stable distribution with respect to $\s$.
  Let a matrix $\bU \in \R^{m \times d}$ and a function $w:\R^d \to [0,1]$ with
  $\E_{X \sim G}[w(X)] \geq 1-7\eps$.
  For the function $g(x) = \|\vec U x\|_2^2$, we have that
  \begin{align*}
   (1-1.35\gamma)  \left\langle\vec U^\top \vec U, \s \right\rangle \leq  \E_{X \sim G}[w(X)g(X)] \leq \left(1+\gamma\right) \left\langle\vec U^\top \vec U, \s  \right\rangle \;.
\end{align*}
\end{restatable}

\begin{proof}
  For the upper bound,
  \begin{align*}
     \E_{X \sim G}[w(X)g(X)] =   \left\langle \vec U^\top \vec U,  \hspace{-3pt}\E_{X \sim G}[w(X) XX^\top] \right\rangle 
     \leq (1+\gamma) \hspace{-3pt}\E_{X \sim G}[w(X)]\left \langle \vec U^\top \vec U, \s  \right\rangle 
     \leq (1+\gamma)\left \langle \vec U^\top \vec U, \s  \right\rangle,
\end{align*}
  where the first step rewrites it in the notation of trace inner product, the second step uses the stability
  condition and \Cref{fact:trace_PSD_ineq} with $\vec A = \vec U^\top \vec U$,
  $\vec B = \E_{X \sim G}[w(X) XX^\top]$, $\vec C = \E_{X \sim
      G}[w(X)](1+\gamma)\s$, and the last step uses $w(x)\leq 1$.
  The other direction is similar, with the only difference being that we lower
  bound $\E_{X \sim G}[w(X)] \geq 1-7\eps$:
  \begin{align*}
    \E_{X \sim G}[w(X)g(X)] &\geq (1-\gamma) (1-7\eps)\left \langle \vec U^\top \vec U, \s  \right\rangle \geq (1-\gamma - 7\eps)\left \langle \vec U^\top \vec U, \s  \right\rangle
     \geq (1-1.35\gamma)\left \langle \vec U^\top \vec U, \s  \right\rangle \;,
\end{align*}
  where we also used that $\eps\leq \gamma/20$.
\end{proof}

We will need the following intermediate result: 
\begin{lemma}
  \label{lem:stability_thresholded}
  Let $0<20\eps\leq \gamma < 1$.
  Let $G$ be a $(20\eps,\gamma)$-stable distribution with respect to $\s$.
  Let a matrix $\bU \in \R^{m \times d}$, a function $w:\R^d \to [0,1]$ with
  $\E_{X \sim G}[w(X)] \geq 1-3\eps$, and a set $S$ such that $\E_{X \sim
      G}[w(X)\1\{X \in S \}] \leq 4\eps$.
  Then, for the function $g(x):=\| \vec U x\|_2^2$, we have the following:
  \begin{align*}
    \E_{X \sim G}[w(X)g(X)\1\{X \in S \}] \leq   2.35\gamma \left\langle \bU^\top \bU, \s  \right\rangle \;.
\end{align*}
\end{lemma}
\begin{proof}
  Let $w'(x) := w(x) \1\{X \not\in S \} = w(x) - w(x)\1\{X \in S \}$, and note that $\E_{X \sim G}[w'(X)] =
    \E_{X \sim G}[w(X)] -\E_{X \sim G}[w(X)\1\{X \in S \}] \geq 1-7\eps$ by our
  assumptions.
  Thus, we have that
  \begin{align*}
    \E_{X \sim G}[w (X)g (X)\1\{X \in S  \}] 
    &= \E_{X \sim G}[w (X)g (X)] - \E_{X \sim G}[w '(X)g (X)] \\
    &\leq \left(1+\gamma\right)\left\langle \bU^\top \bU,\s  \right\rangle  - \left(1-1.35\gamma\right)\left\langle \bU^\top \bU,\s  \right\rangle  \\
    &= 2.35\gamma \left\langle \bU^\top \bU,\s  \right\rangle \;,
\end{align*}
  where the second line is the application of \Cref{lem:stability_multidim} for
  $w'(x)$.
\end{proof}

Now, \Cref{it:goodscores} of \Cref{lem:stability-implications} follows directly as shown below.

\begin{corollary}
  \label{cor:goodscores}
  In the setting of \Cref{lem:stability_thresholded}, let a mixture distribution
  $P=(1-\eps)G+\eps B$ and $S = \{x \, : \, x > L \}$, where $L$ is defined to be
  the top $3\eps$-quantile of $g(x)$ under $P_w$ (the weighted by $w$ version of
  $P$).
  Then,
  \begin{align*}
    \E_{X \sim G}[w(X)g(X)\1\{X \in S \}] \leq 2.35\gamma \left\langle \bU^\top \bU, \s  \right\rangle \;.
\end{align*}
\end{corollary}
\begin{proof}
  In order to apply \Cref{lem:stability_thresholded}, we only have to verify that
  $\E_{X \sim G}[w(X)\1\{X \in S \}] \leq 4\eps$.
  Since $P=(1-\eps)G+\eps B$ and $L$ is by definition the $3\eps$-quantile of
  $g(x)$ under $P_w$, and $\eps<1/10$, we have that
  \begin{align}
    \E_{X \sim G}\left[w (X)\1\{g (X) > L  \}\right]
    &\leq \frac{1}{1-\eps}\E_{X \sim P}\left[w (X)\1\{g (X) > L  \}\right] \notag \\
    &\leq  \frac{1}{1-\eps}\pr_{X \sim P_w}\left[ g (X) > L   \right] \notag \\
    &\leq \frac{3\eps}{1-\eps} <  4\eps \;. \label{eq:weight-bound}
\end{align}
  An application of \Cref{lem:stability_thresholded} completes the proof.
\end{proof}

We show \Cref{it:bound-quantile} of \Cref{lem:stability-implications} below.

\begin{lemma}
  \label{lem:bound-quantile}
  Let $0<20\eps\leq \gamma < 1$.
  Let a mixture distribution  $P=(1-\eps)G + \eps B$ on $\R^d$, where $G$ is $(20\eps,\gamma)$-stable distribution with respect to a matrix $\s$.
  Let   $\bU \in \R^{m \times d}$ and a function $w:\R^d \to [0,1]$ with $\E_{X \sim
      G}[w(X)] \geq 1-3\eps$.
  Define $g(x) := \|\vec U x\|_2^2$ and let $L$ be the top $3\eps$-quantile of
  $g(x)$ under $P_w$.
  Then, $L\leq 1.65(\gamma/\eps) \langle \bU^\top \bU, \s \rangle $.
\end{lemma}
\begin{proof}
  Let $w'(x) := w(x) \Ind(g(x) > L )$.
  Note that
  \begin{align*}
    \E_{X \sim G}[w'(X)]&=\frac{\E_{X \sim P}[w(X)\Ind(g(X) > L )] - \eps  \E_{X \sim B}[w(X)\Ind(g(X) > L )]}{1-\eps} \\
    &\geq \frac{\E_{X \sim P}[w(X)\Ind(g(X) > L )] -\eps  }{1-\eps} \\
    &= \frac{\E_{X \sim P}[w(X) ]\pr_{X \sim P_w}[ g(X) > L ] -\eps  }{1-\eps}\\
    &\geq \frac{(1-\eps)(1-3\eps)3\eps -\eps  }{1-\eps}
    \geq \frac{1.43\eps}{1-\eps} \;,
\end{align*}
  where the second step uses that $w(x)\leq 1$, the fourth step uses that $\E_{X \sim
      P}[w(X) ]\geq (1-\eps) \E_{X \sim G}[w(X) ]\geq (1-\eps)(1-3\eps)$ and that $L$
  is the $3\eps$-quantile of $g$ under $P_w$, and the last step used that
  $\eps<1/20$.
  Using this lower bound on $\E_{X \sim G}[w'(X)]$ with \Cref{cor:goodscores}, we have that
  \begin{align*}
    \E_{X \sim G_{w'}}[g(X)] = \frac{\E_{X \sim G}[w'(X) g(X)] }{\E_{X \sim G}[w'(X)  ]}  
    \leq \frac{1-\eps}{1.43 \eps} \E_{X \sim G}[w(X) g(X) \Ind(g(X) > L )]
    \leq  \frac{1.65\gamma}{\eps} \left\langle \bU^\top \bU, \s  \right\rangle\;.
\end{align*}
  Since the minimum is less than the average, the result above implies that there exists a point in the support of $G_{w'}$ which is
  smaller than $1.65(\gamma/\eps) \langle \bU^\top \bU, \s \rangle $.
  Since $G_{w'}$ is supported only on points bigger than $L$, it means that $L \leq
    1.65(\gamma/\eps) \langle \bU^\top \bU, \s \rangle $.
\end{proof}

We now provide the proof of \Cref{it:proj-var} of \Cref{lem:stability-implications} below.

\begin{lemma}[Variance estimator]
  \label{lem:proj-var}
  Let $0<20\eps\leq \gamma < 1$.
  Let a mixture distribution  $P=(1-\eps)G + \eps B$ on $\R^d$, where $G$ is $(20\eps,\gamma)$-stable distribution with respect to a matrix $\s$.
  Let  $\bU \in \R^{m \times d}$ and a function $w:\R^d \to [0,1]$ with $\E_{X \sim
      G}[w(X)] \geq 1-3\eps$.
  Define $g(x) := \|\vec U x\|_2^2$ and let $L$ be the top $3\eps$-quantile of
  $g(x)$ under $P_w$.
  For the estimator
  $\widehat{\sigma} := \E_{X \sim P}[w(X) g(X) \Ind(g(X) \leq L )]$, we have that
  \begin{align*}
   \left|\widehat{\sigma}-  \left\langle \bU^\top \bU, \s  \right\rangle  \right| \leq 4\gamma \left\langle \bU^\top \bU, \s  \right\rangle \;. 
\end{align*}
\end{lemma}
\begin{proof}
  We consider the contribution to $\widehat{\sigma}$ from the two components of the distribution $P$
  separately.
  For the component $G$, 
  we first note that $\E_{X \sim G}[w(X)\Ind(g(X) \leq L
      )]\geq \E_{X \sim G}[w(X)] - \E_{X \sim G}[w(X)\Ind(g(X) > L )] \geq 1- 3\eps -
    4\eps =1-7\eps$ (the last step can be shown as in \Cref{eq:weight-bound}).
  Thus, by \Cref{lem:stability_multidim}, we have that
  \begin{align}\label{eq:bound1}
    \left|  \E_{X \sim G}[w(X)g(X)\Ind(g(X) \leq L )] - \left\langle \bU^\top \bU, \s  \right\rangle  \right| \leq 1.35\gamma \left\langle \bU^\top \bU, \s  \right\rangle \;.
\end{align}
  We now consider the contribution to $\widehat{\sigma}$ from $B$.
  We first note that $L
    \leq 1.65(\gamma/\eps) \langle \bU^\top \bU, \s \rangle $ by
  \Cref{lem:bound-quantile}.
  Therefore,
  \begin{align}\label{eq:bound2}
    \left|   \E_{X \sim B}[w(X)g(X)\Ind(g(X) \leq L )] - \left\langle \bU^\top \bU, \s  \right\rangle   \right| 
    \leq   L +  \left\langle \bU^\top \bU, \s  \right\rangle 
    \leq \left( \frac{1.65\gamma}{\eps} + 1\right) \left\langle \bU^\top \bU, \s  \right\rangle \;.
\end{align}
  Finally, combining \Cref{eq:bound1,eq:bound2} and using that $\eps \leq \gamma$
  yields:
  \begin{align*}
     \left|  \E_{X \sim P}[w(X)g(X)\Ind(g(X) \leq L )] -\left\langle \bU^\top \bU, \s  \right\rangle   \right| 
     &\leq (1-\eps) \left|   \E_{X \sim G}[w(X)g(X)\Ind(g(X) \leq L )] - \left\langle \bU^\top \bU, \s  \right\rangle \right| \\
     &\quad + \eps \left|   \E_{X \sim B}[w(X)g(X)\Ind(g(X) \leq L )] - \left\langle \bU^\top \bU, \s  \right\rangle  \right| \\
     &\leq  4\gamma \left\langle \bU^\top \bU, \s  \right\rangle    \;.
\end{align*}
\end{proof}
We conclude with the last part of \Cref{lem:stability-implications}.
\begin{claim}
    \label{cl:initialization}
    Let $0<20\eps\leq \gamma < 1$.
  Let   $G$ be a $(20\eps,\gamma)$-stable distribution with respect to a matrix $\s$. Then, 
    $\pr_{X \sim G}\left[ \|X\|_2^2 > 2 (d/\eps)\|\s\|_\op \right] \leq \eps$.
  \end{claim}
  \begin{proof}
    Using Markov's inequality and \Cref{lem:stability_multidim} with $\vec U = \bI$
    \begin{align*}
    \pr_{X \sim G}\left[ \|X\|_2^2 > 2 (d/\eps)\|\s\|_\op  \right] \leq \frac{\E_{X \sim G}[\|X\|_2^2]}{2 (d/\eps)\|\s\|_\op} \leq \frac{(1+\gamma)\tr(\s)}{2 (d/\eps)\|\s\|_\op} \leq \eps \;.
\end{align*}
  \end{proof}

We finally restate and prove \Cref{lem:basic-cert} from \Cref{sec:stability-main-body}. 
\begin{lemma}
  \label{lem:certificate}
  Let $0<20\eps\leq \gamma < \gamma_0$ for a sufficiently small absolute constant $\gamma_0$.
  Let $P=(1-\eps)G+ \eps B$, where $G$ is a $(20\eps,\gamma)$-stable distribution
  with respect to $\s$.
  Let $w:\R^d \to [0,1]$ with
  $\E_{X \sim G}[w(X)] \geq 1-3\eps$, and recall that in our notation $\os_{P_w}:= \E_{X \sim P_w}[XX^\top] = \E_{X \sim P}[w(X)XX^\top]/\E_{X \sim P}[w(X)]$.
  If $u$ is a vector such that $u^\top \os_{P_w} u /\|u\|_2^2 \geq (1-\gamma)\|\os_{P_w}\|_\op$ and $u^\top \s u \geq (1-O(\gamma)) u^\top \os_{P_w} u$, then  $u^\top \s u/\|u\|_2^2
    \geq (1-O(\gamma)) \| \s \|_\op$.%
\end{lemma}
\begin{proof}
Since $\E_{X \sim G}[w(X)] \geq 1-3\eps$, we can use
  the stability condition as follows:
  Let $a$ be the (normalized) top eigenvector of $\os_{P_w}$ and $b$ be the (normalized) top eigenvector of
  $\vec \s$.
Using the definition of the operator norm and stability, 
we first obtain the following lower bound on $\|\s_{P_w}\|$ in terms of $\|\s\|_\op$:
  \begin{align}
    \left\| \os_{P_w} \right\|_\op &= a^\top \os_{P_w} a \geq b^\top \os_{P_w}  b 
    =  b^\top \left(  \frac{\E_{X \sim P}[w(X)XX^\top]}{\E_{X \sim P}[w(X)]} \right)  b \\
    &\geq (1-\eps) b^\top \left(  \frac{\E_{X \sim G}[w(X)XX^\top]}{\E_{X \sim P}[w(X)]} \right)  b 
    = (1-\eps)b^\top \left(  \frac{\E_{X \sim G}[w(X)]}{\E_{X \sim P}[w(X)]} \os_{G_w} \right)  b \\
    &\geq (1-\eps)(1-3\eps) b^\top \os_{G_w} b \geq (1 - 4 \eps)
    (1-\gamma) b^\top \vec \Sigma b \\
    &\geq \left(1 - 2\gamma  \right) b^\top \vec \Sigma b = (1-2\gamma)\|  \s \|_\op \;, \label{eq:helping-ineq}
\end{align}
  where the penultimate step  used that $\eps<\gamma/20$.
  Combining this with the assumption that $u^\top \s u$ is large compared to $u^\top \s_{P_w} u$, 
  we obtain the following:
  \begin{align*}
    \frac{u^\top \s  u}{\|u\|_2^2} &\geq
    (1- O(\gamma )) \frac{u^\top \os_{P_w} u}{\|u\|_2^2} 
    \geq (1- O(\gamma ))(1-\gamma) \|\os_{P_w}\|_\op \notag \\
    &\geq  (1- O(\gamma ))(1-\gamma)(1-2\gamma) \| \s \|_\op
    = (1-O(\gamma))  \| \s \|_\op  \;.
\end{align*}
  
\end{proof}

\subsection{Filtering}
\label{sec:filtering}

In this subsection, we prove the filter guarantees that were described in \Cref{sec:filtering-main-body}. 
We use a slightly more general version of the filter so that it can be employed later on when we will develop a streaming algorithm. 
The difference is the introduction of an additional error parameter $\delta$ in the filter.  This will be useful in \Cref{sec:streaming}, where the quantity $\E_{X \sim P}[w(x) \tau(x)]$  as well as $\widehat{T}$ will have to be estimated by averaging samples, thus $\delta$ will then account for a small additive error in that approximation. 
For the non-streaming algorithm, we will use $\delta=0$ 

\begin{algorithm}
  \caption{\textsc{HardThresholdingFilter}}
  \label{alg:filter-appendix}
  \begin{algorithmic}[1]
    \State \textbf{Input}:
    Distribution $P=(1-\eps)G+\eps B$, weights $w$, scores $\tau$, parameters $\widehat{T},R,\delta$.
    \State $w_{0}(x) \gets w(x)$, and $r_0 \gets R$, $\ell \gets 1$.
    \State \textbf{while }{$\E_{X \sim P}[w(X)\tau(X)] > \frac{5}{2}(\widehat{T}+\delta)$}\label{line:stopping-condd}
    
    \State $\,\,\,\,\,$ Draw $r_\ell \sim \cU([0,r_{\ell-1}])$.
    \label{line:random-thr}
    \State $\,\,\,\,\,$ $w_{\ell+1}(x) \gets w_{\ell}(x) \cdot \Ind( \tau(x) > r_\ell)$.
    \label{line:new-filter}
    \State $\,\,\,\,\,$ $\ell \gets \ell + 1$.
    \State \textbf{return} $w_{\ell}(x)$.
  \end{algorithmic}
\end{algorithm}

\begin{lemma}[Guarantees of Filtering]
  \label{lem:filter_guar}
  Consider one call of
  \textsc{HardThresholdingFilter}($P,w,\tau,\widehat{T},R,\delta$).
  Suppose there exists $T$ such that $(1-\eps)\E_{X \sim
      G}\left[w(X)\tau(X)\right] < T$ and $\widehat{T}$ such that $|\widehat{T}-T| <
    T/5 + \delta$.
  Denote by $F$ the randomness of the filter (i.e., the collection of the random
  thresholds $r_1,r_2\ldots$ used).
  Let $w'$ be the weight function returned.
  Then,
  \begin{enumerate}
    \item $\E_{X \sim P}[w'(X)\tau(X)] \leq 3 T + 5d$ with probability 1 over the randomness of the filter $\cF$.
    \item $\E_{F}[  \eps \E_{X \sim B}[w(X) - w'(X)] ]>  (1-\eps) \E_{X \sim G}[w(X) - w'(X) ]  ]$.
  \end{enumerate}
\end{lemma}

\begin{proof}
  The first part of the lemma follows by the termination condition $\E_{X \sim
      P}[w'(X)\tau(X)] \leq (5/2)(\widehat{T} + \delta)$ and the assumption $|\widehat{T}-T| <
    T/5 + \delta$.

  We show that the second desideratum holds for every filtering round,
  conditioned on the past rounds.
  Consider the $\ell$-th round of filtering.
  The probability that a point gets removed (over the random selection of
  $r_\ell$) is $\tau(x)/r_{\ell-1}$.
 As a result, for any point $x$ that has not been filtered, the probability that a point that gets removed is equal to $\E[ (w_\ell(x) - w_{\ell +1 } (x) )] =\tau(x)/r_{\ell -1} $
  Also, note that since the weights are binary decreasing functions, $w_{\ell}(x) -
    w_{\ell+1}(x) = w_{\ell}(x)(w_{\ell}(x) - w_{\ell+1}(x))$.
  For the inliers, we have that
  \begin{align*}
    \E_{r_\ell}\left[ (1-\eps) \E_{X \sim G}\left[w_{\ell}(X) - w_{\ell+1}(X) \right]  \right] &=  
    \E_{r_\ell}\left[ (1-\eps) \E_{X \sim G}\left[w_{\ell}(X)(w_{\ell}(X) - w_{\ell+1}(X)) \right]  \right] \\
    &= (1-\eps) \E_{X \sim G}\left[ w_{\ell}(X)  \E_{r_\ell}\left[w_{\ell}(X) - w_{\ell+1}(X)\right]  \right]\\
    &= \frac{1-\eps}{r_{\ell-1}} \E_{X \sim G}\left[w_{\ell}(X)\tau(X)\right] 
    <\frac{T}{r_{\ell-1}} \;.
\end{align*}
  We now turn to the outliers.
  If the filtering hasn't stopped, it means that the stopping condition of the
  while loop is still violated and thus $\E_{X \sim P}\left[w_{\ell}(X)\tau(X)
      \right]> (5/2)(\widehat{T}+\delta)>2T$, (since we have assumed $|\widehat{T}-T| <
    T/5 + \delta$).
  Thus,
  \begin{align*}
    \E_{r_\ell}\left[ \eps \E_{X \sim B}\left[w_{\ell}(X) - w_{\ell+1}(X) \right]  \right] 
    &=\E_{r_\ell}\left[ \eps \E_{X \sim B}\left[w_{\ell}(X)(w_{\ell}(X) - w_{\ell+1}(X)) \right]  \right] \\
    &=   \E_{X \sim P}\left[w_{\ell}(X)\E_{r_\ell}\left[w_{\ell}(X) - w_{\ell+1}(X)\right] \right]  \\
    &\quad \quad- (1-\eps) \E_{X \sim G}\left[w_{\ell}(X)\E_{r_\ell}\left[w_{\ell}(X) - w_{\ell+1}(X)\right] \right]  \\
    &= \frac{1}{r_{\ell-1}}\E_{X \sim P}\left[w_{\ell}(X)\tau(X) \right] - \frac{1-\eps}{r_{\ell-1}}\E_{X \sim G}\left[w_{\ell}(X)\tau(X)\right] \\
    &> \frac{2T}{r_{\ell-1}} - \frac{T}{r_{\ell-1}}= \frac{T}{r_{\ell-1}} \;.
\end{align*}

\end{proof}

We now provide the guarantees of the filtering algorithm over the course of the entire execution of \textsc{RobustPCA}.
\begin{restatable}{lemma}{MARTINGALE}
  \label{lem:martingale}
  Let $0<20\eps\leq \gamma < 1/20$.
  Let $P=(1-\eps)G+ \eps B$, where $G$ is a $(20\eps,\gamma)$-stable distribution
  with respect to $\s$.
  Consider an execution of \textsc{RobustPCA}.
  With probability at least $0.2$ over the randomness of the algorithm, we have
  that $(1-\eps) \E_{X \sim G}[1-w_{k,t}(X)] + \eps \E_{X \sim B}[w_{k,t}(X)]
    \leq 2.9 \eps$.
\end{restatable}

\begin{proof}
  For notational simplicity, first define an ordering over the multi-indices $\{
    (k,t) : k \in \{1,\ldots,k_\mathrm{end}\}, t \in \{1,\ldots,t_\mathrm{end} \}
    \}$ in the natural way: $(k',t')\leq(k,t)$ is defined to mean that either $k'
    \leq k$ or $k'=k$ and $t' \leq t$.
  Also define a successor operator $\suc(\cdot)$ that returns the next
  multi-index according to the way the two loops are structured in
  \textsc{RobustPCA}, that is, $\suc(k,t) = (k,t+1)$ if $t<t_\mathrm{end}$ and
  $\suc(k,t) = (k+1,0)$ otherwise.

  Let $(k^*,t^*)$ be the first time that $(1-\eps) \E_{X \sim G}[1-w_{k,t}(X)] +
    \eps \E_{X \sim B}[w_{k,t}(X)] \geq 2.9 \eps$ (if there exists one, otherwise set
  $(k^*,t^*)=(k_\mathrm{end},t_\mathrm{end})$).
  Define the sequence
  \begin{align*}
\Delta_{(k,t)} = (1-\eps) \E_{X \sim G}\left[1-w_{\min\{(k,t),(k^*,t^*)   \}}(X)\right] + \eps \E_{X \sim B}\left[w_{\min\{(k,t),(k^*,t^*)   \}}(X) \right]    \;.
\end{align*}
  We note that $\Delta_{(k,t)}$ is a supermartingale with respect to the
  randomness of the filter: If $(k,t)<(k^*,t^*)$, then $\E[\Delta_{\suc(k,t)}
      \mid \cF_{(k,t)}] \leq \Delta_{(k,t)}$, where $\cF_{(k,t)}$ denotes the
  randomness of the filters applied so far.
  This follows by \Cref{lem:filter_guar}, as long as the lemma is applicable: We
  apply that lemma with $T= 2.35\gamma v_{k,t}^\top\s v_{k,t}$ and
  $\widehat{T}=2.35 \gamma \widehat{\sigma}_{k,t}$.
  Its first requirement that $(1-\eps)\E_{X \sim
      G}\left[w_{k,t}(x)\tau_{k,t}(x)\right] < T$ is satisfied by
  \Cref{it:goodscores} of \Cref{lem:stability-implications} (specifically, the special case that uses $\vec U =
  v_{k,t}^\top$).
  For its second requirement $|\widehat{T}-T|<T/5$ we have that $|\widehat{T}-T|=2.35
    \gamma |\widehat{\sigma}_{k,t}-v_{k,t}^\top \s v_{k,t}| \leq 2.35 \gamma \cdot
    4 \gamma\cdot v_{k,t}^\top \s v_{k,t} \leq (1/5)\cdot 2.35 \gamma
    v_{k,t}^\top\s v_{k,t} = T/5$, where the second step uses \Cref{it:proj-var} of \Cref{lem:stability-implications} 
  (with $\vec U = v_{k,t}^\top$) and the last inequality uses that $\gamma<1/20$.
  All these intermediate lemmata that we just refered to are applicable when $(1-\eps) \E_{X \sim
    G}[1-w_{k,t}(X)] + \eps \E_{X \sim B}[w_{k,t}(X)] \leq 2.9 \eps$.

  Otherwise, if $(k,t)\geq (k^*,t^*)$ we still have $\E[\Delta_{\suc(k,t)} \mid
      \cF_{(k,t)}] \leq \Delta_{(k,t)}$, since by definition the sequence
  $\Delta_{\suc(k,t)}$ remains unchanged.

  In the beginning, we have that $\Delta_{(1,1)}\leq 2\eps$.
  This is because we start with $\eps$-fraction of outliers and the na\"ive
  pruning of line \ref{line:naive-prune} of the algorithm cannot remove more than
  $\eps$ mass from the distribution $G$, as shown in \Cref{it:initialization} of \Cref{lem:stability-implications}.
  
  \noindent By Doob's inequality for martingales, we have that
  \begin{align}\label{eq:doob}
    \pr \left[  \max_{(1,1) \leq (k,t) \leq (k_\mathrm{end}, t_\mathrm{end})} \Delta_{(k,t)} > 2.5 \eps \right] \leq \frac{\Delta_{(1,1)}}{2.5 \eps} \leq 0.8 \;.
\end{align}
  This implies that with probability $0.2$, it holds $(1-\eps) \E_{X \sim
    G}[1-w_{k,t}(X)] + \eps \E_{X \sim B}[w_{k,t}(X)] \leq 2.5 \eps$ throughout the
  algorithm's iterations (which is already a bit stronger than the conclusion of
  the lemma that wanted to show).
  To see that, assume the contrary, and
  define $(\widetilde{k},\widetilde{t})$ to be the first time that $(1-\eps)
    \E_{X \sim G}[1-w_{k,t}(X)] + \eps \E_{X \sim B}[w_{k,t}(X)] > 2.5 \eps$.
  Combining that with the trivial observation that $(\widetilde{k},\widetilde{t})
    \leq (k^*,t^*)$ yields that $\Delta_{(\widetilde{k},\widetilde{t})} > 2.5
    \eps$, which by \Cref{eq:doob} cannot happen with probability more than 0.2.
\end{proof}

\section{Omitted Proofs from \Cref{sec:streaming-main}}
\label{sec:streaming}

We now provide details that were omitted from \Cref{sec:streaming-main}. In particular, 
we will show that the deterministic conditions in \Cref{cond:deterministic} hold with high probability.
Our focus in this section will be on \Cref{it:g-scores,it:fr-norm,it:scores-mean,it:quantile_est} of \Cref{cond:deterministic} since we have already shown  \Cref{it:strong-power-est,it:power-est} in \Cref{sec:estimators}. 
The proofs of  \Cref{it:g-scores,it:fr-norm,it:scores-mean} will closely follow \cite{DKPP22}, while \Cref{it:quantile_est} is a standard implication of Chernoff-Hoeffding bounds.

\subsection{\Cref{it:g-scores,it:fr-norm}}

\Cref{it:g-scores,it:fr-norm} of \Cref{cond:deterministic} rely on the fact that the empirical covariances used in the estimator $\prod_{\ell=1}^{p_k} \widehat{\vec B}_{k,t,\ell}$ are close enough to $\vec B_{k,t}$, as in \Cref{lem:op-norm-closeness} below. We first prove the lemma and then show how the two conditions follow.

\begin{lemma}\label{lem:op-norm-closeness}
  Let $\delta<1$. Consider \Cref{alg:matrix-power-est} and assume that $P_{k,t}$ is supported on an $\ell_2$-ball of radius $r\sqrt{d \|\s\|_\op}$.
  If $\widehat{W}_{k,t}$ and every $\widehat{\vec B}_{k,t,\ell}$ 
in \Cref{alg:matrix-power-est} are calculated using 
$\tilde{n} >  C\frac{r^2 d p^2}{\delta^2}\log\left(\frac{d}{\tau}  \right)$
samples, where $C$ is a sufficiently large constant,  then with probability at least $1-\failp$, we have that
\begin{align*}
 \left\|\widehat{\vec M}_{k,t}  - \bM_{k,t} \right\|_2 \leq  \delta  \left\| \bM_{k,t}\right\|_\op\;,
\end{align*}

\end{lemma}
\begin{proof}
  Let $\| \widehat{\vec B}_{k,t,\ell} - \vec B_{k,t} \|_\op \leq  \delta_1 \|  \vec B_{k,t} \|_\op $ and $|\widehat{W}_{k,t} - \E_{X \sim P}[w_{k,t}(X)] | \leq \delta_2 $, for $\delta_1,\delta_2$ to be determined. The accuracy of the estimators $\widehat{\vec B}_{k,t,\ell}$ and that of $\hatm_{k,t}$ are related as follows:
  
  \begin{lemma}[Lemma 4.14 in \cite{DKPP22}] \label{lem:matrix_prop} 
Let $\bA, \bB, \bB_1,\dots,\bB_p$ be symmetric $d\times d$ matrices and define $\bM = \bB^p, \bM_S = \prod_{i=1}^p \bB_i $. 
If $\|\bB_i - \bB\|_{2} \le \delta_1 \|\bB\|_{2}$, then $\|\bM_S - \bB^p\|_{\op} \leq p \delta_1 (1 + \delta_1)^{p} \|\bB\|_{\op}^p$. \label{prop:2}
\end{lemma}

 By \Cref{lem:matrix_prop}, it suffices to ensure $p \delta_1 e^{p \delta_1} < \delta$. For that, it suffices to use $\delta_1 := \delta/(3p)$. In the reminder we focus on ensuring  $\| \widehat{\vec B}_{k,t,\ell} - \vec B_{k,t} \|_\op \leq  \delta_1 \|  \vec B_{k,t} \|_\op$ for that choice of $\delta_1$.
  
  Let $  \| \widehat{\s}_{k,t,\ell} - \os_{k,t} \|_\op \leq  \delta_3 \|  \os_{k,t} \|_\op $ for $\delta_3$ to be specified. We have that
  \begin{align*}
       \left\| \widehat{\vec B}_{k,t,\ell} - \vec B_{k,t} \right\|_\op  
       &= \left\| \widehat{W} \widehat{\s}_{k,t,\ell} - \E_{X \sim P}[w_{k,t}(X)] \os_{k,t}    \right\|_\op \\
       &\leq \E_{X \sim P}[w_{k,t}(X)] \left\| \widehat{\s}_{k,t,\ell} - \os_{k,t} \right\|_\op + \left|\widehat{W}_{k,t} - \E_{X \sim P}[w_{k,t}(X)] \right| \cdot \left\|  \widehat{\s}_{k,t,\ell} \right\|_\op \\
       & \leq  \left\| \widehat{\s}_{k,t,\ell} - \os_{k,t} \right\|_\op + \delta_2 \left\|  \widehat{\s}_{k,t,\ell} \right\|_\op\\
       &\leq  (1+\delta_2) \left\| \widehat{\s}_{k,t,\ell} - \os_{k,t,\ell} \right\|_\op + \delta_2 \left\| \os_{k,t,\ell} \right\|_\op \\
       &\leq (1+\delta_2)\delta_3 \left\| \os_{k,t,\ell} \right\|_\op + \delta_2 \left\| \os_{k,t,\ell} \right\|_\op \\
       &\leq 2\delta_3 \left\| \os_{k,t,\ell} \right\|_\op + \delta_2 \left\| \os_{k,t,\ell} \right\|_\op \;.
  \end{align*}
  To make the right hand side above bounded by $\delta_1 \| \os_{k,t,\ell}\|_\op$, we choose  $\delta_3 = 0.1 \delta_1$ and $\delta_2 =: 0.1 \delta_1$. The sample complexity of achieving $|\widehat{W}_{k,t} - \E_{X \sim P}[w_{k,t}(X)] | \leq \delta_2 =\delta/(30 p)$ is $O(\frac{p^2}{\delta^2} \log(1/\tau))$ by Hoeffding inequality. The sample complexity for achieving $  \| \widehat{\s}_{k,t} - \os_{k,t} \|_\op \leq  \delta_3 \|  \os_{k,t} \|_\op $ can be obtained by \Cref{fact:ver_cov} as follows:
  \begin{align*}
      \widetilde{n} &\lesssim \frac{r^2 d  \|\s\|_\op}{\delta_3^2 \| \os_{k,t} \|_\op}\log\left(\frac{d}{\tau}  \right) 
      \lesssim  \frac{r^2 d  p^2 \|\s\|_\op}{\delta^2 \| \os_{k,t} \|_\op}\log\left(\frac{d}{\tau}  \right)  
      \lesssim \frac{r^2 d  p^2 }{\delta^2 }\log\left(\frac{d}{\tau}  \right)\;,  
  \end{align*}
  where the last inequality uses that because of our assumptions $\E_{X \sim P}[w_{k,t}(X)] \geq 1-O(\eps)$ and $\|\vec B_{k,t}\|_\op \geq 0.5 \| \s \|_\op$, we have that $\| \os_{k,t} \|_\op = \| \vec B_{k,t} \|_\op/\E_{X \sim P}[w_{k,t}(X)] \gtrsim \|\s \|_\op$.
  
\end{proof}

\begin{proof}[Proof of \Cref{it:g-scores,it:fr-norm} of \Cref{cond:deterministic}]
  The two conditions follow directly from  \Cref{lem:op-norm-closeness} using \new{$\delta = \frac{0.01}{\sqrt{d}}\min\left(\frac{\sqrt{\gamma/\eps}}{r  },1 \right)$}.
  For the first one: Since $\| \hatm_{k,t} - \vec M_{k,t} \|_\op \leq \frac{0.1}{r \sqrt{d}}\sqrt{\gamma/\eps} \| \vec M_{k,t}\|_\op \leq \frac{0.1}{r \sqrt{d}}\sqrt{\gamma/\eps} \| \vec M_{k,t}\|_\fr$ and we are interested only for points with $\|x\|_2 \leq r\sqrt{d \|\s\|_\op} $, we have that
  \begin{align*}
       g(x) &\leq \left\| \vec M_{k,t}  x \right\|_2^2
      \leq  2\left\| \hatm_{k,t}  x \right\|_2^2 +  2\left\| \left(\vec M_{k,t}-\hatm_{k,t}\right)  x \right\|_2^2 \\
      &\leq  2\left\| \hatm_{k,t}  x \right\|_2^2 + 2\|x\|_2^2 \left\| \vec M_{k,t}-\hatm_{k,t} \right\|_\op^2 
      \leq 2\widehat{g}(x) + 0.02(\gamma/\eps) \| \vec M_{k,t}\|_\fr^2\|\s \|_\op \;.
  \end{align*}
  Equivalently, $\widehat{g}(x) \geq 0.5g(x) - 0.01 (\gamma/\eps) \| \vec M_{k,t}\|_\fr^2\|\s \|_\op$.  The proof of \Cref{it:fr-norm} can be found in \cite{DKPP22}.
\end{proof}

\subsection{\Cref{it:scores-mean}}

We prove \Cref{it:scores-mean} in the lemma below, which is a simple adaptation of Lemma 4.16 from \cite{DKPP22}. For the sake of completeness we include a proof.

\Cref{it:scores-mean} is used for estimating $\E_{X \sim P}[w(x)\tau(x)]$ in Line \ref{line:stopping-condd} of \textsc{HardThresholdingFilter} as well as in lines \ref{line:avg-est} of \Cref{alg:sol-gen2} and \ref{line:sigma_prime} of \Cref{alg:streaming}. We briefly clarify how exactly it is used before giving the proof; we will focus on
Line \ref{line:avg-est} of \Cref{alg:sol-gen2} since the other ones can be checked similarly. For that, we apply \Cref{lem:averaging-est} with $w(x) = w_{k,t}(x) \Ind((u^\top x)^2 \leq Q)$. Then, the guarantee of the lemma is that (using the notation of Line \ref{line:avg-est}) $|\widehat{\sigma}_u' - \widehat{\sigma}_u | \leq 0.01 \gamma \widehat{\sigma}_u + \frac{0.01\gamma}{r^2 d} \|u\|_2^2 \| \os_{k,t} \|_\op \leq  0.01 \gamma \widehat{\sigma}_u + 0.01\gamma \|u\|_2^2 \| \s \|_\op$, where the last inequality follows from the fact that $\|\os_{k,t}\|_\op \leq r^2 d \|\s\|_\op$ as the support of $P_{k,t}$ is in a ball of radius $r \sqrt{d \|\s\|_\op}$ (c.f. Line \ref{line:weights-init} of \Cref{alg:streaming}).

\begin{lemma} \label{lem:averaging-est}
     In the setting of  \Cref{alg:streaming}, let $r$ be a radius such that $\pr_{X \sim G}[\|X\|_2 > r \sqrt{d\|\s\|_\op}]\leq \eps$ and assume that $\|\vec B_{k,t} \|_\op\geq 0.5 \|\s\|_\op$. For any $\tau$ and any weight function $w : \R^d \to [0,1]$ and any vector $u$, there is an estimator $\widehat{F}$ for the quantity $F_{k,t}:= \E_{X \sim P}[w(X) (u^\top X)^2]$ that 
     uses \new{$n=O\left(\frac{r^4 d^2}{\gamma^2}\log(1/\tau)\right)$} samples from $P$, runs in $O(nd)$ time, uses memory $O(\log(1/\tau))$, and, with probability $1-\tau$ it satisfies $ | \widehat{F} - F_{k,t} | \leq 0.01 \gamma F_{k,t} + \frac{0.01\gamma}{r^2 d} \| u \|_2^2 \|\os_{k,t}\|_\op$.
\end{lemma}
\begin{proof}
    We show the lemma for constant probability of success as then one can boost that to $1-\tau$ by repeating $\log(1/\tau)$ times and keeping the median. The estimator is just sample average. By Chebysev's inequality, with constant probability, we have that
    \begin{align*}
        \left| \frac{1}{n} \sum_{i=1}^n w(x_i)(u^\top x_i)^2 - \E_{X \sim P}[w(X) (u^\top X)^2] \right|  \lesssim \frac{\sqrt{\Var_{X \sim P}[w(X) (u^\top X)^2]}}{n} \;.
    \end{align*}
    We want to upper bound the RHS by $0.01 \gamma \E_{X \sim P}[w(X) (u^\top X)^2] + \frac{0.01\gamma}{r^2 d}\| u \|_2^2 \|\os_{k,t}\|_\op$. A sufficient number of samples for that is big enough multiple of the following:
    \begin{align*}
         \frac{ \Var_{X \sim P}[w(X) (u^\top X)^2]}{\left(\gamma \E_{X \sim P}[w(X) (u^\top X)^2] + \frac{\gamma}{r^2 d}\| u \|_2^2 \|\os_{k,t}\|_\op \right)^2} 
        &\leq \frac{ r^2 d\Var_{X \sim P}[w(X) (u^\top X)^2]}{\gamma^2 \E_{X \sim P}[w(X) (u^\top X)^2] \| u \|_2^2 \|\os_{k,t}\|_\op} \\
        &\leq \frac{ r^2 d \E_{X \sim P}[w^2(X) (u^\top X)^4]}{\gamma^2 \E_{X \sim P}[w(X) (u^\top X)^2] \| u \|_2^2 \|\os_{k,t}\|_\op} \\
        &\leq \frac{r^2 d \| u \|_2^2 r^2 d \| \s \|_\op \E_{X \sim P}[w(X) (u^\top X)^2]}{\gamma^2 \E_{X \sim P}[w(X) (u^\top X)^2] \| u \|_2^2 \|\os_{k,t}\|_\op} \\
        &=  \frac{   r^4 d^2 \| \s \|_\op  }{\gamma^2  \|\os_{k,t}\|_\op} \lesssim \frac{   r^4 d^2   }{\gamma^2 } \;.
    \end{align*}
    where the third line from the end used $w^2(x) \leq w(x)$, $(u^\top x)^2 \leq\| u \|_2^2 \|x\|_2^2$, and  $\|x\|_2^2 \leq r^2 d \| \s \|_\op$ (by the pruning of Lines \ref{line:radius-est} and \ref{line:weights-init} in \Cref{alg:streaming}) and the fourth line uses our assumption $\|\vec B_{k,t} \|_\op\geq 0.5 \|\s\|_\op$.
\end{proof}

\subsection{\Cref{it:quantile_est}}

\begin{lemma} \label{l:simple-partition} 
For any  distribution $P$ over $\mathbb{R}$ and any $\eps,\delta \in (0,1)$, there is an estimator $\widehat{L}$ that uses
$O\left( \frac{1}{\eps} \log\left( \frac{1}{\tau} \right)   \right)$ samples from $D$ and satisfies
\begin{equation*}
\abs[\Big]{  \pr_{X \sim P}[X > \widehat{L}] - \eps} \leq \frac{\eps}{100} \;,
\end{equation*}
with probability at least $1-\tau$. The memory usage and runtime of the estimator are also $O\left( \frac{1}{\eps} \log\left( \frac{1}{\tau} \right)   \right)$.
\end{lemma}
\begin{proof}
The estimator $\widehat{L}(X_1,\ldots,X_m)$ from $m$ samples $X_1,\ldots,X_m \sim P$  is defined to  be the $(m \cdot \eps)$-greatest sample.
Let $L$ denote the target threshold, that is, the real for which $\pr_{X \sim P}[X > L] = \eps$. Additionally, let $L_{-\eps},L_{+\eps}$ such that $\pr_{X \sim P}[L_{-} \leq  X  \leq L] = \pr_{X \sim P}[L \leq  X  \leq L_{+}] = \eps/100$.  Then, the probability that $\widehat{L}$ is not accurate is
\begin{align*}
\pr_{X_1,\ldots,X_m \sim P}\left[ \abs[\Big]{  \pr_{X \sim P}[X > \widehat{L}] - \eps} > \frac{\eps}{100} \right] \leq \pr[\widehat{L} > L_{+}] + \pr[\widehat{L}< L_{-}]  \;.
\end{align*}
Each term is bounded by an application of the multiplicative version of Chernoff bounds. Regarding first one,  $\widehat{L} > L_{+}$ implies that we had at least $m \cdot \eps$ samples in the interval $( L_{+}, +\infty)$. But this is a low probability event because each sample belongs in this region with probability only $\eps(1-1/100)$:
\begin{align*}
\pr[\widehat{L} > L_{+}] \leq \pr_{X_1,\ldots,X_m \sim P}\left[ \frac{1}{m}\sum_{j=1}^m\Ind (X_j > L_+) - \pr_{X \sim P}[X >  L_+] > \frac{\eps}{100} \right] \leq e^{-\Omega(\eps m)} \;.
\end{align*}
Choosing $m$ to be a sufficiently large multiple of $\eps^{-1}\log\left(  \frac{1}{\tau}\right)$ makes that probability of failure less than $\tau/2$. The other term $\pr[\widehat{L} < L_{-}] <\tau/2$ can be shown similarly. 

This algorithm stores $O(\eps^{-1}\log(1/\tau))$ one dimensional samples. If the algorithm selects the top $\eps \cdot m$ element by sorting the data first, its runtime is $O(\eps^{-1}\log(1/\tau)\log(\eps^{-1}\log(1/\tau)))$. However, $O(\eps^{-1}\log(1/\tau))$ is also possible for selection using a probabilistic divide and conquer algorithm.
\end{proof}

\end{document}